\documentclass{article}

\pdfoutput=1

\usepackage[utf8]{inputenc} 
\usepackage[T1]{fontenc}    
\usepackage{hyperref}       
\usepackage{url}            
\usepackage{booktabs}       
\usepackage{amsfonts}       
\usepackage{nicefrac}       
\usepackage{microtype}      
\usepackage{hyperref,fullpage,graphicx,amsmath,amsfonts,subcaption,amssymb,bm,url,breakurl,epsfig,epsf,color,MnSymbol,mathbbol,fmtcount,semtrans,caption,multirow,comment}
\usepackage{wrapfig}
\usepackage{mathtools}

\usepackage{dsfont}

\usepackage{hyperref,graphicx,amsmath,amsfonts,amssymb,bm,url,breakurl,epsfig,epsf,color,MnSymbol,mathbbol,fmtcount,semtrans,caption,subcaption,multirow,comment, boldline}
\usepackage[noend]{algpseudocode}
\usepackage{wrapfig}
\usepackage{enumitem}
\usepackage{amssymb}

\usepackage[utf8]{inputenc} 
\usepackage[T1]{fontenc}    
\usepackage{url}            
\usepackage{booktabs}       
\usepackage{amsfonts}       
\usepackage{nicefrac}       
\usepackage{microtype}      

\makeatletter
\providecommand*{\boxast}{%
  \mathbin{
    \mathpalette\@boxit{*}%
  }%
}
\newcommand*{\@boxit}[2]{%
  \sbox0{$\m@th#1\Box$}%
  \ifx#1\displaystyle \ht0=\dimexpr\ht0+.05ex\relax \fi
  \ifx#1\textstyle \ht0=\dimexpr\ht0+.05ex\relax \fi
  \ifx#1\scriptstyle \ht0=\dimexpr\ht0+.04ex\relax \fi
  \ifx#1\scriptscriptstyle \ht0=\dimexpr\ht0+.065ex\relax \fi
  \sbox2{$#1\vcenter{}$}
  \rlap{%
    \hbox to \wd0{%
      \hfill
      \raisebox{%
        \dimexpr.5\dimexpr\ht0+\dp0\relax-\ht2\relax
      }{$\m@th#1#2$}%
      \hfill
    }%
  }%
  \Box
}
\makeatother

  \makeatletter
\def\BState{\State\hskip-\ALG@thistlm}
\makeatother

  \usepackage{mathtools}

\usepackage{titlesec}

\usepackage{tikz}
\usepackage{pgfplots}
\usetikzlibrary{pgfplots.groupplots}

\titleformat{\paragraph}
{\normalfont\normalsize\bfseries}{\theparagraph}{1em}{}
\titlespacing*{\paragraph}
{0pt}{3.25ex plus 1ex minus .2ex}{1.5ex plus .2ex}

\usepackage{movie15}

\usepackage{caption}
\usepackage[bottom,hang,flushmargin]{footmisc} 

\newcommand{\tsn}[1]{{\left\vert\kern-0.25ex\left\vert\kern-0.25ex\left\vert #1 
    \right\vert\kern-0.25ex\right\vert\kern-0.25ex\right\vert}}

\definecolor{darkred}{RGB}{150,0,0}
\definecolor{darkgreen}{RGB}{0,150,0}
\definecolor{darkblue}{RGB}{0,0,200}
\hypersetup{colorlinks=true, linkcolor=darkred, citecolor=darkgreen, urlcolor=darkblue}

\newtheorem{theorem}{Theorem}[section]

\newtheorem{assumption}{Assumption}

\newtheorem{lemma}[theorem]{Lemma}
\newtheorem{corollary}[theorem]{Corollary}
\newtheorem{propo}[theorem]{Proposition}

\newtheorem{remark}[subsection]{Remark}

\newcommand{\beq}{\begin{equation}}

\newcommand{\eeq}{\end{equation}}

\newcommand{\nn}{\nonumber}

\newcommand{\A}{{\mtx{A}}}

\newcommand{\Ub}{{\mtx{U}}}

\newcommand{\V}{{\mtx{V}}}

\newcommand{\Sb}{\mtx{S}}
\newcommand{\Gb}{{\mtx{G}}}

\newcommand{\diag}[1]{\text{diag}(#1)}

\newcommand{\Dc}{{\cal{D}}}

\newcommand{\Pb}{{\mtx{P}}}
\newcommand{\Tb}{{\mtx{T}}}

\newcommand{\Cb}{{\mtx{C}}}
\newcommand{\Eb}{{\mtx{E}}}

\newcommand{\Gc}{{\cal{G}}}

\newcommand{\bSi}{{\boldsymbol{{\Sigma}}}}

\newcommand{\Db}{{\mtx{D}}}
\newcommand{\db}{{\vct{d}}}

\newcommand{\Iden}{{\mtx{I}}}
\newcommand{\M}{{\mtx{M}}}

\newcommand{\z}{{\vct{z}}}

\newcommand{\tn}[1]{\|{#1}\|_{\ell_2}}

\newcommand{\Cc}{\mathcal{C}}

\newcommand{\Rc}{\mathcal{R}}

\newcommand{\Nn}{\mathcal{N}}

\newcommand{\vb}{\vct{v}}

\newcommand{\Ic}{{\mathcal{I}}}

\newcommand{\cb}{\mtx{c}}

\newcommand{\w}{\vct{w}}

\newcommand{\wh}{{\hat{\mtx{w}}}}

\newcommand{\s}{\vct{s}}
\newcommand{\ab}{\vct{a}}
\newcommand{\bb}{\vct{b}}
\newcommand{\ub}{{\vct{u}}}

\newcommand{\h}{\vct{h}}

\newcommand{\g}{{\vct{g}}}

\newcommand{\Zb}{\mtx{Z}}

\newcommand{\yh}{\hat{\y}}

\newcommand{\gt}{\tilde{\g}}

\newcommand{\fronorm}[1]{\left\|#1\right\|_{F}}

\newcommand{\twonorm}[1]{\left\|#1\right\|_{\ell_2}}

\newcommand{\abs}[1]{\left|#1\right|}

\newcommand{\x}{\vct{x}}

\newcommand{\y}{\vct{y}}

\newcommand{\W}{\mtx{W}}

\definecolor{emmanuel}{RGB}{255,127,0}

\newcommand{\R}{\mathbb{R}}

\newcommand{\Var}{\textrm{Var}}

\renewcommand{\P}{\operatorname{\mathbb{P}}}
\newcommand{\E}{\operatorname{\mathbb{E}}}

\newcommand{\e}{\mathbf{e}}
\newcommand{\eb}{\vct{e}}
\renewcommand{\i}{\imath}

\newcommand{\vct}[1]{\bm{#1}}
\newcommand{\mtx}[1]{\bm{#1}}

\newcommand{\Id}{\text{\em I}}

\newcommand{\X}{{\mtx{X}}}
\newcommand{\Y}{{\mtx{Y}}}
\newcommand{\Vb}{{\mtx{V}}}

\newcommand{\Rb}{{\mtx{R}}}

\numberwithin{equation}{section} 

\def \endprf{\hfill {\vrule height6pt width6pt depth0pt}\medskip}
\newenvironment{proof}{\noindent {\bf Proof} }{\endprf\par}

\newcommand\zero{\mtx{0}}

\def\ind{\mathds{1}}

\newcommand{\Pib}{\boldsymbol{\Pi}}

\newcommand{\eqD}{\stackrel{\rm (D)}{=\joinrel=}}
\def\Wh{\widehat{\mtx{W}}}
\def\one{\vct{1}}
\def\wh{\widehat{\vct{w}}}
\def\bh{\widehat{\vct{b}}}
\newcommand{\bea}{\begin{align}}
\newcommand{\eea}{\end{align}}
\newcommand{\rP}{\stackrel{{P}}{\longrightarrow}}
\newcommand{\rD}{\stackrel{{(D)}}{\longrightarrow}}
\DeclarePairedDelimiterX{\inp}[2]{\langle}{\rangle}{#1, #2}
\def\nn{\notag}

\def\mub{\vct{\mu}}
\def\tb{\vct{t}}
\newcommand{\simiid}{\stackrel{iid}{\sim}}
\newcommand{\vp}{\vspace{3pt}}

\newcommand{\Ac}{\mathcal{A}}
\newcommand{\pib}{\boldsymbol{\pi}}
\newcommand{\alphab}{\boldsymbol{\alpha}}
\newcommand{\Sigmab}{\boldsymbol{\Sigma}}
\newcommand{\Deltab}{\boldsymbol{\Delta}}
\newcommand{\betab}{\boldsymbol{\beta}}
\newcommand{\fb}{\mathbf{f}}
\newcommand{\Gt}{\widetilde\Gb}

\newcommand{\gwt}{\widetilde{\g}}

\newcommand{\G}{\mtx{G}}
\newcommand{\pibt}{\widetilde{\pib}}

\newcommand{\nub}{\vct{\nu}}

\usepackage{authblk}

\title{Theoretical Insights Into Multiclass Classification:\\ A High-dimensional Asymptotic View}
\author{Christos Thrampoulidis\thanks{Department of Electrical and Computer Engineering, University of California, Santa Barbara, CA}\hspace{20pt} Samet Oymak\thanks{Department of Electrical and Computer Engineering, University of California, Riverside, CA}\hspace{20pt} Mahdi Soltanolkotabi\thanks{Ming Hsieh Department of Electrical Engineering, University of Southern California, Los Angeles, CA}}

\begin{document}

\maketitle

\begin{abstract}
Contemporary machine learning applications often involve classification tasks with many classes. Despite their extensive use, a precise understanding of the statistical properties and behavior of classification algorithms is still missing, especially in modern regimes where the number of classes is rather large. In this paper, we take a step in this direction by providing the first asymptotically precise analysis of linear multiclass classification. Our theoretical analysis allows us to precisely characterize how the test error varies over different training algorithms, data distributions, problem dimensions as well as number of classes, inter/intra class correlations and class priors. Specifically, our analysis reveals that the classification accuracy is highly distribution-dependent with different algorithms achieving optimal performance for different data distributions and/or training/features sizes. Unlike linear regression/binary classification, the test error in multiclass classification relies on intricate functions of the trained model (e.g., correlation between some of the trained weights) whose asymptotic behavior is difficult to characterize. This challenge is already present in simple classifiers, such as those minimizing a square loss. Our novel theoretical techniques allow us to overcome some of these challenges. The insights gained may pave the way for a precise understanding of other classification algorithms beyond those studied in this paper.
\end{abstract}



\addtocontents{toc}{\protect\setcounter{tocdepth}{0}}

\section{Introduction}
Multiclass classification is fundamental to a large number of real-world machine learning applications that demand the ability to automatically distinguish between thousands of different classes. Applications include essentially any problem with categorical outputs spanning natural language processing \cite{sutskever2014sequence}, where a seq2seq decoder has to choose the correct word token, reinforcement learning \cite{jang2016categorical,mei2020global}, where the agent has to choose the correct action, to recommendation systems, where the model should recommend the correct movie out of many other options. For instance, YouTube's recommendation system is modeled as an extreme multiclass problem with more than a million classes where each video corresponds to a viable class \cite{covington2016deep}.

The growing list of applications motivate an in-depth exploration of multiclass classification algorithms. Despite their extensive use however, a precise understanding of the statistical properties and behavior of classification algorithms is still missing with many open questions: \emph{What is the total and per class test accuracy? How does this quantity depend on various problem parameters such as data distributions, problem dimensions, etc.? What is the highest test accuracy achievable by any algorithm? What is the best algorithm for  each scenario? Which algorithm achieves the highest accuracy on rare or minority classes? How does the answer to the above question change in modern regimes where the number of classes is large?}



Asymptotic analysis in modern high-dimensional regimes where the number of training data and feature sizes grow in tandem with each other provides a promising setting for precisely quantifying the accuracy of classification algorithms as a function of problem variables and resolving the questions above. However, despite the rich literature on precise high-dimensional estimation and more recently binary classification, multiclass classification is an under-explored venue possibly due to the difficulty of capturing the intricate dependencies between the classes even for relatively simple linear classifiers.


\vspace{0.1in}
\noindent\textbf{Contributions.}~~We initiate a precise asymptotic study of linear multiclass classification in the modern high-dimensional regime, where the sizes of the training data and of the feature vectors grow large at a proportional rate. A key promise of such a precise analysis is that it allows us to accurately compare between different classification algorithms and data models. Compared to linear regression/binary classification, we identify the following crucial challenge: \emph{the test accuracy in multiclass classification relies on intricate cross-correlations between the trained weights of the classifier.} This has two consequences that drive our analysis. First, in order to obtain sharp asymptotics on the test error of any classifier, it is a prerequisite to precisely quantify the asymptotics of these cross-correlations. Second, the test error does not depend on the correlations in closed-form expressions. Thus, to compare between different classifiers, we need efficient numerical and analytic means to evaluate the test error in terms of the correlation matrices. Interestingly, we show that these challenges are already present in simple classifiers, such as minimizing the square loss, and in stylized distributional settings, such as Gaussian features. Our contributions are as follows:

%
%

\noindent$\bullet$~We study two different data models: a Gaussian Mixtures Model (GMM) and a Multinomial Logit Model (MLM) with Gaussian features. For each one of them, we provide a precise characterization of total and class-wise test accuracy for three different training algorithms: (i) a least-squares (LS) based classifier, (ii) a weighted least-squares (WLS) based classifier, and (iii) a simple per class averaging (Avg) estimator. For the least-squares based classifiers, we develop a new technique to overcome the technical challenge of characterizing the limiting behavior of the weights' cross-correlations. For the per class averaging classifier, we show that it is Bayes optimal for a GMM with equal priors. 

\noindent$\bullet$~We discuss efficient means of evaluating the test accuracy as a function of the weights' cross-correlations. This, together with the derived asymptotic formulae for the latter, lead to the first precise high-dimensional characterization of how the total/class-wise accuracy varies for different algorithms, data distributions, problem dimensions as well as number of classes, the inter/intra class correlations and class priors. For special problem geometries, we derive precise conditions on the data distribution and on the relative size of the training set over which each of the two studied algorithms dominates.




\noindent$\bullet$~We present and discuss numerical simulations that corroborate our theoretical findings. For instance, with an eye towards making classification algorithms more fair/equitable, we use our precise characterization of the class-wise accuracy to demonstrate how different algorithms behave in the presence of rare/minority classes. We also empirically compare the algorithms studied in this paper to other popular losses such as cross-entropy minimization. This allows us to better understand the performance of various algorithms in modern regimes of large number of classes.

\vspace{0.1in}
\noindent\textbf{Related Work.}~~There is a classical body of algorithmic work on multiclass classification, e.g.,  \cite{crammer2001algorithmic,lee2004multicategory,weston1998multi,bredensteiner1999multicategory,dietterich1994solving} and several empirical studies of their comparative performance \cite{rifkin2004defense,furnkranz2002round,allwein2000reducing,pal2005support}. A more recent extension of this line of work investigates the effect of the loss function in deep neural networks \cite{hou2016squared,gajowniczek2017generalized,kumar2018robust,bosman2020visualising,demirkaya2020exploring}. 
Algorithms for extreme multiclass problems with huge number of classes has also been studied in several \cite{choromanska2013extreme,yen2016pd,rawat2019sampled,kuznetsov2015rademacher} works. On the theory front, numerous works 
  have  investigated consistency \cite{zhang2004statistical,lee2004multicategory,tewari2007consistency,pires2013cost,pires2016multiclass} and finite-sample behavior \cite{koltchinskii2002empirical,guermeur2002combining,allwein2000reducing, li2018multi, cortes2016structured,lei2015multi,maurer2016vector,lei2019data} of multiclass classification algorithms. Our work differs from this literature in that we are interested in \emph{precise} characterizations of the test accuracy rather than order-wise bounds. Here we focus on linear classifiers, but we consider the modern high-dimensional regime in which both the sample size and the features' dimension are large. 


Specifically, our theoretical approach to linear multiclass classification fits in the rapidly growing literature on  \emph{sharp} high-dimensional asymptotics of convex optimization-based estimators \cite{donoho2006compressed,Sto,oymRank,Cha,TroppEdge,DMM,montanariLasso,TroppEdge,StoLasso,OTH13,COLT,karoui2013asymptotic,karoui15,donoho2016high,IT,TSP18,Master,miolane2018distribution,wang2019does,celentano2019fundamental,hu2019asymptotics,bu2019algorithmic,NIPS2019_9404,javanmard2020precise}.  Most of this line of work studies linear models and regression problems. More recently there has been a surge of interest in sharp analysis of a variety of methods tailored to binary classification models \cite{NIPS, huang2017asymptotic,candes2018phase,sur2019modern,mai2019large,logistic_regression,svm_abla,salehi2019impact,taheri2020sharp,deng2019model,montanari2019generalization,liang2020precise,kini2020analytic,mignacco2020role,lolas2020regularization,taheri2020fundamental}. Nevertheless, none of these prior works have yet considered multiclass classification settings. Our paper unveils the salient features of the multiclass setting and shows that corresponding results from the binary setting do not directly apply here. We emphasize that this is the case even for seemingly simple one-vs-all (OVA) classifiers, such as minimizing the square-loss, that involve training a single binary classifier per class \cite{rifkin2004defense}. The key technical tool behind our sharp analysis is the convex Gaussian min-max Theorem (CGMT) \cite{COLT,StoLasso}. However, a ``naive" application of the CGMT on the original optimization of the classifier does not allow us to compute all the necessary correleations between the classfier's weights to precisely capture the total/class-wise errors. Instead, our key idea is to formulate an artificial optimization problem, which captures the missing correlations and at the same time conveniently allows us to leverage the CGMT.
%
%


\vspace{0.1in}
\noindent\textbf{Notation.}~~We use $[k]$ to denote $\{1,\ldots,k\}$. We use boldface lowercase letters $\x,\y,\mub,\ldots$ to denote vectors and boldface uppercase letters $\X,\Y,\M,\ldots$ for matrices. We write $\eb_\ell$ for the $\ell$-th standard basis vector in $\R^k$. We also write $\Iden_k, \zero_{k\times k} $ and $\one_k$ for the $k\times k$ identity and all-zeros matrices and the $k\times 1$ all-ones vectors. For a vector $\vct{c}\in\R^k$ we write $\arg\max \vct{c}$ to denote the index of its largest entry, i.e.,~$\arg\max \vct{c}={\arg\max}_{j\in[k]}\text{ }\vct{c}_i$. The superscript $\dagger$ denotes pseudoinverse. We use $Q(x)$ for the tail of a standard Gaussian (Q-function).
  Finally, we reserve variables $G_0,G_1,\ldots,G_k\simiid\Nn(0,1)$ to denote i.i.d. standard Gaussians.

\section{Problem formulation}
We focus on multiclass classification problems with $k$ classes. Specifically, we assume the training data consists of $n$ feature/label pairs $\{(\vct{x}_i,Y_i)\}_{i=1}^n$ with $\vct{x}_i\in\R^d$ representing the features and $Y_i\in\{1,2,\ldots,k\}$ the associated labels representing one of $k$ classes. It will be convenient to also model the labels as one-hot encoded vectors $\vct{y}_i\in\R^k$ representing one of $k$ classes with one-hot encoding, i.e.,~$\vct{y}_i=\vct{e}_{Y_i}$. Therefore, when convenient we shall use $\{(\vct{x}_i,\vct{y}_i)\}_{i=1}^n$ to represent the training data. Throughout, we shall use $$\X=\begin{bmatrix} \x_1 & \x_2 & \ldots & \x_n \end{bmatrix}\in\R^{d\times n}\,\quad\text{and}\quad\Y=\begin{bmatrix} \vct{y}_1 & \vct{y}_2 & \ldots & \vct{y}_n\end{bmatrix} \in\R^{k\times n}\,,$$
 to denote the matrix of features and their labels aggregated into a matrix, respectively.
We shall also use $\mtx{Y}_\ell\in\R^n$ to denote the $\ell$-th row of $\Y$. In our analysis we focus on training linear classifiers. Specifically, we use 
$$\mtx{W} = \begin{bmatrix} 
\w_1 & \w_2&  \cdots & \w_k
\end{bmatrix}^T
\in\R^{k\times d}\quad \text{and} \quad\vct{b}\in\R^k$$
to denote the weights and biases of this linear model, respectively. The overall input-output relationship of the classifier in this case is a function that maps an input vector $\vct{x}\in\R^d$ into an output of size $k$ via $\vct{x}\mapsto \mtx{W}\vct{x}+\vct{b}\in\R^k,$ where a training algorithm is used to train the corresponding weights $\mtx{W}\in\R^{k\times d}$ and biases $\vct{b}\in\R^k$. Next we detail the data models and training algorithms that are formally studied in this paper. We end this section by discussing how the test error can be calculated for the different data models. 

\subsection{Data Models}\label{sec:data_models}
In our theoretical analysis we assume the training data $\{(\vct{x}_i,Y_i)\}_{i=1}^n$  (alternatively $\{(\vct{x}_i,\vct{y}_i)\}_{i=1}^n$) are generated i.i.d.~according to $(\vct{x},Y)$/$(\vct{x},\vct{y})$. We consider two models for the distribution of $(\vct{x},\vct{y})$ which we detail next. In both models we shall use mean/regressor vectors $\{\vct{\mu}_\ell\}_{\ell=1}^k\in\R^d$ and aggregate them into columns of a matrix of the form $$\mtx{M}:=\begin{bmatrix}\vct{\mu}_1& \vct{\mu}_2 &\ldots & \vct{\mu}_k\end{bmatrix}\in\R^{d\times k}.$$ In the first model, these vectors represent the mean of the features conditioned on the class, i.e.,~$\vct{\mu}_\ell=\E\big[\vct{x}| Y=\ell\big]$, whereas in the second model these vectors can be viewed as regressor coefficients. We shall refer to $\{\vct{\mu}_\ell\}_{\ell=1}^k$/$\mtx{M}$ as ``mean'' vectors/matrix in both models. We denote the Grammian matrix of means as $\Sigmab_{\mub,\mub} = \M^T\M.$ Furthermore, we shall use $\mu_\ell:=\twonorm{\vct{\mu}_\ell}$ to denote the norm of the mean vector $\vct{\mu}_\ell$.

\vspace{0.1in}
\noindent\textbf{Gaussian Mixture Model (GMM).}~~In this model each example $(\vct{x},Y)$ belongs to class $\ell\in[k]$ with probability $\pi_\ell$, i.e., $\P\{Y= \ell \}= \pi_\ell$. We let $\pib=\begin{bmatrix}\pi_1 & \pi_2 & \ldots & \pi_k \end{bmatrix}^T\in\R^k$ denote the vector of priors which of course obeys $\pib\geq \zero$ and $\one^T\pib=1$. Also, we model the class conditional density of an example in class $\ell$ with an isotropic Gaussian centered at a mean vector $\mub_\ell$. In particular, we say that a data point $(\x,Y)$ (or its one-hot encoded representation $(\x,\vct{y})$) follows the GMM model when
\begin{align}\label{eq:GM_model}
\P\{Y = \ell \} = \pi_\ell\qquad\text{and}\qquad\x=\mub_Y + \z,~\z\sim\mathcal{N}(\zero,\sigma^2\Iden_d).
\end{align} 
We note that for a training set summarized by the feature and label matrices $\X$ and $\Y$ with columns generated i.i.d.~according to the above distribution we have: $\X = \M\Y + \Zb$ where   $\mtx{Z}\in\R^{d\times n}$ is a Gaussian noise matrix with i.i.d.~$\mathcal{N}(0,\sigma^2)$ entries. 

\vspace{0.1in}
\noindent\textbf{Multinomial Logit Model (MLM).}~~In this model we assume that feature vectors $\x$ are distributed i.i.d.~$\Nn(\zero,\mtx{I}_d)$ and that the conditional density of the class labels is given by the soft-max function. 
Concretely, we say that a data point $(\x,Y)$ (or its one-hot encoded representation $(\x,\vct{y})$) follows the multinomial logit model when
\begin{align}\label{eq:log_model}
\x\sim\mathcal{N}(\zero,\Iden_d)\qquad\text{and}\qquad\P\{Y = \ell ~|~ \x \} = {e^{\inp{\mub_\ell}{\x}}}\big/{\sum_{j\in[k]}e^{\inp{\mub_j}{\x}}}.
\end{align} 

\subsection{Classification algorithms}
\label{classalg}
As mentioned earlier, in this paper we focus on training linear classifiers of the form $\vct{x}\mapsto \mtx{W}\vct{x}+\vct{b}$ with $\mtx{W}\in\R^{k\times d}$ denoting the weights and $\vct{b}\in\R^k$ the offset values.


\vp
\noindent\textbf{Least-squares (LS).} In this approach 
we train a linear classifier $\vct{x}\mapsto \mtx{W}\vct{x}+\vct{b}$ via a least-squares fit to the training data: 
$$
(\Wh,\bh):=\arg\min_{\W,\bb} \frac{1}{2n}\sum_{i=1}^n \twonorm{\mtx{W}\vct{x}_i+\vct{b}-\vct{y}_i}^2=\frac{1}{2n}\fronorm{\mtx{W}\mtx{X}+\vct{b}\vct{1}_n^T-\mtx{Y}}^2.
$$



\vp
\noindent\textbf{Class averaging (Avg).}
This approach uses the following weight and offset values
$$\Wh:=\frac{1}{n}\Y\X^T\quad\text{and}\quad\bh:= \frac{1}{n}\Y\one.$$
Let $n_\ell$ be the number of training data from class $\ell$ then, equivalently, 
$
\widehat{\vct{w}}_\ell
=\frac{n_\ell}{n}\left(\frac{1}{n_\ell}\sum_{i:\text{ }Y_i=\ell}^n\vct{x}_i\right)\text{ and }\bh_\ell
=\frac{n_\ell}{n}.
$
Therefore, this classifier picks weights according to the empirical mean of features of each class  multiplied by the relative frequency of that class and the offset value as the fraction of data points from that class. We note that this algorithm has the same classification performance as the outcome of the ridge-regularized least-squares with infinite regularization.

\vp
\noindent\textbf{Weighted Least-squares (WLS).}
This is a variation of the Least-squares approach where we fit a weighted least squares loss of the form
$$(\Wh,\bh):=\arg\min_{\mtx{W},\vct{b}}~ \frac{1}{2n}\fronorm{\left(\mtx{W}\mtx{X}+\vct{b}\vct{1}_n^T-\mtx{Y}\right)\mtx{D}}^2.$$
Here, $\mtx{D}\in\R^{n\times n}$ is a diagonal matrix with the ith diagonal entry equal to $D_{ii}=\omega_\ell$ when the i-th data point is from class $\ell$ (i.e.~$Y_i=\ell$) and $\omega_\ell\geq 0,~\ell\in[k]$ denote the weights. Aggregating the weights into a vector of the form $\vct{\omega}=\begin{bmatrix}\omega_1 & \omega_2 & \ldots & \omega_k\end{bmatrix}^T\in\R^k$ we can rewrite $\mtx{D}$ in the form
$$
\mtx{D}=\text{diag}\left(\mtx{Y}^T\vct{\omega}\right).
$$
In this approach the loss associated to data points to class $\ell$ is weighted by a factor $\omega_\ell^2$.  For instance, if the class priors are known, a natural choice might be $\omega_\ell=1/\sqrt{\pi_\ell}$.
Such a weighted approach allows the classification algorithm to focus on rare/minority classes which are not well represented in the training data.


\vp
\noindent\textbf{Cross-entropy (CE).} In this approach the best weight/offset values are determined by fitting a cross entropy loss 
$
(\Wh,\bh):=\arg\min_{\mtx{W}, \vct{b}} \frac{1}{n}\sum_{i=1}^n \log\Big(\frac{\sum_{\ell=1}^ke^{\langle \widehat{\vct{w}}_\ell, \vct{x}_i\rangle+\vct{b}_\ell}}{e^{\langle \widehat{\vct{w}}_{Y_i}, \vct{x}_i\rangle+\vct{b}_{Y_i}}}\Big).
$
Theoretical analysis for CE is substantially more involved and we defer it to future work. Nevertheless, we compare with this classifier in our numerical simulations. 

\subsection{Class-wise and total test classification error}\label{sec:test}
Let $\Wh,\bh$ denote the parameters of a trained classifier. Now consider a fresh data sample $(\vct{x},Y)$ generated according to the same distribution as the training data. Once, we have learned the parameters $\Wh,\bh$ of the classifier, the class $\widehat{Y}$ predicted by the classifier is made by a winner takes it all strategy, as follows,
$
\widehat{Y}={\arg\max}_{j\in[k]}\text{ }\inp{\wh_j}{\x} + \bh_j.
$
Therefore, the classification error condition on the the true label being $c$, which we shall refer to as the \emph{class-wise test error}, is equal to
\begin{align}
\label{eq:class_errorc}
\P_{e|c}:=\P\big\{\widehat{Y}\neq Y \big | Y=c\big\}=\P\big\{ \inp{\wh_c}{\x} + \bh_c \leq \max_{j\neq c}~\inp{\wh_j}{\x} + \bh_j \big\}.
\end{align}
Correspondingly, the \emph{total classification error} is given by
\begin{align}\label{eq:class_error}
\hspace{-0.25in}\P_e := \P\big\{\widehat{Y}\neq Y\big\}= \P\Big\{ \arg\max_{j\in[k]} \{ \inp{\wh_j}{\x} + \bh_j \} \neq Y \} \Big\}=\P\big\{ \inp{\wh_Y}{\x} + \bh_Y \leq \max_{j\neq Y}~\inp{\wh_j}{\x} + \bh_j \big\}.
\end{align}
For both the GMM and MLM, the classification error depends on the vector of intercepts $\bh\in\R^k$ and the following key ``correlation" matrices:
\vp
$$\mtx{\Sigma}_{\w,\w} := \Wh\Wh^T\quad\text{and}\quad \mtx{\Sigma}_{\w,\mub} := \Wh\M.$$

\textbf{GMM.}~~In model \eqref{eq:GM_model}, the test error probability is explicitly given by
\begin{align}\label{eq:Pe_GMM_gen}
\P_e =\P\Big\{ \arg\max\text{ }\left(\sigma\,\g+\bh + \Sigmab_{\w,\mub}\eb_Y\right) \neq Y \Big\},~~\text{where}~\g\sim\Nn\left(\zero,\mtx{\Sigma}_{\w,\w}\right),
\end{align}
and $Y$ is independent of $\g$ with probability mass function $\P\{Y=\ell\} = \pi_\ell,~~ \ell\in[k].$

\vp
\textbf{MLM.}~~In model \eqref{eq:log_model}, the test error probability is explicitly given by
\begin{align}
\label{eq:Pe_log_gen}
\P_e = \P\big\{ \arg\max\text{ }(\,\g + \bh\,) \neq Y(\h)  \big\}\,,\quad\text{where}~~\begin{bmatrix}
\g \\ \h
\end{bmatrix} \sim \Nn\Big( \zero , \begin{bmatrix}
\Sigmab_{\w,\w} & \Sigmab_{\w,\mub} \\ \Sigmab_{\w,\mub}^T & \Sigmab_{\mub,\mub}
\end{bmatrix} \Big)\, ,
\end{align}
and $\P\{Y(\h)=\ell\} = {e^{\h_\ell}}/{\sum_{j\in[k]} e^{\h_j}},~~ \ell\in[k].$

\vspace{0.1in}
\noindent\textbf{Calculating the class-wise/total misclassifcation errors.} The identities \eqref{eq:Pe_GMM_gen} and \eqref{eq:Pe_log_gen} (see Section \ref{sec:proof_Pe_gen} for a proof) as well as similar ones for the class-wise test error demonstrate that the total/class-wise errors only depend on the correlation matrices $\mtx{\Sigma}_{\w,\w}$ and $\mtx{\Sigma}_{\w,\mub}$, the offset values $\bh$ and the the class conditional means.
For instance, as we show in the supplementary for GMM the class-wise errors are given by %
\bea\label{eq:condPintro}
\P_{e|c}= 1 - \P\big\{\Sb_c^{{1}/{2}}\, \z \geq \tb_c \big\},
\end{align}
where $\vct{z}$ is a Gaussian random vector distributed as $\mathcal{N}(\vct{0},\sigma^2\mtx{I}_{k-1})$, $\Sb_c\in\R^{{(k-1)}\times{(k-1)}}$ is a symmetric matrix such that its $i,j$ element is given by $[\Sb_c]_{ij}:=\inp{\wh_c-\wh_j}{\wh_c-\wh_i}$ and $\tb_c\in\R^{k-1}$ a vector with entries $[\tb_c]_i:=\inp{\wh_i-\wh_c}{\mub_c} +  (\bh_i-\bh_c)$. Similarly, based on \eqref{eq:condPintro} the total classification error in GMM is equal to $\P_{e}=\sum_{\ell=1}^k \pi_\ell \P_{e|c}=1-\sum_{\ell=1}^k\pi_\ell \P\big\{ \Sb_c^{{1}/{2}}\, \z \geq \tb_c \big\}$. As also detailed in the supplementary, the class-wise/total test errors for MLM similarly depends on quantities of the form $\P\{\mtx{A}\vct{z}\ge \vct{t} \}$ with $\vct{z}$ a standard Gaussian random vector, $\mtx{A}$ and $\vct{t}$ depending only on correlation matrices, conditional means and classifier offset-values; see Section \ref{sec:last1}. There are a variety of algorithmic approaches to calculate $\P\{\mtx{A}\vct{z}\ge \vct{t} \}$ once $\mtx{A}$ and $\vct{t}$ are known based on Monte Carlo methods. Analytic bounds on this quantity have also been studied in the literature, e.g., \cite{hashorva2003multivariate,sathe1980note}; see more details in Section \ref{Testerrcal}. 


\subsection{High-dimensional regime}
This paper derives sharp asymptotic formulae for the class-wise and total classification error of averaging and (weighted) LS algorithms for GMM and MLM. We defer all our proofs to the appendix. All our results hold in the following high-dimensional regime with finite $k$.
\begin{assumption}\label{ass:HD}
We focus on a double asymptotic regime where $n,d\rightarrow\infty$ at a fixed ratio $\gamma=d/n>0$. 
\end{assumption}
For the (weighted) least-squares classifier, we focus here in the overdetermined regime $\gamma< 1$. However, our approach is also directly applicable to regularized (or min-norm) LS/WLS in the overparameterized regime $\gamma>1$. 

For a sequence of random variables $\mathcal{X}_{n,d}$ that converges in probability 
 to some constant $c$ 
 in the limit above, we simply write $\mathcal{X}_{n,d}\rP c$. For a random vector/matrix $\vb_{n,d}$/$\Vb_{n,d}$ and a deterministic vector/matrix $\cb$/$\Cb$, the expressions $\vb_{n,d}\rP \cb$ and $\Vb_{n,d}\rP \Cb$ are to be understood entry-wise.
 
%



\section{Results for Gaussian Mixture Model}\label{gmm sec}
In this section we discuss the asymptotics of the intercepts/correlation matrices for the averaging and the LS classifiers for the GMM. The derived formulas can be directly plugged in \eqref{eq:Pe_GMM_gen} and \eqref{eq:condPintro} to obtain asymptotics for the total and class-wise test error, respectively.   We end this section by also characterizing the Bayes optimal estimator in this model when priors are balanced $\pi_\ell=1/k, \ell\in[k]$. Additional results on the performance of Weighted LS are deferred to the appendix.

\subsection{Class averaging classifier}

\begin{propo}\label{thm:ave_GM}
Consider data generated according to GMM in an asymptotic regime with any $\gamma>0$. For the averaging estimator discussed in Section \ref{classalg}, the following high-dimensional limits hold
\begin{subequations}\label{eq:ave0}
\bea
\bh &\rP \pib \label{eq:ave01}\,, \quad \Sigmab_{\w\mub} \rP \diag{\pib}\,  \cdot\Sigmab_{\mub,\mub}\,,
 \\
\Sigmab_{\w,\w} &\rP \gamma\sigma^2\cdot\diag{\pib} + \diag{\pib}\cdot\Sigmab_{\mub,\mub}\cdot\diag{\pib}\,.  
\end{align}
\end{subequations}
\end{propo}


The above result allows us to precisely characterize the behavior of the averaging estimator in the high-dimensional regime. Let us consider a few special cases.

 \vspace{0.1in}
\noindent\textbf{Two classes.}~~Consider the special case with two classes with class priors $\pi_1=1-\pi_2=:\pi$. In this case we can compute the class-wise misclassification probabilities $\P_{e|1}$ and $\P_{e|2}$ explicitly.  Specifically using \eqref{eq:ave0}, we have
$
 \mtx{S}_1 = \twonorm{\pi\mub_1 - (1-\pi)\mub_2}^2 + \gamma\sigma^2$ and  $t_1=(1-2\pi) + (1-\pi)\inp{\mub_1}{\mub_2} - \pi \twonorm{\mub_1}^2.
$
Substituting the latter two in \eqref{eq:condPintro} we arrive at
 $
 \P_{e|1} \rP Q\big(\frac{\pi\twonorm{\mub_1}^2 - (1-\pi)\inp{\mub_1}{\mub_2} + 2\pi - 1}{\sqrt{\twonorm{\pi\mub_1 - (1-\pi)\mub_2}^2 + \gamma\sigma^2}}\big)\,.
 $
In the case of equal priors $\pi=\pi_1=\pi_2=1/2$, antipodal and equal energy of the means, i.e.,~$\mub_1=-\mub_2$ and $\mu:=\twonorm{\vct{\mu}_1}=\twonorm{\vct{\mu}_2}$, we can use the above to conclude that
$ \P_{e|1}= \P_{e|2}=\frac{1}{2}\P_e = \frac{1}{2}Q\big( \sqrt{\frac{\mu^2}{\mu^2+\gamma\sigma^2}}  \big).$
This formula recovers the result of \cite{mignacco2020role} for this special case. Also, as mentioned in \cite{mignacco2020role}, the formula matches the Bayes optimal error computed in \cite{lelarge2019asymptotic} for Gaussian mean vectors. This shows that the class averaging method is Bayes optimal in this very simple setting.  In Section \ref{sec:Bayes}, we generalize this result to multiple classes: we show that the average estimator is (asymptotically) Bayes optimal for balanced classes and equal-energy Gaussian means for any $k\geq 2$. 

 \vspace{0.1in}
\noindent\textbf{Orthogonal means, equal priors and equal energy.}~~Next we focus on a special case with orthogonal means $\inp{\mub_i}{\mub_j}=0,~i\neq j\in[k]$ of equal energy $\mu^2:=\twonorm{\mub_\i}^2$ and of equal priors $\pi_i=\pi={1}\big/{k}$ for $i\in[k]$. In this case, the class-wise miss-classification error converges to 
$
\P_{e|c} \rP 1 - \P\{ \Sb_c^{{1}/{2}}\, \z > \tb \},
$
where 
$\Sb_c=\pi(\pi \mu^2+\gamma\sigma^2)(\Iden_{k-1} + \one_{k-1}\one_{k-1}^T)$ and $\tb = -\pi \mu^2\one_{k-1}$. Defining $$u_{\rm Avg}:=\frac{\mu^2}{\sigma}\sqrt{\frac{1}{\mu^2+k{\gamma\sigma^2}}}\,,$$ after some algebraic manipulations the total classification error of the averaging estimator in this case is given by
$$ \P_{e|c}=\P_{e, \rm Avg} \rP \P\big\{  G_0 + \max_{j\in[k-1]} G_j \geq \,u_{\rm Avg} \big\},$$
where  $G_0,\ldots,G_{k-1}\simiid\Nn(0,1).$

\subsection{Least-squares classifier}
This section focuses on characterizing the intercepts and correlation matrices for the least-squares classifier. To present our results, we assume that the  Grammian matrix has eigenvalue decomposition 
\begin{align}\label{eq:eigen}
\Sigmab_{\mub,\mub}=\M^T\M=\Vb\Sigmab^2\Vb^T, \qquad \Sigmab\succ\mathbf{0}_{r\times r},~\Vb\in\R^{k\times r},~ r\leq k.
\end{align}
with $\Sigmab$ a diagonal positive-definite matrix  and $\Vb$ an orthonormal matrix obeying $\Vb^T\Vb=\Iden_r$.

\begin{theorem}\label{thm:LS_GM}
Consider data generated according to GMM in an asymptotic regime with $\gamma<1$.
In addition to \eqref{eq:eigen}, define the following two positive (semi)-definite matrices:
$
\Pb := \diag{\pib}-\pib\pib^T\succeq \zero_{k\times k}$ and $\Deltab:=\sigma^2\Iden_r + \Sigmab\Vb^T\Pb\Vb\Sigmab \succ \zero_{r\times r}.$
Then, for the least-squares linear classifier $\left(\Wh,\bh\right)$ the following limits are true asymptotically
\begin{subequations}\label{eq:LS_GM}
\begin{align}\label{eq:LS_GMa}
\bh &\rP\pib-\Pb\Vb\Sigmab\Deltab^{-1}\Sigmab\Vb^T\pib\,,\quad
\Sigmab_{\w,\mub} \rP \Pb\Vb\Sigmab\Deltab^{-1}\Sigmab\Vb^T \,,\\
\Sigmab_{\w,\w} &\rP\frac{\gamma}{(1-\gamma)\sigma^2} \Pb+ \Pb\Vb\Sigmab \Deltab^{-1} \Big(\Deltab^{-1} - \frac{\gamma}{(1-\gamma)\sigma^2}\Iden_r\Big) \Sigmab\Vb^T \Pb\label{eq:LS_GMb}\,.
\end{align}
\end{subequations}


\end{theorem} 

The above result allows us to precisely characterize the behavior of the least-squares classifier in the high-dimensional regime. 
In Section \ref{sec:ortho_app}, we specialize \eqref{eq:LS_GM} to the case of orthogonal means. Compared to the weight vectors $\wh_i, i\in[k]$ of the class averaging classifier that are also (asymptotically) orthogonal when means are orthogonal, this is \emph{not} the case for  LS. We show next that these spurious correlations only hurt the classification error when classes are balanced.
\begin{propo}\label{cor:LS_ortho}
Consider the case of orthogonal, equal energy-means $\Sigmab_{\mub,\mub}=\mu\Iden_k$, balanced priors $\pi_i=1/k,~i\in[k]$ and $\gamma<1$. 
Setting $u_{\rm LS}:=\frac{\mu^2}{\sigma}\sqrt{ \frac{1-\gamma}{\mu^2+k{\gamma\sigma^2}}}\,,$  it holds that
$$\P_{e,\rm LS} \rP \P\big\{  G_0 + \max_{j\in[k-1]} G_j \geq u_{\rm LS} \big\}.$$ Specifically, since $u_{\rm LS} = u_{\rm Avg}\sqrt{1-\gamma}\,<\,u_{\rm Avg}$, the averaging estimator strictly outperforms LS for all $0<\gamma< 1$ and $k\geq 2$ in this setting.
\end{propo}

\subsection{Bayes estimator for the balanced Gaussian Mixture Model}\label{sec:Bayes}
To check how far the above algorithms are from the lowest misclassification error achievable by any algorithm in this section, we consider a Bayesian setting with Gaussian mean vectors and we derive the Bayes-optimal risk for the case of equal priors. 
Recall that the Bayes estimator $
\hat{Y} = \arg\max_{\ell\in[k]} \P\{ Y=\ell~|~\X,\Y,\x \}
$ minimizes the risk $\P_{e} = \P\{ \hat{Y}\neq Y \}= \E_{\X,\Y,\x,Y}\left[\mathbb{1}[\hat{Y}\neq Y]\right]$.

\begin{propo}\label{propo:Bayes}
Consider $\mub_i\simiid\Nn(\zero,\frac{\mu^2}{d}\Id_d)$
and $\pi_i={1}/{k}$ for all $i\in[k]$. Set $u_{\rm Bayes}:=\frac{\mu^2}{\sigma}\frac{1}{ \sqrt{\mu^2 + k{\gamma\sigma^2}}}.$ Then, the Bayes risk converges to 
$\P\Big\{ G_0 + \max_{\ell\in[k-1]}  G_\ell \geq u_{\rm Bayes}  \Big\}.$
\end{propo}

Under Gaussian prior, the means are asymptotically orthogonal and equal-energy. As shown earlier, in this setting, $\P_{e,\rm Avg}\rP\P\big\{ G_0 + \max_{\ell\in[k-1]}  G_\ell \geq u_{\rm Avg}  \big\}$. But, $u_{\rm Avg}=u_{\rm Bayes}$. Thus, the averaging method is (asymptotically) Bayes optimal for equal-norm, orthogonal means and balanced classes. An analogous result was derived in \cite{lelarge2019asymptotic,mignacco2020role}, but only for binary classification.

\section{Results for Multinomial Logit Model}
In this section we discuss the asymptotics of the intercepts/correlation matrices for MLM. 
We present results for arbitrary mean-vectors as well as special cases where the means are mutually orthogonal. Recall the eigenvalue decomposition of the Grammian $\Sigmab_{\mub,\mub}=\Vb\Sigmab^2\Vb^T$ in \eqref{eq:eigen}. In order to state our results, it is convenient to introduce the following probability vectors in $\R^k$ and $\R^{k^2}$:
\bea\label{eq:alphas_gen}
\pib := \E\Big[\frac{e^{\Vb\Sigmab\g}}{\one_k^Te^{\Vb\Sigmab\g}} \Big]\in\R^k~~\text{and}~~ \Pib := \E\Big[\frac{\left(e^{\Vb\Sigmab\g}\right)\left(e^{\Vb\Sigmab\g}\right)^T}{\left(\one_k^Te^{\Vb\Sigmab\g}\right)^2}  \Big]\in\R^{k\times k},~~\text{where}~\g\sim\Nn(\zero,\Iden_r).
\end{align}
Note that $\pib$ and $\Pib$ are the first and second moments of the soft-max mapping of  $\Vb\Sigmab\g\sim\Nn\left(\zero,\Sigmab_{\mub,\mub}\right)$. 
In fact, for the MLM in \eqref{eq:log_model} 
it holds that $$\P\{Y=\ell\}= \E[\P\{Y=\ell\,|\,\x\}]= \E\big[\frac{e^{\eb_\ell^T\Vb\Sigmab\g}}{\one_k^Te^{\Vb\Sigmab\g}}\big]=\pib_\ell,~\ell\in[k]$$  since $\M^T\x$ is distributed as ${\Vb\Sigmab\g}$. Thus, $\pib$ is the vector of class priors  (which explains the slight abuse of notation here in relation to our notation for the class priors of the GMM).

\subsection{Class averaging classifier}

\begin{propo}\label{propo:ave_log}
Consider data generated according to MLM in an asymptotic regime with any $\gamma>0$. For the averaging classifier, the following high-dimensional limits hold
\begin{subequations}\label{eq:ave_soft}
\begin{align}
\bh &\rP \pib\,,
\quad\Sigmab_{\w,\mub} \rP  \left(\diag{\pib}-\Pib\right)\cdot{\Sigmab_{\mub,\mub}}\,,
\label{eq:ave_soft_2}
\\
\Sigmab_{\w,\w} &\rP \gamma\cdot\diag{\pib} + \left(\diag{\pib}-\Pib\right)\cdot{\Sigmab_{\mub,\mub}}\left(\diag{\pib}-\Pib\right)\,.
\label{eq:ave_soft_3}
\end{align}
\end{subequations}
\end{propo}
Using Gaussian decomposition in \eqref{eq:Pe_log_gen} and checking from \eqref{eq:ave_soft} that $\Sigmab_{\w,\w}-\Sigmab_{\w,\mub}\Sigmab_{\mub,\mub}^\dagger\Sigmab_{\w,\mub}^T \rP \gamma\cdot\diag{\pib}$ the test error obtains the following explicit form:
\begin{align}\label{eq:simple_log_Pe_ave}
\P_{e,\rm Avg} \rP \P\big\{ \arg\max \left\{ \sqrt{\gamma}\cdot\diag{\sqrt{\pib}}\cdot\gwt + \left(\diag{\pib}-\Pib\right)\cdot \Vb\Sigmab \cdot\g  +\pib \right\} \neq Y(\g) \big\},
\end{align}
where $\gwt\sim\Nn(\zero,\Iden_k)$, $\g\sim\Nn(\zero,\Iden_r)$ and $\P\{Y(\g)=c\} ={e^{\eb_c^T\Vb\Sigmab\g}}\big/{\sum_{j\in[k]} e^{\eb_j^T\Vb\Sigmab\g}},c\in[k]$.

\subsection{Least-squares classifier}

This section focuses on characterizing the intercepts and correlation matrices for the least-squares classifier. We also use the result to characterize conditions under which LS outperforms averaging. 

\begin{theorem}\label{propo:LS_log}
Consider data generated according to MLM in an asymptotic regime with  $0<\gamma<1$. Recall the notation in \eqref{eq:alphas_gen}. For the LS classifier, the following high-dimensional limits hold.
\begin{subequations}\label{eq:LS_soft}
\begin{align}
\bh &\rP \pib\,,
\quad\Sigmab_{\w,\mub} \rP  \left(\diag{\pib}-\Pib\right)\cdot{\Sigmab_{\mub,\mub}}\,,
\label{eq:ls_soft_2}
\\
\Sigmab_{\w,\w} &\rP \frac{\gamma}{1-\gamma}\cdot\left(\diag{\pib}-\pib\pib^T\right) + \frac{1-2\gamma}{1-\gamma}\cdot\left(\diag{\pib}-\Pib\right)\cdot\Sigmab_{\mub,\mub}\cdot\left(\diag{\pib}-\Pib\right)\,.
\label{eq:ls_soft_3}
\end{align}
\end{subequations}
%
\end{theorem}

It is interesting to observe that \eqref{eq:ls_soft_2} is identical to \eqref{eq:ave_soft_2}.  However, the cross-correlations in $\Sigmab_{\w,\w}$ differ. We prove below that this leads to an improved performance of the LS classifier for large sample sizes. First, Theorem \ref{propo:LS_log} can be used to check that 
$$
\Sigmab_{\w,\w}-\Sigmab_{\w,\mub}\Sigmab_{\mub,\mub}^\dagger\Sigmab_{\w,\mub}^T \rP \frac{\gamma}{1-\gamma}\left(\diag{\pib} - \pib\pib^T - \left(\diag{\pib}-\Pib\right)\Sigmab_{\mub,\mub}\left(\diag{\pib}-\Pib\right)\right).\nn
$$
Thus, the only change in the test-error formula compared to \eqref{eq:simple_log_Pe_ave} is the term $\gamma\cdot\diag{\pib}$ substituted by the matrix above. 
\begin{propo}\label{propo:gamma_star}
Assume orthogonal, equal-energy means $\Sigmab_{\mub,\mub}=\mu^2\Iden_k$, $k\geq 2$. 
Let 
$$\gamma_\star
 = \frac{\mu^2 k}{(k-1)^2}\,\Big(\,1-k\E\Big[\frac{e^{2\mu G_1}}{\left(\sum_{\ell\in[k]}e^{\mu G_\ell}\,\right)^2} \Big]\,\Big)^2\in(0,1).$$
Then, with probability 1 as $n\rightarrow\infty$, 
$\P_{e,\rm LS} < \P_{e,\rm Avg} \Longleftrightarrow \gamma< \gamma_\star.$
\end{propo}

\section{Numerical Results}

This section validates our theory via numerical experiments and provides further insights on multiclass classification. See also Section \ref{sec:additionalnum} for more extensive experiments. We study the class-wise/total test misclassification error in both GMM and MLM for different sample sizes, number of classes and class priors. In line with Section \ref{classalg} we consider four algorithms: (i) Averaging (Avg), (ii) LS, (iii) Weighted LS (WLS) with the $i$th class weighted by $\omega_\ell^2=1/\pi_\ell$, (iv) Cross-Entropy (CE).

Figures \ref{fig11} and \ref{fig12} focus on GMM with $k=9$ classes, $d=300$ and $\tn{\mub_i}^2=15$. To model different class prior probabilities, we use the distribution $
\pi_{1}=\pi_2=\pi_3=0.5, \pi_{4}=0.5, \pi_5=0.5, \pi_6=0.25, \pi_{7}=0.25, \pi_8=0.25, \pi_9=1/21.$
We consider three scenarios: (a) orthogonal means, equal prior ($\pi_i=1/9$); (b) orthogonal means, different prior; (c) correlated means with pairwise correlation coefficient equal to $0.5$ (i.e.,~$\langle \vct{\mu}_i,\vct{\mu}_j\rangle/(\twonorm{\vct{\mu}_i}\twonorm{\vct{\mu}_j})=0.5$ for $i\neq j$) and different priors as discussed above.
Figure \ref{fig11} shows the test miss-classification errors as a function of $\gamma:=d/n$. In all scenarios our theoretical predictions are a near perfect match to the empirical performance. In scenario (a), class-wise averaging achieves the lowest error as predicted by Proposition \ref{propo:Bayes}. However, in scenario (b) where the means have different norms the averaging method has higher misclassification error compared with CE, LS and WLS for large sample sizes (small $\gamma$). We note that both LS and WLS achieve lower errors compared with CE as the sample size grows. Scenario (c) is similar to (b). However, due to class correlations, the errors are uniformly higher. Figure \ref{fig12} shows the corresponding class-wise miss-classification errors for the smallest $\gamma$ in Figure \ref{fig11} ($\gamma=0.117$). In scenario (a), errors are equal which is expected given the equal class priors. In scenarios (b) and (c) however, due to different priors, large classes 7,8,9 achieve best accuracy. The performance difference is most visible for the averaging approach. LS mitigates this issue to some extent, while WLS creates the flattest class-wise errors suggesting that it can reduce the miss-classification error on small/minority classes.

Figure \ref{fig13} focuses on orthogonal classes with varying number of classes $k$ where $\tn{\mub_i}^2=15$ and $d\in\{50,100,200\}$ with $kd/n=k\gamma$ fixed at $k\gamma=20/11$. It plots the ratio of the empirical error probability and our theoretical prediction as $k$ grows until $k=d$. Two observations are worth mentioning here. (1) The accuracy of our predictions noticeably improves as the problem dimension $d,n$ grow as expected given the asymptotic nature of our analysis. Interestingly, the convergence appears to be noticeably faster (as a function of $d$) for the LS rather than the Averaging classifier. (2) Our theoretical results formally require that $k$ is fixed while $d$ (and $n$) grow large. Yet, the presented experimental results suggest that they might also hold for large $k$ under the shown scaling. This is a fascinating research question that we believe is worth investigating further. 

Figure \ref{fig14} provides experiments on  MLM with $k=9$ orthogonal classes. Unlike GMM, CE achieves the best performance in MLM. In Figure \ref{fig14} (a), classes have same norms $\tn{\mu_i}=10$, while in Figure \ref{fig14} (b) we have quadrupled the norms of classes 7,8,9 and doubled the norms of classes 4,5,6. This disparity between the norms seems to help improve the CE accuracy, but hurt LS/averaging accuracy for small $\gamma$. Finally, Figure \ref{fig14} (c) shows the class-wise probability of error associated with (b) for $\gamma=0.117$ and demonstrates that LS outperforms averaging. 



\begin{figure}[t!]
	\centering
	\begin{subfigure}{1.8in}
	\begin{tikzpicture}
	\node at (0,0) {\includegraphics[scale=0.22]{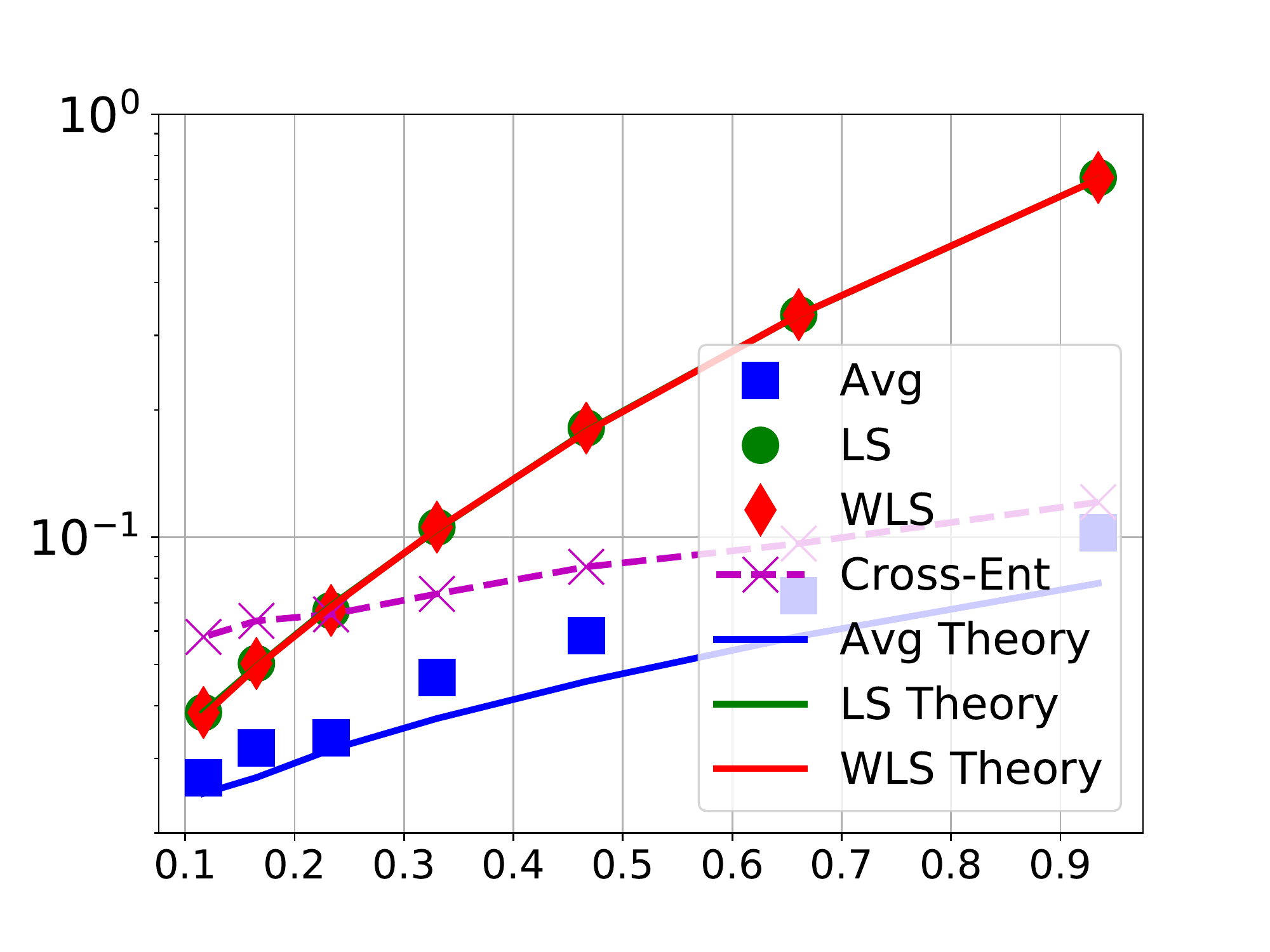}};
	\node at (-2.35,0) [rotate=90,scale=1.]{Prob. of Error};
	\node at (0,-1.68) [scale=1.]{$\gamma$};
	\end{tikzpicture}
	\label{fig:covariance}
	\end{subfigure}
	\begin{subfigure}{1.8in}
	\begin{tikzpicture}
	\node at (0,0) {\includegraphics[scale=0.22]{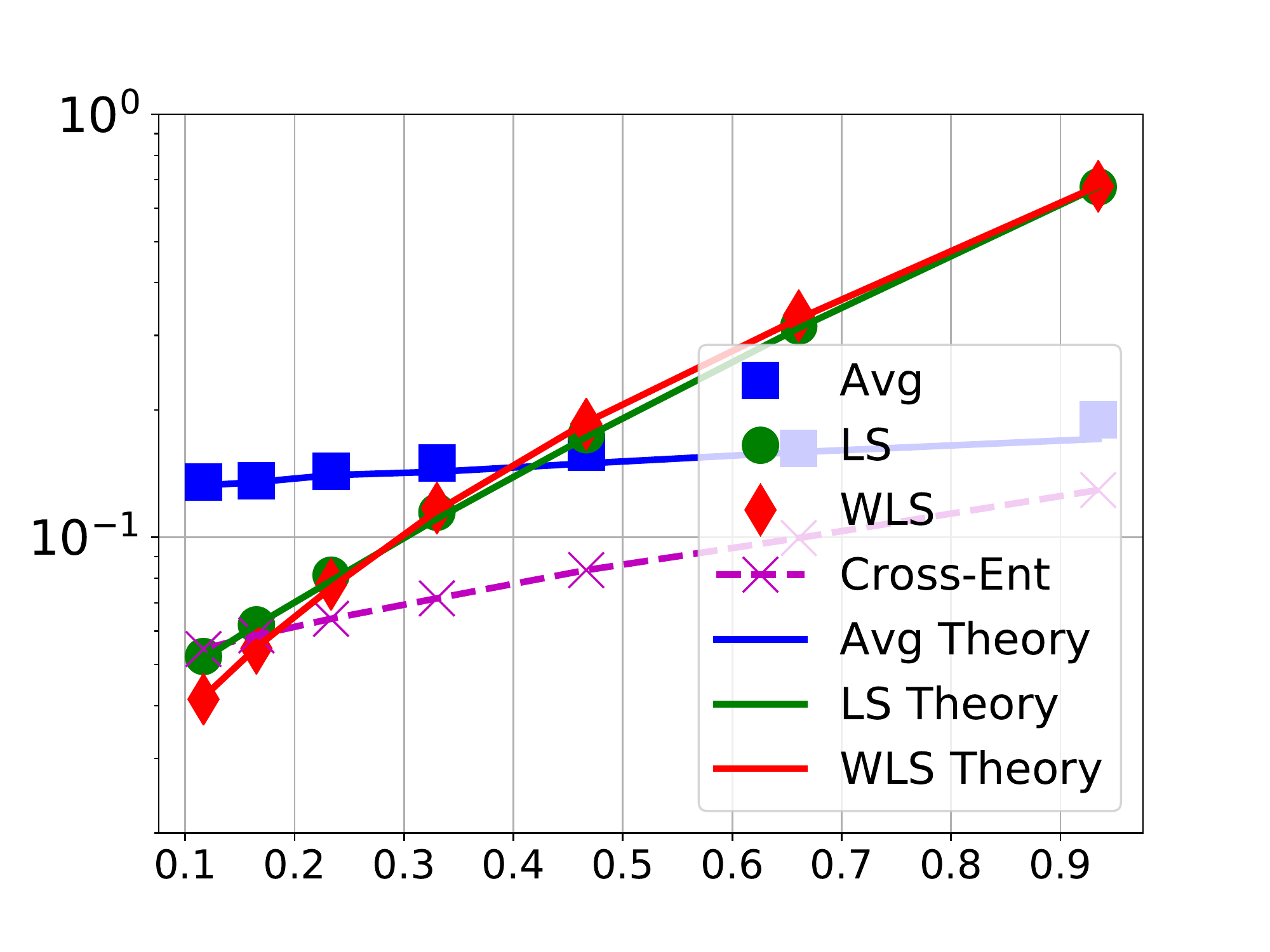}};
	\node at (-2.35,0) [rotate=90,scale=1.]{};
	\node at (0,-1.68) [scale=1.]{$\gamma$};
	\end{tikzpicture}
	\end{subfigure}
	\begin{subfigure}{1.8in}
	\begin{tikzpicture}
	\node at (0,0) {\includegraphics[scale=0.22]{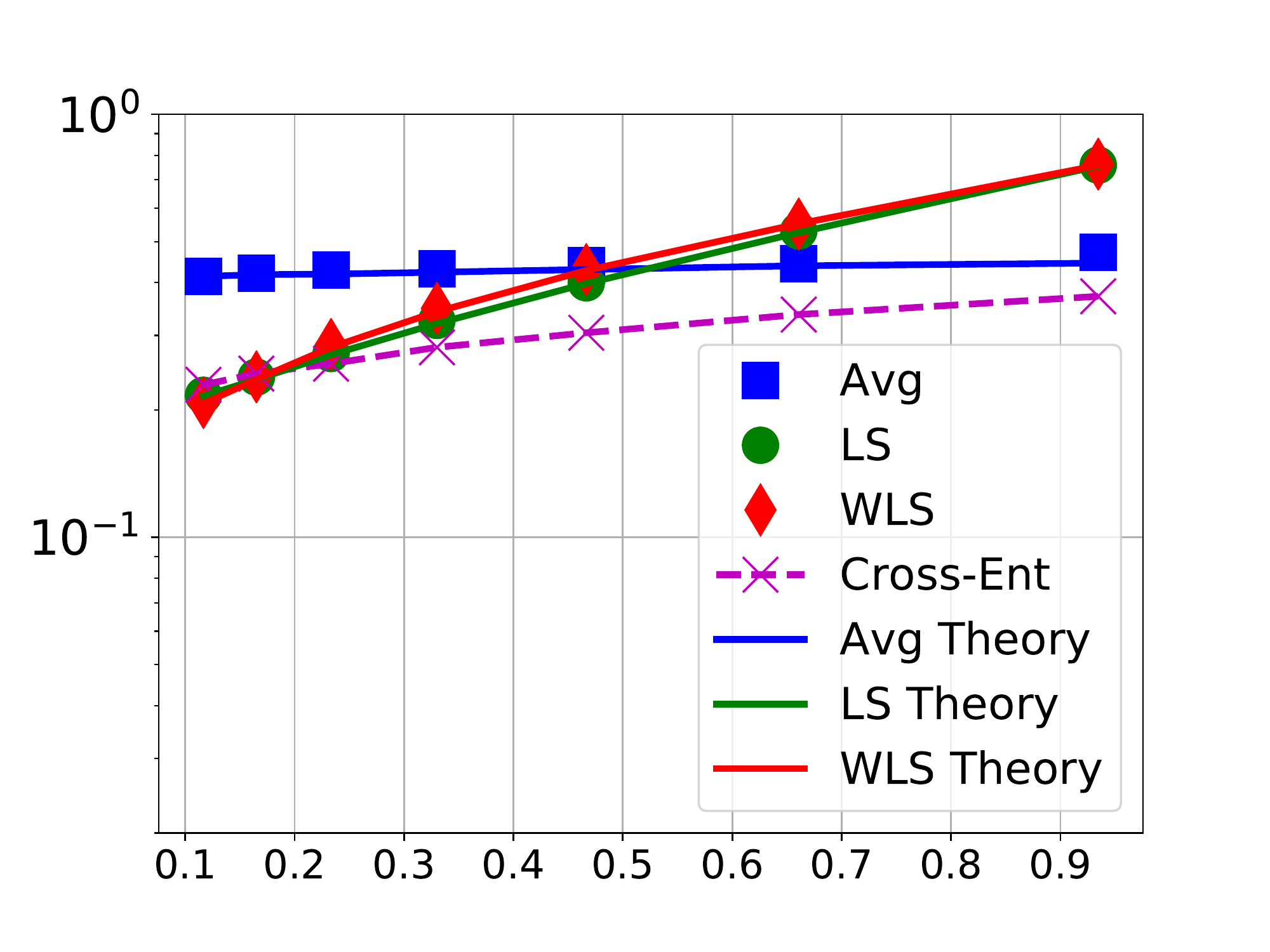}};
	\node at (-2.35,0) [rotate=90,scale=1.]{};
	\node at (0,-1.68) [scale=1.]{$\gamma$};
	\end{tikzpicture}
		\end{subfigure}
		\caption{GMM with $k=9,d=300$. (a) orthogonal, equal prior, (b) orthogonal, different prior, (c) correlated, different prior.}\label{fig11}
\end{figure}
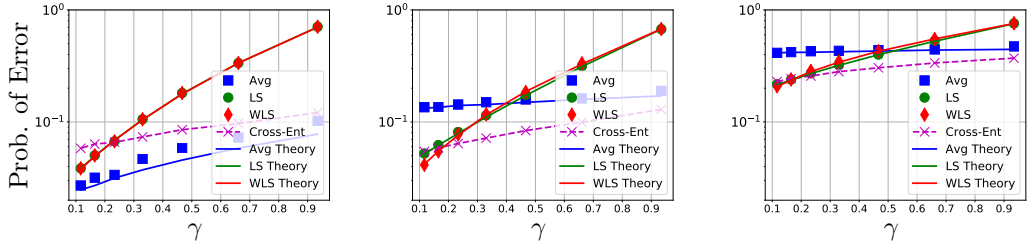

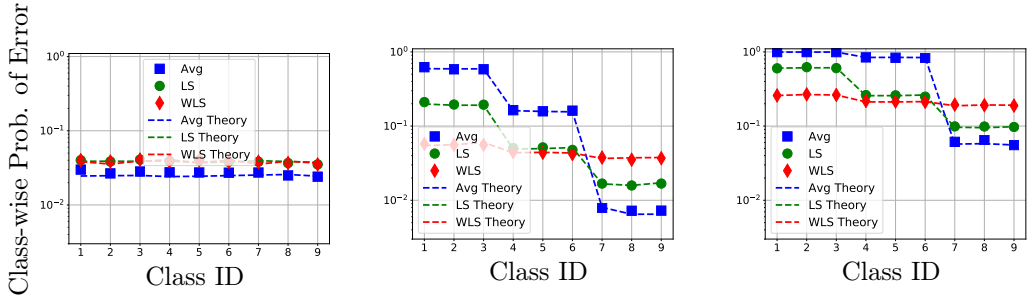
\begin{figure}[t!]
	\centering
	\begin{subfigure}{1.8in}
	\begin{tikzpicture}
	\node at (0,0) {\includegraphics[scale=0.22]{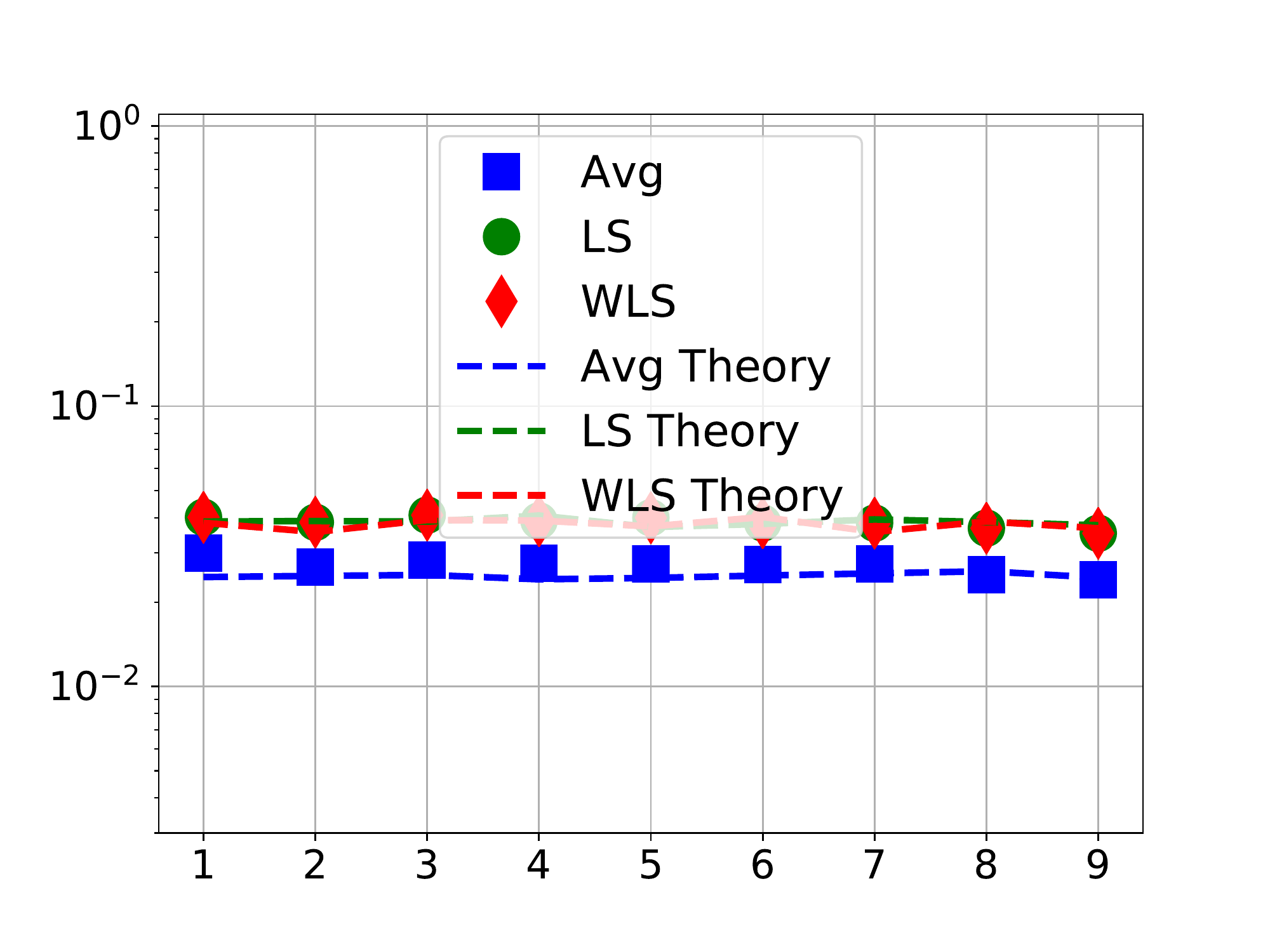}};
	\node at (-2.35,0) [rotate=90,scale=1.]{Class-wise Prob. of Error};
	\node at (0,-1.68) [scale=1.]{Class ID};
	\end{tikzpicture}
	\end{subfigure}
	\begin{subfigure}{1.8in}
	\begin{tikzpicture}
	\node at (0,0) {\includegraphics[scale=0.22]{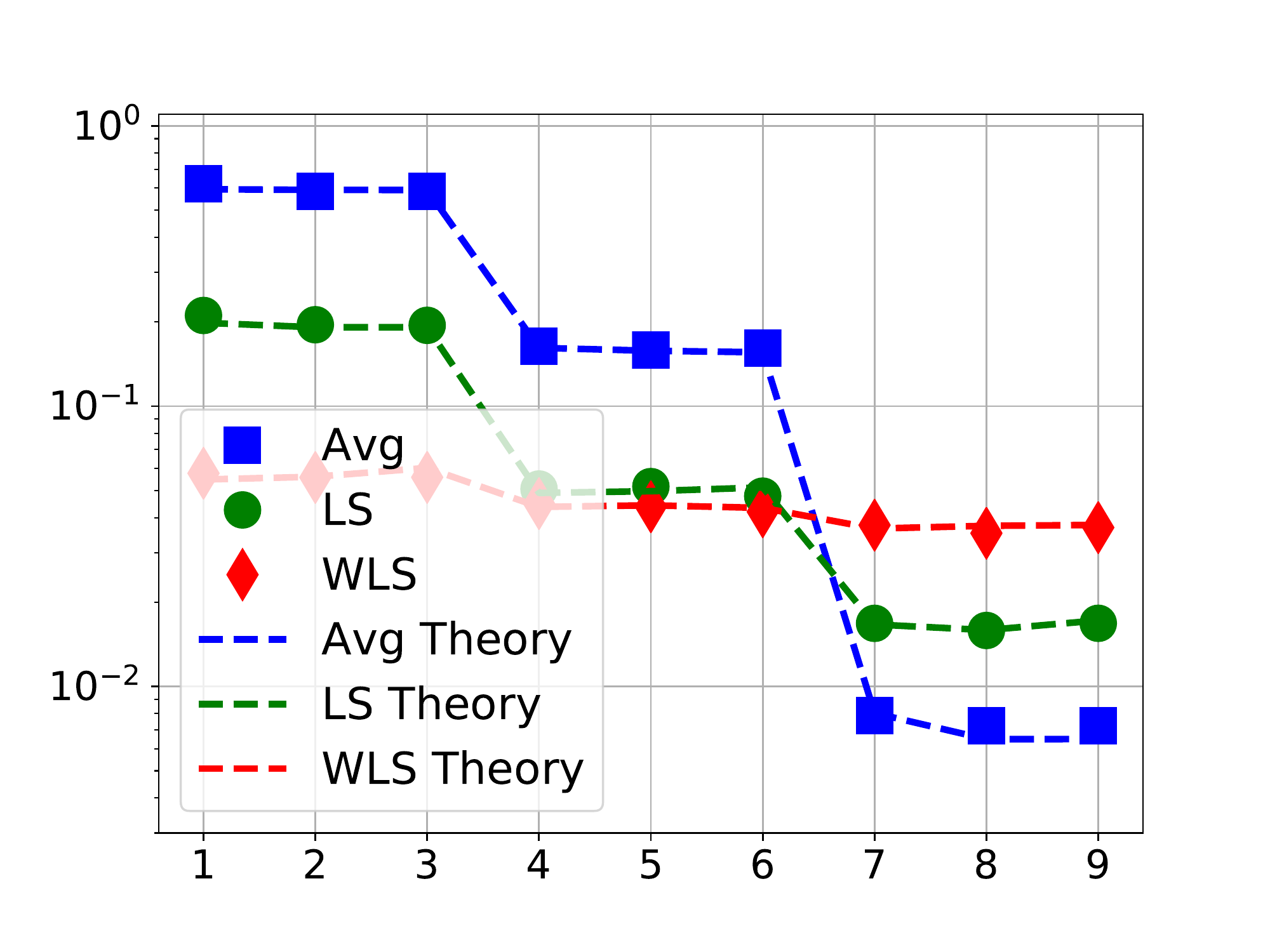}};
	\node at (-2.35,0) [rotate=90,scale=1.]{};
	\node at (0,-1.68) [scale=1.]{Class ID};
	\end{tikzpicture}
	\label{fig:positive}
	\end{subfigure}
	\begin{subfigure}{1.8in}
	\begin{tikzpicture}
	\node at (0,0) {\includegraphics[scale=0.22]{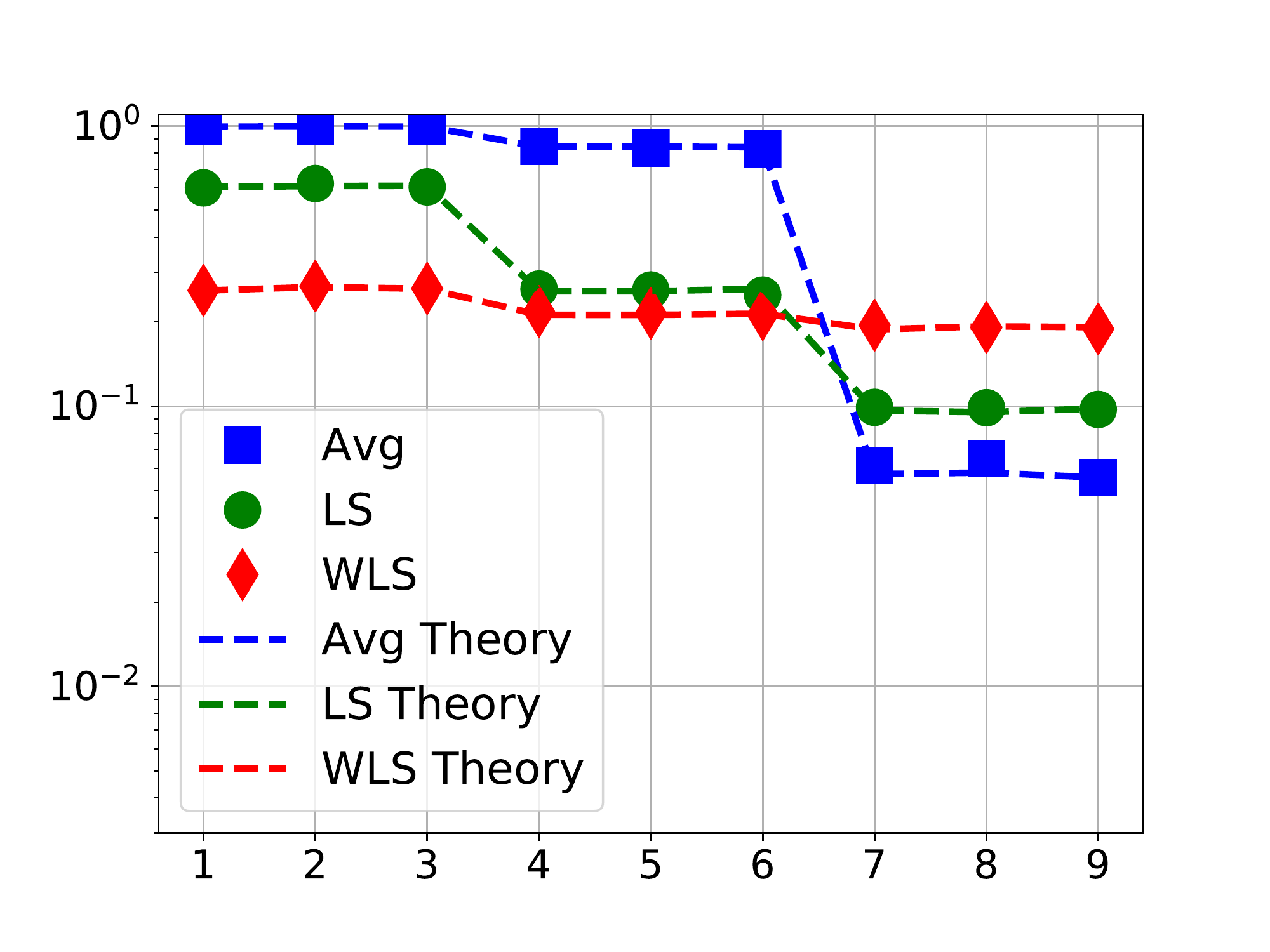}};
	\node at (-2.35,0) [rotate=90,scale=1.]{};
	\node at (0,-1.68) [scale=1.]{Class ID};
	\end{tikzpicture}
		\end{subfigure}
		\caption{Class-wise probability of errors corresponding to Figure \ref{fig11} with $\gamma=0.117$.}\label{fig12}
\end{figure}
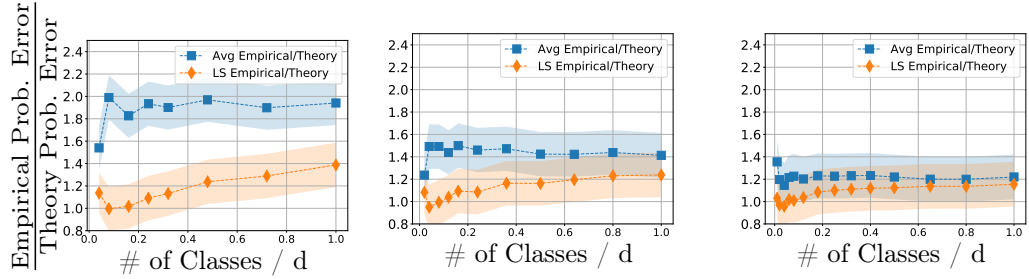
\begin{figure}[t!]
	\centering
	\begin{subfigure}{1.8in}
	\begin{tikzpicture}
	\node at (0,0) {\includegraphics[scale=0.22]{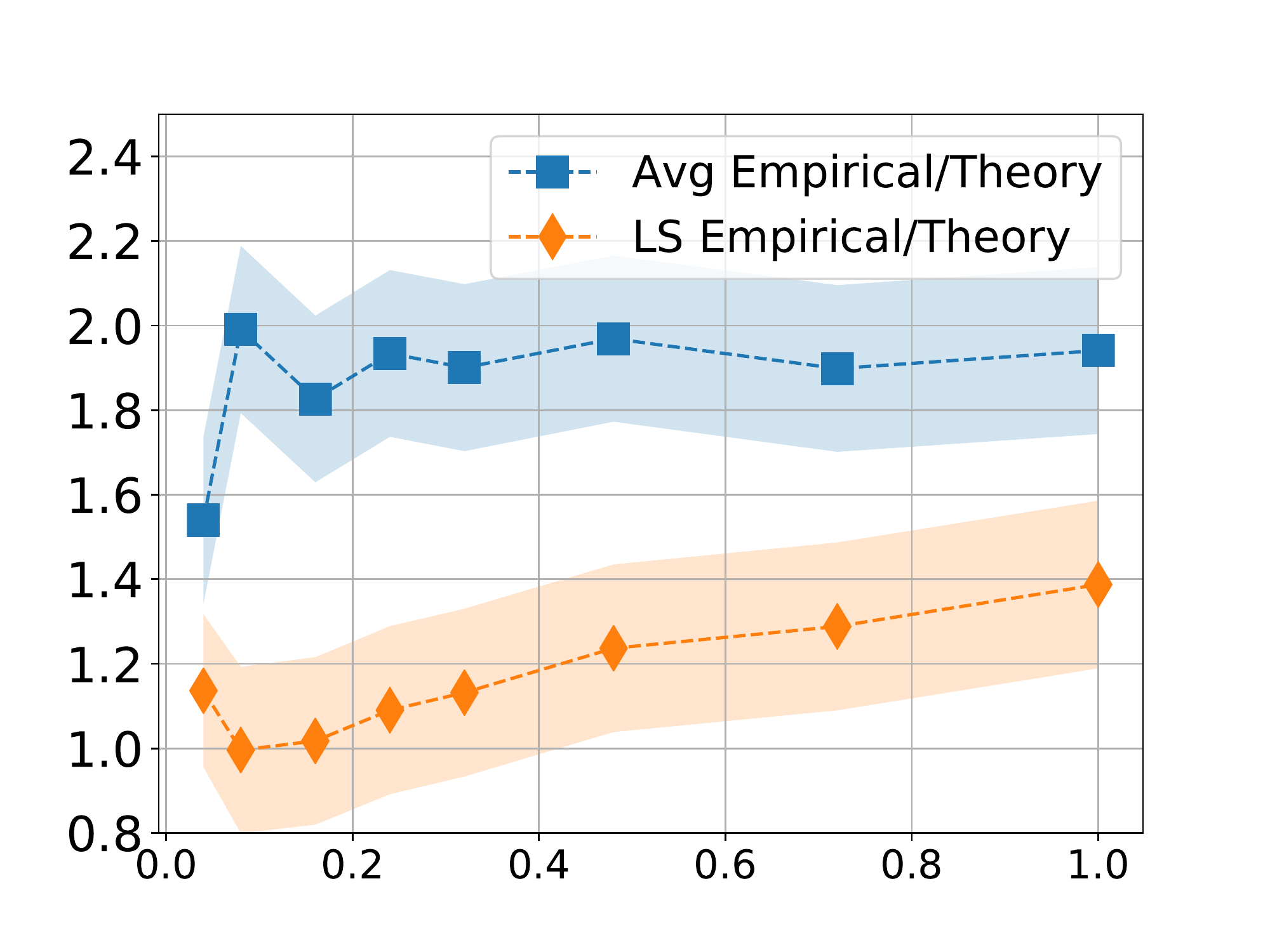}};
	\node at (-2.35,0) [rotate=90,scale=1.3]{$\frac{\text{Empirical Prob. Error}}{\text{Theory Prob. Error}}$};
	\node at (0,-1.68) [scale=1.]{\# of Classes / d};
	\end{tikzpicture}
	\end{subfigure}
	\begin{subfigure}{1.8in}
	\begin{tikzpicture}
	\node at (0,0) {\includegraphics[scale=0.22]{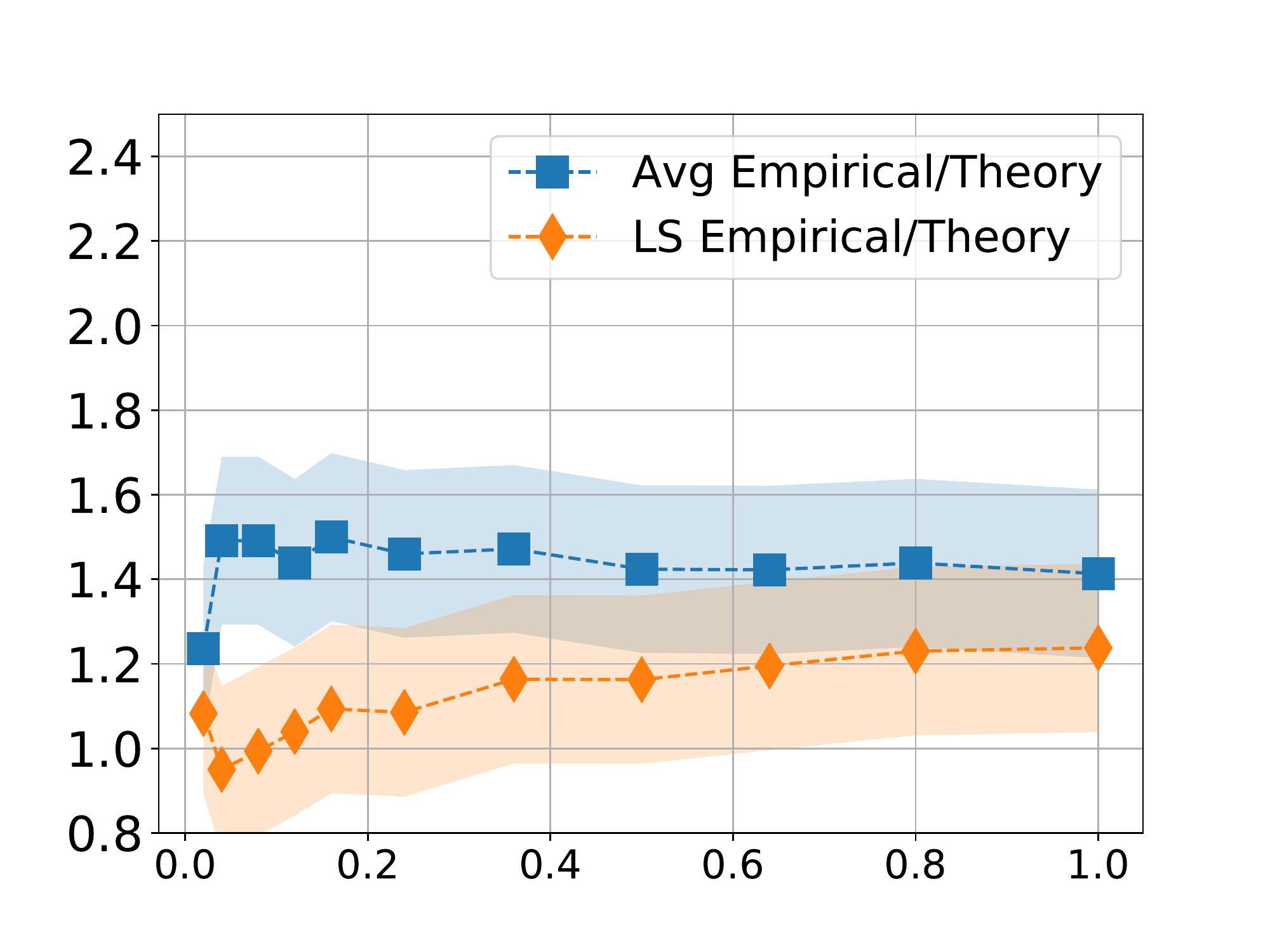}};
	\node at (-2.35,0) [rotate=90,scale=1.]{};
	\node at (0,-1.68) [scale=1.]{\# of Classes / d};
	\end{tikzpicture}
	\end{subfigure}
	\begin{subfigure}{1.8in}
	\begin{tikzpicture}
	\node at (0,0) {\includegraphics[scale=0.22]{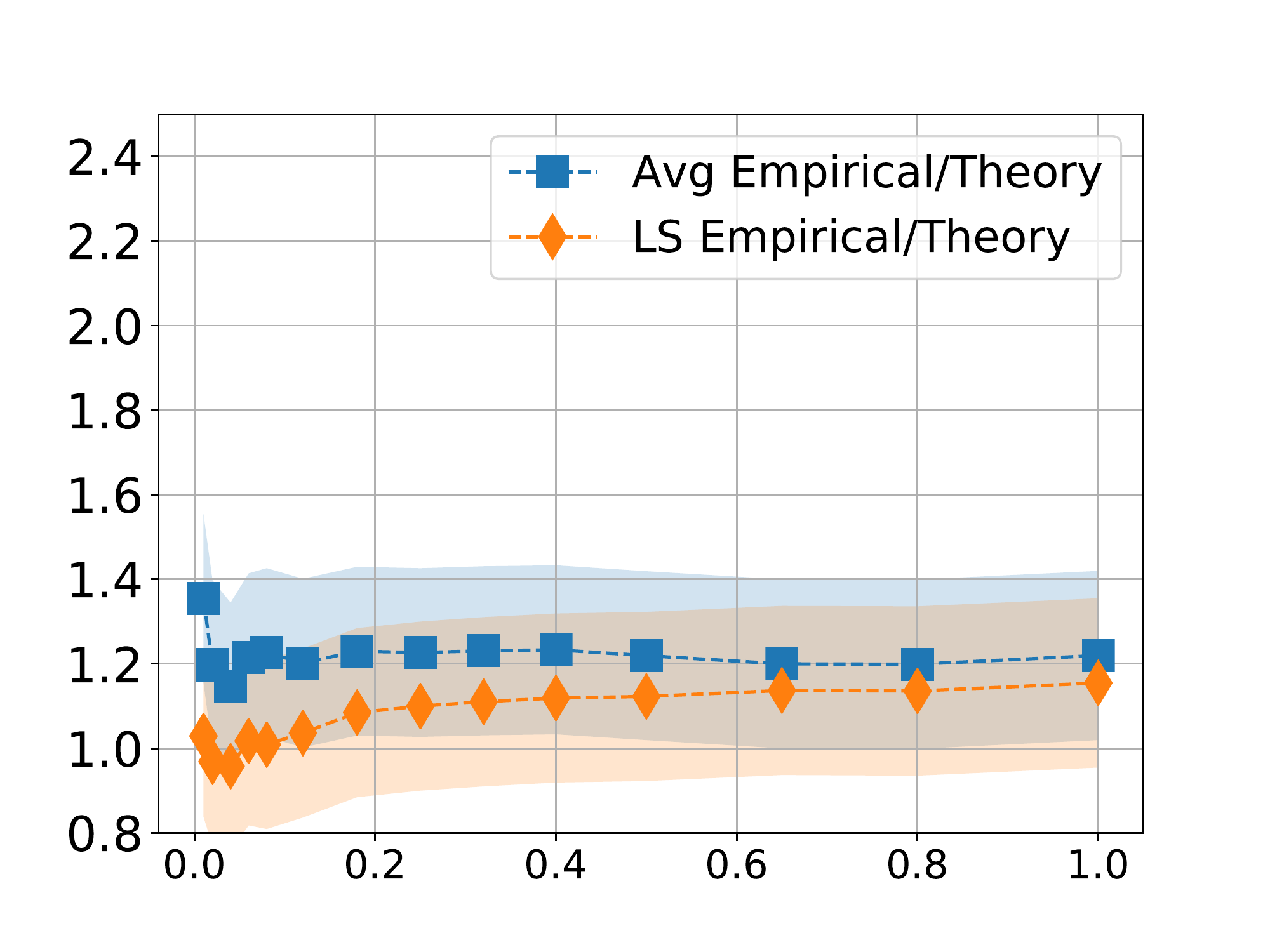}};
	\node at (-2.35,0) [rotate=90,scale=1.]{};
	\node at (0,-1.68) [scale=1.]{\# of Classes / d};
	\end{tikzpicture}
		\end{subfigure}
		\caption{GMM, $K/d$ is varied from $0$ to $1$ while keeping $K\gamma$ constant for (a) $d=50$, (b) $d=100$, (c) $d=200$.}\label{fig13}
		\vspace{-0.5cm}
\end{figure}
\begin{figure}[t!]
	\centering
	\begin{subfigure}{1.8in}
	\begin{tikzpicture}
	\node at (0,0) {\includegraphics[scale=0.22]{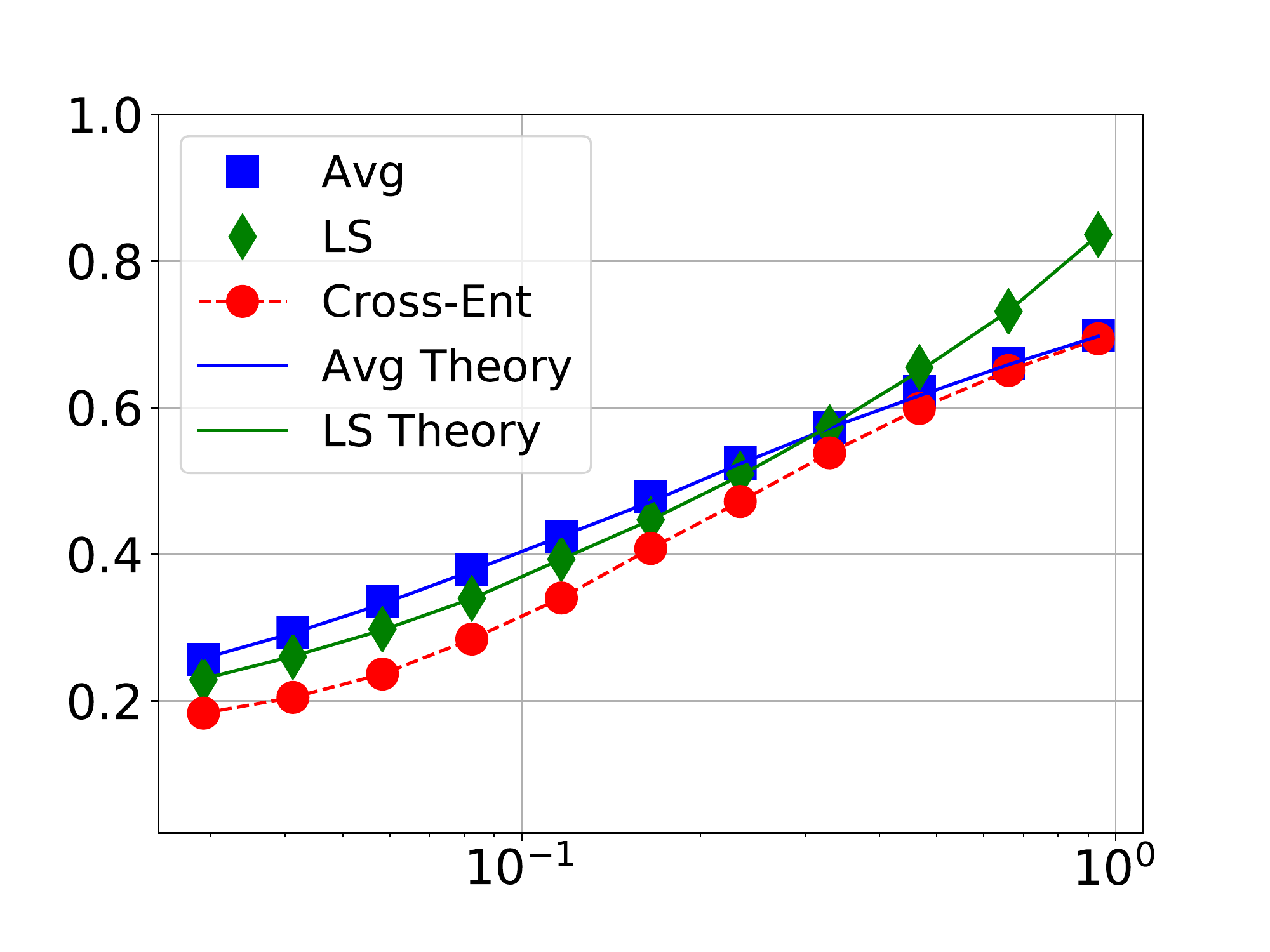}};
	\node at (-2.35,0) [rotate=90,scale=1.]{Prob. of Error};
	\node at (0,-1.68) [scale=1.]{$\gamma$};
	\end{tikzpicture}
	\end{subfigure}
	\begin{subfigure}{1.8in}
	\begin{tikzpicture}
	\node at (0,0) {\includegraphics[scale=0.22]{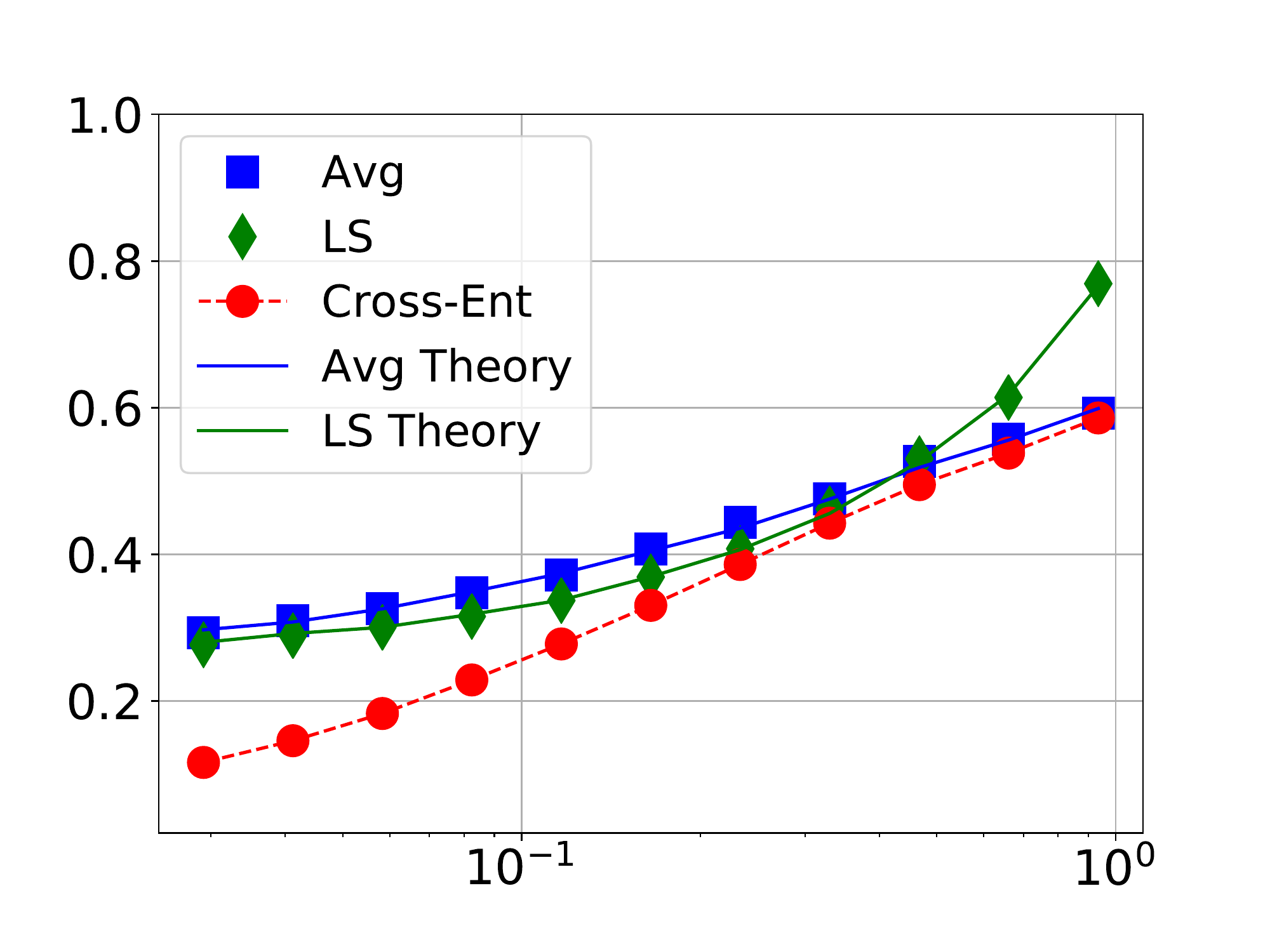}};
	\node at (-2.35,0) [rotate=90,scale=1.]{};
	\node at (0,-1.68) [scale=1.]{$\gamma$};
	\end{tikzpicture}
	\end{subfigure}
	\begin{subfigure}{1.8in}
	\begin{tikzpicture}
	\node at (0,0) {\includegraphics[scale=0.22]{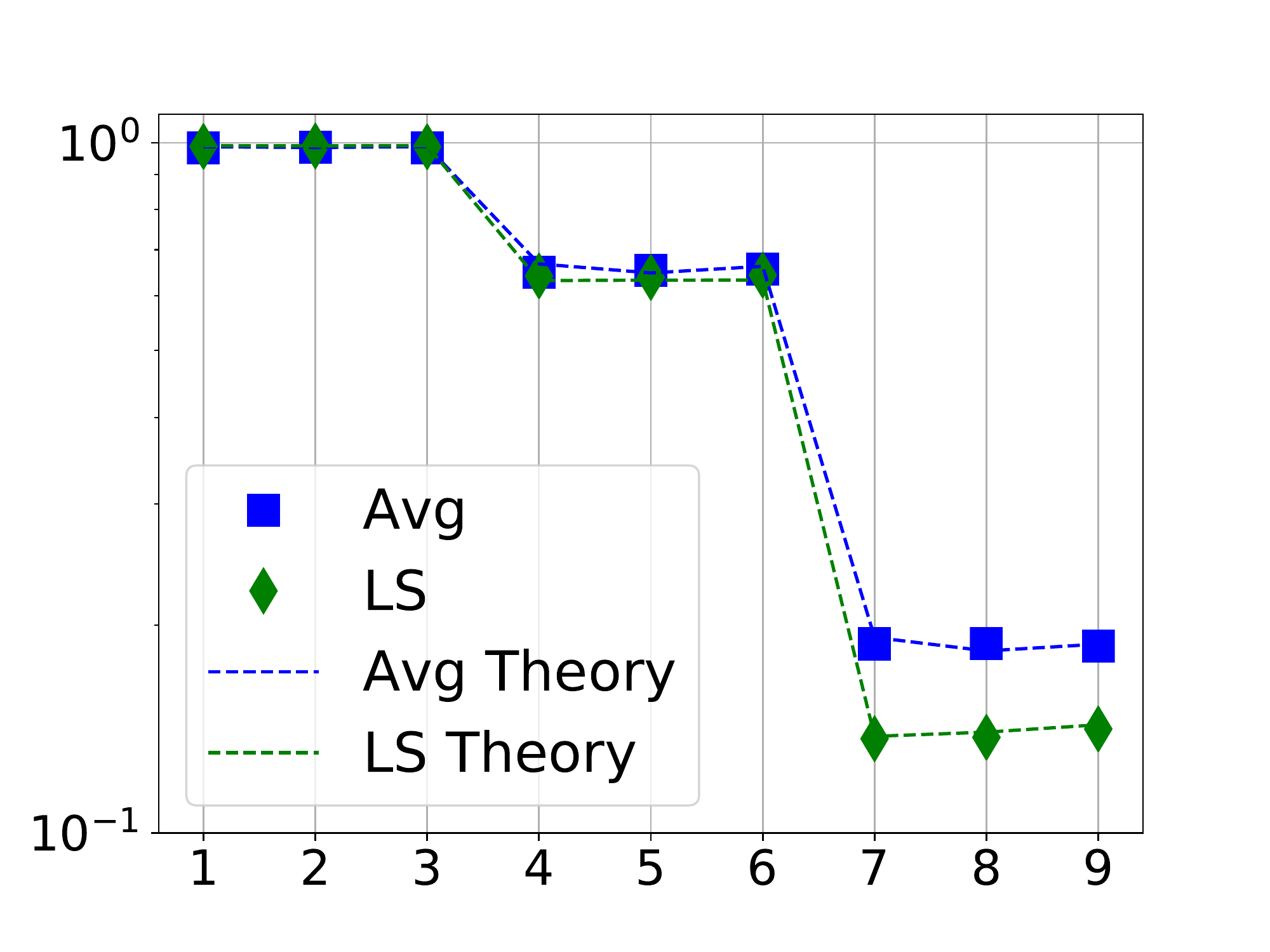}};
	\node at (-2.35,0) [rotate=90,scale=1.]{Class-wise Prob. of Error};
	\node at (0,-1.68) [scale=1.]{Class ID};
	\end{tikzpicture}
		\end{subfigure}
		\caption{MLM with orthogonal means for (a) equal norms and (b) different norms. (c) Class-wise probability errors for (b).}\label{fig14}
\end{figure}
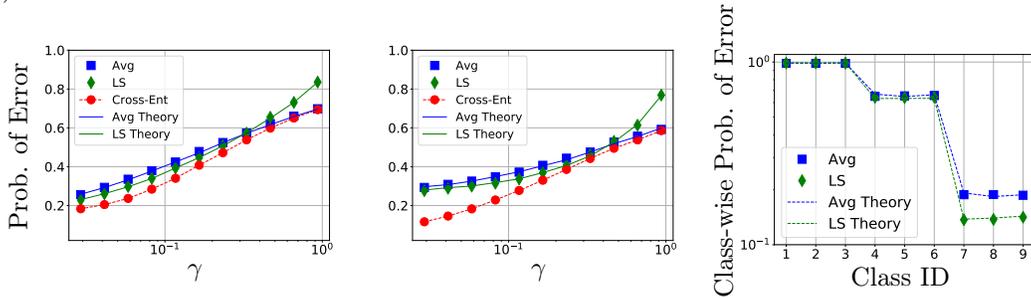

\section{Proof outline for least-squares: key ideas and challenges}\label{sec:key2}
In this section, we provide a proof sketch for the analysis of the multiclass least-squares (LS) classifier. 

Specifically, we discuss our approach towards specifying the high-dimensional limits of the key quantities needed to evaluate the classification error:
$
\bb, \Sigmab_{\w,\mub},
$ and, $ \Sigmab_{\w,\w}$. For simplicity, we focus here on the performance of the LS classifier GMM. We note that our proofs for the MLM and the Weighted Least-Squares (WLS) classifiers follow the same general strategy, but in some parts require more involved and intricate analysis and derivations. Our proof follows the following general steps; see the appendix for complete details and derivations.


\vp
\noindent\textbf{Step I: Decomposing the loss across classes.} Recall from Section \ref{classalg} that the multiclass LS classifier produces a linear classifier $\vct{x}\mapsto \mtx{W}\vct{x}+\vct{b}$ via a least-squares fit to the training data:
\begin{align}
(\Wh,\bh):=\frac{1}{2n}\fronorm{\mtx{W}\mtx{X}+\vct{b}\vct{1}_n^T-\mtx{Y}}^2.  \label{eq:LS_def_app}
\end{align}
Notice that the objective function above is separable. That is, 
\begin{align*}
\frac{1}{2n}\fronorm{\mtx{W}\mtx{X}+\vct{b}\vct{1}_n^T-\mtx{Y}}^2=\frac{1}{2n}\sum_{\ell=1}^k\twonorm{\mtx{X}^T\vct{w}_\ell+b_\ell\vct{1}_n-\mtx{Y}_\ell}^2.
\end{align*}
 Hence, for each $\ell\in[k]$, 
\bea
(\wh_\ell,\widehat{b}_\ell) = \arg\min_{\vct{w}_\ell, {b}_\ell}\text{ }\frac{1}{2n}\twonorm{\mtx{X}^T\vct{w}_\ell+b_\ell\vct{1}_n-\mtx{Y}_\ell}^2. \label{eq:LS_sep}
\end{align}
This decomposition is convenient for analysis as it is easier to compute the statistical properties of the simple single-output LS in \eqref{eq:LS_sep} compared to the multi-output objective in \eqref{eq:LS_def_app}. Indeed, as we show, in Step III, this simplification will eventually allow us to compute the high-dimensional behavior of the following key quantities for all $\ell\in[k]$: (i) the intercept $\bh_\ell$, (ii) the mean-correlations $\inp{\wh_\ell}{\mub_c},~c\in[k]$, (iii) the norm $\twonorm{\wh_\ell}$. 

\vp\noindent\textbf{Step II: Reduction to an Auxiliary Optimization (AO) problem via CGMT.} To calculate the high-dimensional statistical behavior of  \eqref{eq:LS_sep} we use the Convex Gaussian min-max Theorem (CGMT) \cite{StoLasso,COLT} framework. We provide a brief introduction of the CGMT machinery in Section \ref{sec:back_CGMT}. Roughly stated, this framework allows us to replace a \emph{Primary Optimization} (PO) problem of the form \eqref{eq:LS_sep} with an \emph{Auxiliary Optimization} (AO) problem that is simpler to analyze, but is predictive of the behavior of the latter. For instance, for the PO in \eqref{eq:LS_sep} in the GMM, after some algebraic manipulations, the AO problem takes the form
\begin{align}
\label{PFAO}
\frac{1}{2}\left(\min_{\vct{w}_\ell, {b}_\ell}\text{ } \frac{1}{\sqrt{n}}\twonorm{ \sigma\twonorm{\vct{w}_\ell}\vct{g}+\mtx{Y}^T\mtx{M}^T\vct{w}_\ell+b_\ell\vct{1}_n-\mtx{Y}_\ell}+\frac{1}{\sqrt{n}}\sigma \vct{h}^T\vct{w}_\ell\right)_{+}^2,
\end{align}
where $(x)_{+}:=\max(0,x)$ and $\vct{g}\in\R^n$ and $\vct{h}\in\R^d$ are two independent Gaussian random vectors distributed as $\mathcal{N}(\vct{0},\mtx{I}_n)$ and $\mathcal{N}(\vct{0},\mtx{I}_d)$.

\vp\noindent\textbf{Step III: Simplification of the AO and computing $\mtx{\Sigma}_{\w,\mub}$ and $\vct{b}$.}
In this step we carry out a series of intricate calculations to further simplify \eqref{PFAO} and characterize its various asymptotic properties. At a high-level, we follow the principled machinery introduced in \cite{COLT,Master}, organizing our analysis in three intermediate steps: (a) Scalarization; (b) Convergence analysis; and (c) Deterministic analysis. We note that each one of these intermediate steps for the multiclass setting is more involved than in previously considered regression and binary classification settings. The detailed derivations are deferred to the Appendix \ref{sec:proof_LS_GM_main}. At the end of this analysis step, we have computed the high-dimensional behavior of the intercepts $\bh_\ell,~\ell\in[k]$, the mean-correlations $\inp{\wh_\ell}{\mub_c},~\ell,c\in[k]$, the norms $\twonorm{\wh_\ell},~\ell\in[k]$ and the LS training loss $ \twonorm{\mtx{X}^T\wh_\ell+\bh_\ell\vct{1}_n-\mtx{Y}_\ell}$. In particular, for GMM these calculations allow us to conclude the following limits for all $\ell\in[k]$:

\begin{align}
\widehat{b}_\ell \rP \pi_\ell\left(1-(\eb_\ell-\pib)^T\Vb\Sigmab\Deltab^{-1}\Sigmab\Vb^T\pib\right), \qquad
\M^T\wh_\ell  \rP  {\pi_\ell}\Vb\Sigmab\Deltab^{-1}\Sigmab\Vb^T\left(\eb_\ell-\pib\right),\quad \label{eq:ell_lims_sketch}
\end{align}
and 
\begin{align}\label{eq:norm_lims_sketch}
\twonorm{\w_\ell}^2 \rP \frac{\gamma}{(1-\gamma)\sigma^2}\pi_\ell(1-\pi_\ell)  + \pi_\ell^2 \left(\eb_\ell-\pib\right)^T\Vb\Sigmab \Deltab^{-1} \left(\Deltab^{-1} - \frac{\gamma}{(1-\gamma)\sigma^2}\Iden_{r}\right) \Sigmab\Vb^T \left(\eb_\ell - \pib \right) \,.
\end{align}
where  $\Deltab:=\sigma^2\Iden_r + \Sigmab\Vb^T\Pb\Vb\Sigmab \succ \zero_{r\times r}$ and $
\Pb := \diag{\pib}-\pib\pib^T$.

Expressing \eqref{eq:ell_lims_sketch} in matrix form leads to \eqref{eq:LS_GMa} in Theorem \ref{thm:LS_GM}. Thus, it remains to prove \eqref{eq:LS_GMb}, i.e., to determine the high-dimensional limit of $\Sigma_{\w,\w}$. Note that \eqref{eq:norm_lims_sketch} already determines the diagonal entries of $\Sigma_{\w,\w}$. However, thus far, our analysis treats the optimization of each classifier $\wh_\ell,~\ell\in[k]$ independently and provides no information for the cross-correlation $\inp{\wh_\ell}{\wh_c},\ell\neq c\in[k]$

\vp\noindent\textbf{Step IV: Computing $\mtx{\Sigma}_{\w,\w}$ and capturing cross-correlations.}
The final and most involved part of our analysis is characterizing the asymptotic behavior of $\mtx{\Sigma}_{w,w}$. To see why this is particularly challenging note that the reduction from \eqref{eq:LS_def_app} to \eqref{eq:LS_sep} ``breaks" the dependence of all 
$\wh_1,\wh_2,\ldots,\wh_k$ on the \emph{same} feature matrix $\X$. Capturing this dependence is crucial in determining  the ``cross-correlations" $\inp{\wh_\ell}{\wh_c},~\ell\neq c$. As noted in Section \ref{sec:test} the matrix $\mtx{\Sigma}_{\w,\w}$ is needed to calculate the class-wise and total miss-classification errors. Unfortunately, the CGMT is \emph{not} directly applicable to the multi-output LS optimization in \eqref{eq:LS_def_app}. Our idea to circumvent this challenge builds on the following simple observation: the vector $\wh_{\ell,c} = \wh_\ell + \wh_c$ is itself the solution to another simple single-output LS problem.
\begin{lemma}\label{lem:joint_LS_sketch}
For $\ell\neq c\in[k]$, let $\wh_\ell$, $\wh_c$ be the $\ell$ and $c$-th row of $\widehat{\mtx{W}}$ which is the solution to the multi-output least-squares minimization \eqref{eq:LS_def_app}. Denote $\wh_{\ell,c}:=\wh_\ell+\wh_c$. Then, $\wh_{\ell,c}$ is a minimizer in the following single-output least-squares problem:
\begin{align}
\wh_{\ell,c} = \arg\min_{\w, b} \frac{1}{2n}\twonorm{\Y_{\ell} + \Y_{c} - \X^T \w - b\one_n }^2\,.\nn
\end{align}
\end{lemma}
Thanks to Lemma \ref{lem:joint_LS_sketch}, we can use the CGMT to characterize the limiting behavior of $\twonorm{\wh_\ell+\wh_c}$. These calculations are similar to (but, in certain cases, such as for weighted least-squares, more involved than) those in Steps II and III above. Now note that an asymptotic characterization of $\twonorm{\wh_\ell+\wh_c}$ immediately yields the asymptotic characterization of $\inp{\wh_\ell}{\wh_c}$ as
\begin{align}\label{eq:inp_from_norms}
\inp{\wh_\ell}{\wh_c} = \frac{\twonorm{\wh_\ell+\wh_c}^2-\twonorm{\wh_\ell}^2-\twonorm{\wh_c}^2}{2},
\end{align} 
and $\twonorm{\wh_\ell}, \twonorm{\wh_c}$ are already computed in Step IV (cf. \eqref{eq:norm_lims_sketch}). For the GMM, the analysis in this step allow us to calculate the asymptotic behavior of $\mtx{\Sigma}_{w,w}$ as promised in \eqref{eq:LS_GMb} in Theorem \ref{thm:LS_GM}:
\begin{align*}
\Sigmab_{\w,\w} &\rP\frac{\gamma}{(1-\gamma)\sigma^2} \Pb+ \Pb\Vb\Sigmab \Deltab^{-1} \Big(\Deltab^{-1} - \frac{\gamma}{(1-\gamma)\sigma^2}\Iden_r\Big) \Sigmab\Vb^T \Pb.
\end{align*}

\subsection{Background on the CGMT}\label{sec:back_CGMT}

The CGMT is an extension of Gordon's Gaussian min-max inequality (GMT) \cite{Gor88}. In the context of high-dimensional inference problems, Gordon's inequality was first successfully used in the study oh sharp phase-transitions in noiseless Compressed Sensing \cite{Sto,Cha,TroppEdge,Sto}. More recently, \cite{StoLasso} (see also \cite[Sec.~10.3]{TroppEdge}) discovered that Gordon's inequality is essentially tight for certain convex problems. A concrete and general formulation of this idea was given by \cite{COLT} and was called the CGMT. 


In order to summarize the essential ideas, consider the following two Gaussian processes:
\begin{subequations}\label{eq:POAO}
\begin{align}
X_{\w,\ub} &:= \ub^T \G \w + \psi(\w,\ub),\label{eq:PO_obj}\\
Y_{\w,\ub} &:= \twonorm{\w} \g^T \ub + \twonorm{\ub} \h^T \w + \psi(\w,\ub),\label{eq:AO_obj}
\end{align}
\end{subequations}
where: $\G\in\mathbb{R}^{n\times d}$, $\g \in \mathbb{R}^n$, $\h\in\mathbb{R}^d$, they all have entries iid Gaussian; the sets $\mathcal{S}_{\w}\subset\R^d$ and $\mathcal{S}_{\ub}\subset\R^n$ are compact; and, $\psi: \mathbb{R}^d\times \mathbb{R}^n \to \mathbb{R}$. For these two processes, define the following (random) min-max optimization programs, which are refered to as the \emph{primary optimization} (PO) problem and the \emph{auxiliary optimization} AO:
\begin{subequations}
\begin{align}\label{eq:PO_loc}
\Phi(\G)&=\min\limits_{\w \in \mathcal{S}_{\w}} \max\limits_{\ub\in\mathcal{S}_{\ub}} X_{\w,\ub},\\
\label{eq:AO_loc}
\phi(\g,\h)&=\min\limits_{\w \in \mathcal{S}_{\w}} \max\limits_{\ub\in\mathcal{S}_{\ub}} Y_{\w,\ub}.
\end{align}
\end{subequations}

%
%
If the sets $\mathcal{S}_{\w}$ and $\mathcal{S}_{\ub}$ are convex and \emph{bounded}, and $\psi$ is continuous \emph{convex-concave} on $\mathcal{S}_{\w}\times \mathcal{S}_{\ub}$, then, for any $\nu \in \mathbb{R}$ and $t>0$, it holds \cite[Thm.~3]{COLT}:
\begin{equation}\label{eq:cgmt}
\mathbb{P}\left( \abs{\Phi(\G)-\nu} > t\right) \leq 2\,\mathbb{P}\left(  \abs{\phi(\g,\h)-\nu} > t \right).
\end{equation}
In words, concentration of the optimal cost of the AO problem around $q^\ast$ implies concentration of the optimal cost of the corresponding PO problem around the same value $q^\ast$.  Asymptotically, if we can show that $\phi(\g,\h)\rP q^\ast$, then we can conclude that $\Phi(\G)\rP q^\ast$. Moreover, starting from \eqref{eq:cgmt} and under appropriate strict convexity conditions, the CGMT shows that concentration of the optimal solution of the AO problem implies concentration of the optimal solution of the PO around the same value. For example, if minimizers of \eqref{eq:AO_loc} satisfy $\twonorm{\w_\phi(\g,\h)} \rP \alpha^\ast$ for some $\alpha^\ast>0$, then, the same holds true for the minimizers of \eqref{eq:PO_loc}: $\twonorm{\w_\Phi(\G)} \rP \alpha^\ast$. Thus, one can analyze the AO to infer corresponding properties of the PO, the premise being of course that the former is simpler to handle than the latter. 

In \cite{Master}, the authors introduce a principled machinery that allows to (a) express a quite general family of convex inference optimization problems  in the form of the PO and (b) properly analyze the corresponding AO. In particular, the analysis of the AO is performed in three intermediate steps. First, the (random) optimization over vector variables is simplified to an easier optimization over only few scalar variables, termed the ``{scalarized AO}". After the scalarization step, it is possible to establish (uniform) convergence of the scalarized AO to a deterministic min-max optimization problem over only a few scalar variables. The convergence step is followed by the analysis of the latter deterministic problem, which leads to the desired asymptotic characterizations. Our proofs outlined in Section \ref{sec:key2} follow this general strategy, but the new idea introduced in Step IV therein is key to capture the asymptotic behavior of the off-diagonal entries of 	$\Sigmab_{\w\w}$.


\section{Future Directions}

This work aims at initiating a precise asymptotic study of multiclass classifiers that provides a promising setting for resolving a rich set of open questions regarding the (comparative) performance of classification algorithms as a function of the involved problem variables.
 As mentioned, even understanding the statistical performance of one-vs-all multiclass classifiers does not follow directly from the existing literature on binary classifiers. Extending the results of this paper to the one-vs-all logistic and SVM classifiers would allow for a principled comparison among these different choices. A possibly more challenging, albeit mathematically intriguing and practically relevant task, is characterizing the asymptotics of more complicated  (non-separable) losses, such as the cross-entropy loss. For this, even characterizing the asymptotic behavior of the correlations $\Sigmab_{\w,\mub}$ requires new ideas. The previously mentioned study of ``extreme multiclass classification" in which the number of classes $k$ is very large is another fascinating direction. 

%
%

\section*{Acknowledgments}
 C.~Thrampoulidis is partially supported by the NSF under Grant Numbers CCF-2009030 and HDR-1934641. S.~Oymak is partially supported by the NSF award CNS-1932254. M.~Soltanolkotabi is supported by the Packard Fellowship in Science
and Engineering, a Sloan Research Fellowship in Mathematics, an NSF-CAREER under award
$\#1846369$, the Air Force Office of Scientific Research Young Investigator Program (AFOSR-YIP)
under award $\#$FA$9550-18-1-0078$, DARPA Learning with Less Labels (LwLL) and FastNICS programs, and NSF-CIF awards $\#1813877$ and $\#2008443$. 

%

\bibliography{compbib}

\newpage
\appendix
\addtocontents{toc}{\protect\setcounter{tocdepth}{3}}
\tableofcontents

\newpage
\section{Additional Numerical Results}
\label{sec:additionalnum}

\begin{figure}[b!]
	\centering
\centering
\begin{tikzpicture}[scale=0.9,every node/.style={scale=0.9}]
		\node at (0,0) {\includegraphics[scale=0.45]{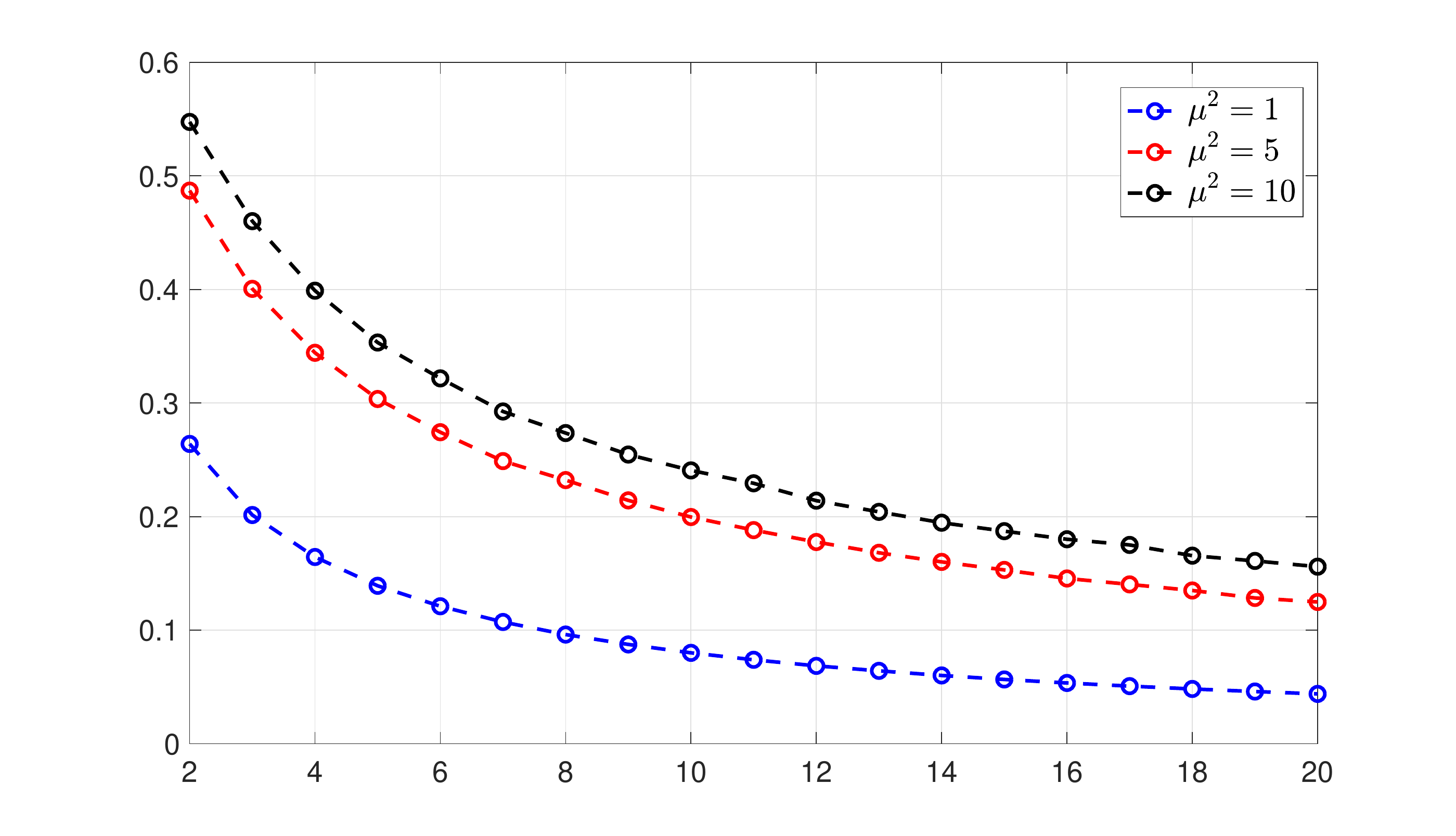}};
		\node at (-5.58,0) [rotate=90,scale=1.]{Transition point $\gamma_\star$};
		\node at (0,-3.5) [scale=1]{Number of classes $k$};
		\end{tikzpicture} \vspace{-10pt}
\caption{\small{The threshold $\gamma_\star$ of Proposition \ref{propo:gamma_star} as a function of the number of classes $k$ and the means' energy $\mu$. LS provably outperforms class-averaging for $\gamma<\gamma_\star$.}}
\label{fig:gamma_star}
\end{figure}

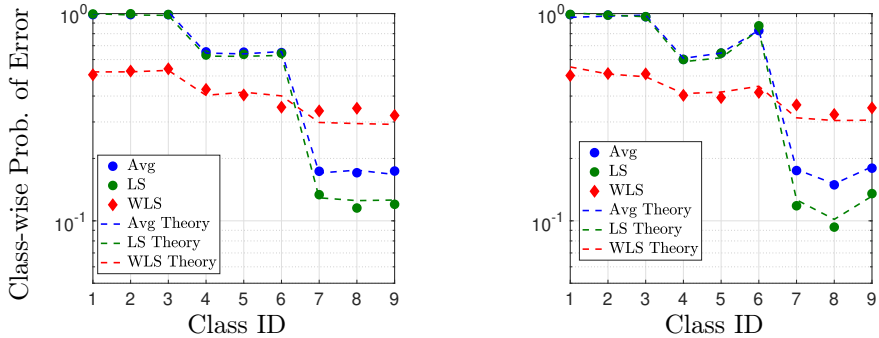
\begin{figure}[b!]
	\centering
	\begin{subfigure}{2.5in}
	\begin{tikzpicture}
	\node at (0,0) {\includegraphics[scale=0.3]{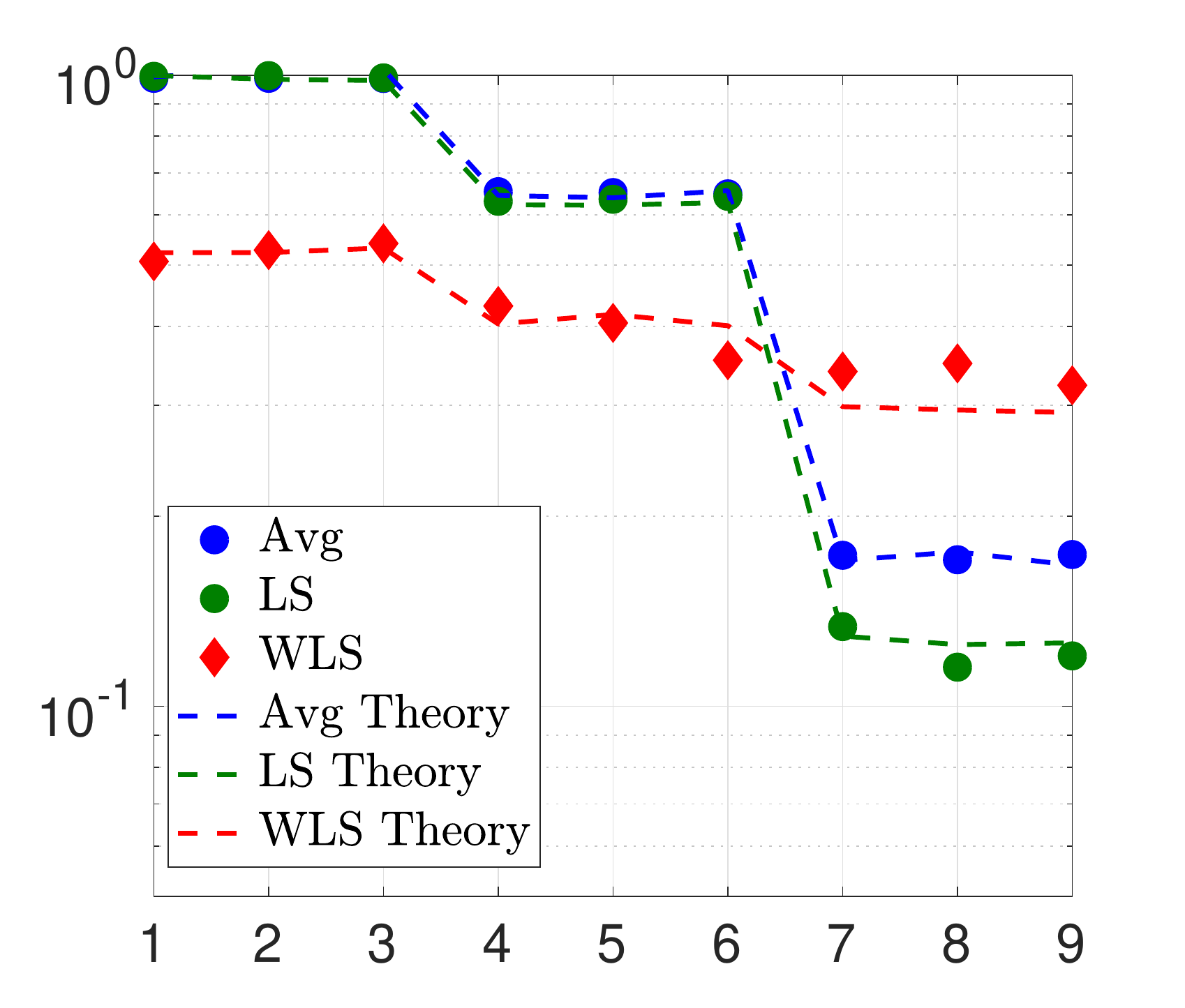}};
	\node at (-2.90,0) [rotate=90,scale=1]{Class-wise Prob. of Error};
	\node at (0,-2.26) [scale=1]{Class ID};
	\end{tikzpicture}\vspace{-5pt}
	\end{subfigure}
	\begin{subfigure}{2.5in}
	\begin{tikzpicture}
	\node at (0,0) {\includegraphics[scale=0.3]{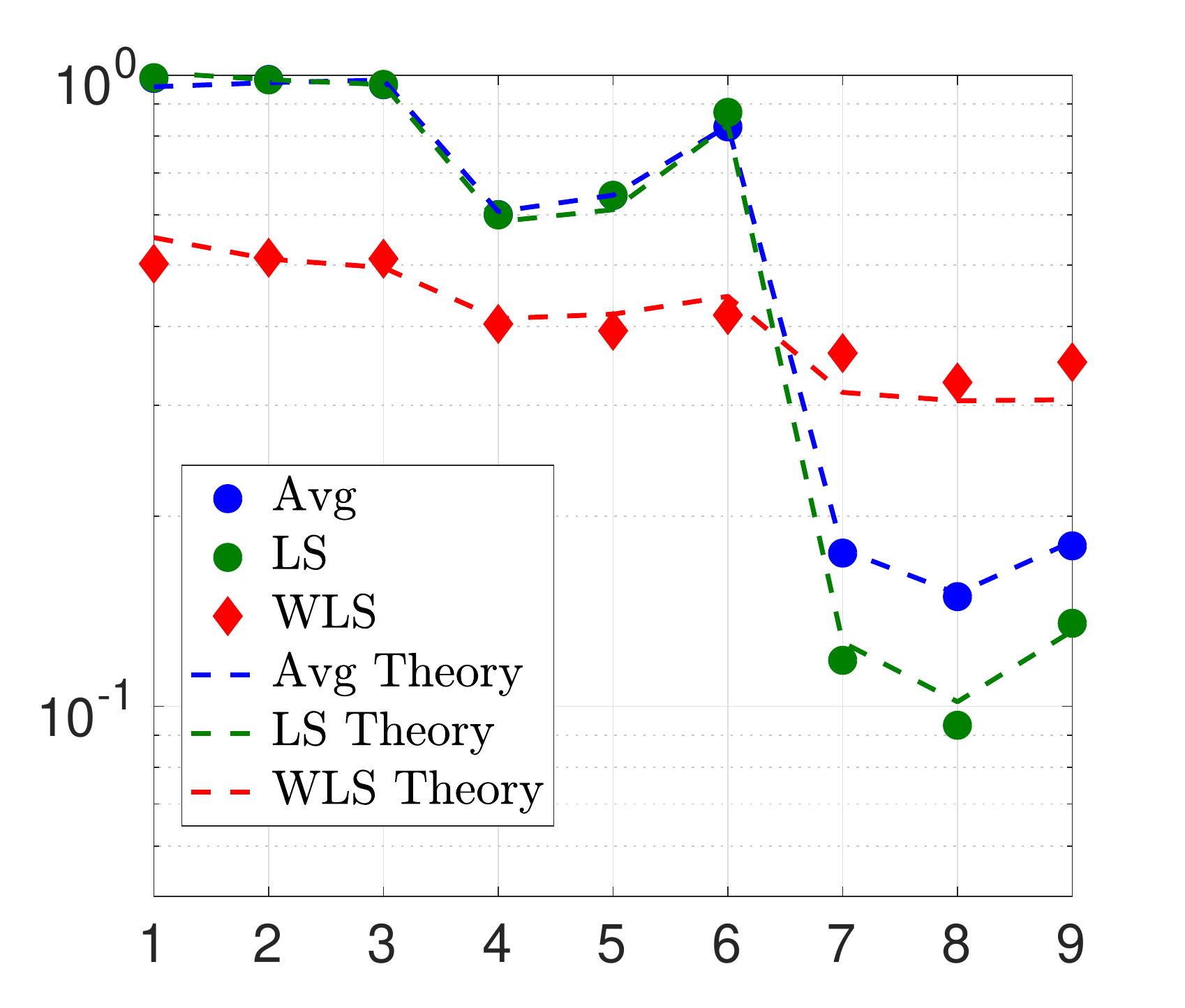}};
	\node at (-2.90,0) [rotate=90,scale=1.]{};
	\node at (0,-2.26) [scale=1.]{Class ID};
	\end{tikzpicture}\vspace{-5pt}
	\end{subfigure}
		\caption{
		Class-wise probabilities of error for MLM with (a) orthogonal means and (b) correlated means. 
		}\label{fig:MLM_class}
\end{figure}

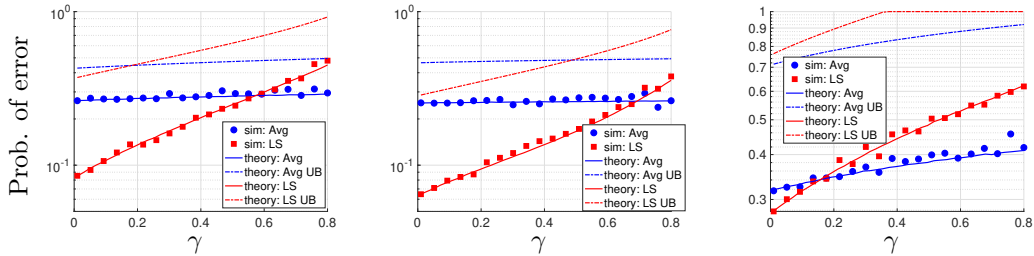
\begin{figure}[b!]
	\centering
	\begin{subfigure}{1.8in}
	\begin{tikzpicture}
	\node at (0,0) {\includegraphics[scale=0.22]{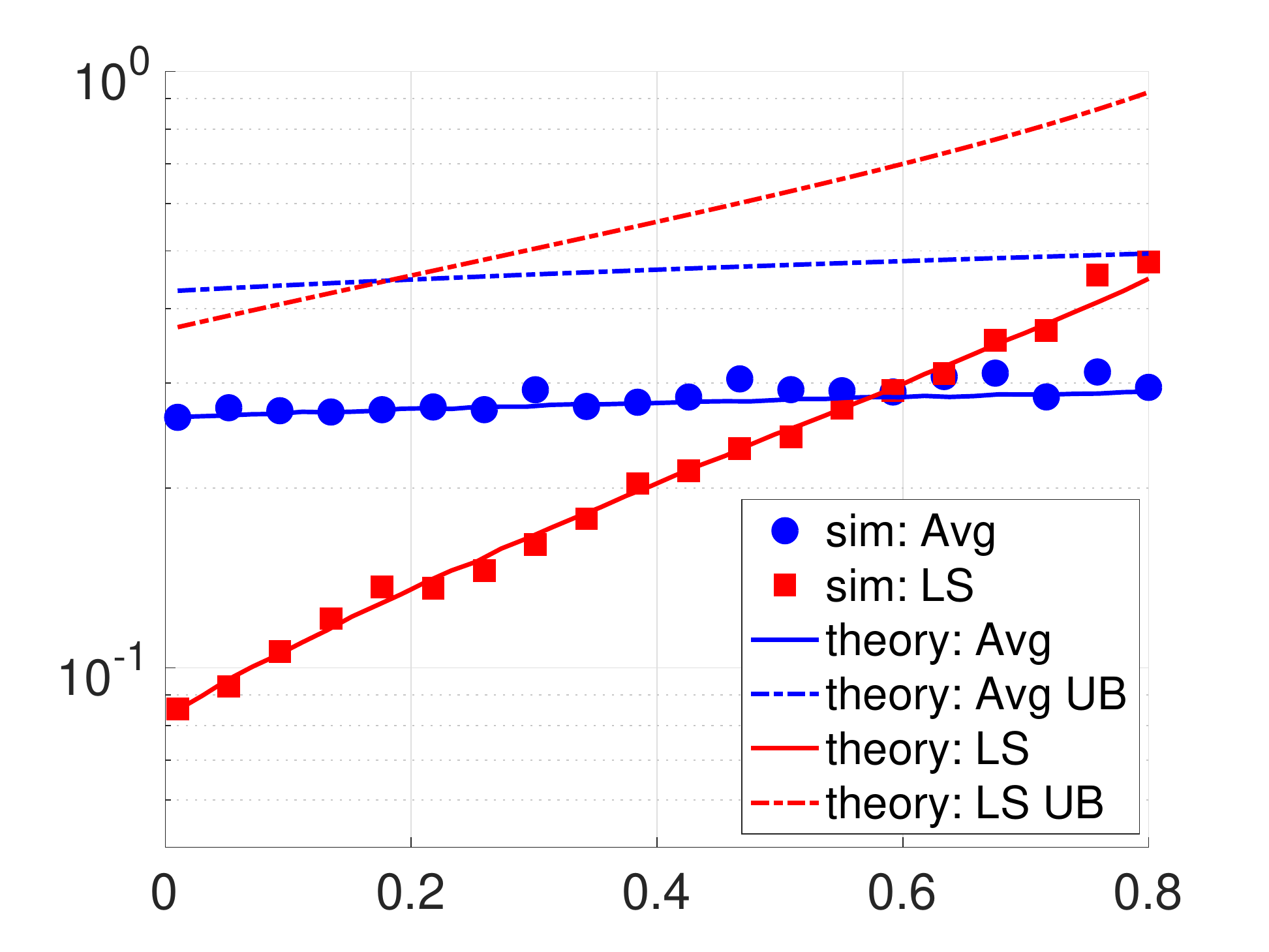}};
	\node at (-2.35,0) [rotate=90,scale=1]{Prob. of error};
	\node at (0,-1.72) [scale=1]{$\gamma$};
	\end{tikzpicture}\vspace{-5pt}
	\end{subfigure}
		\begin{subfigure}{1.8in}
	\begin{tikzpicture}
	\node at (0,0) {\includegraphics[scale=0.22]{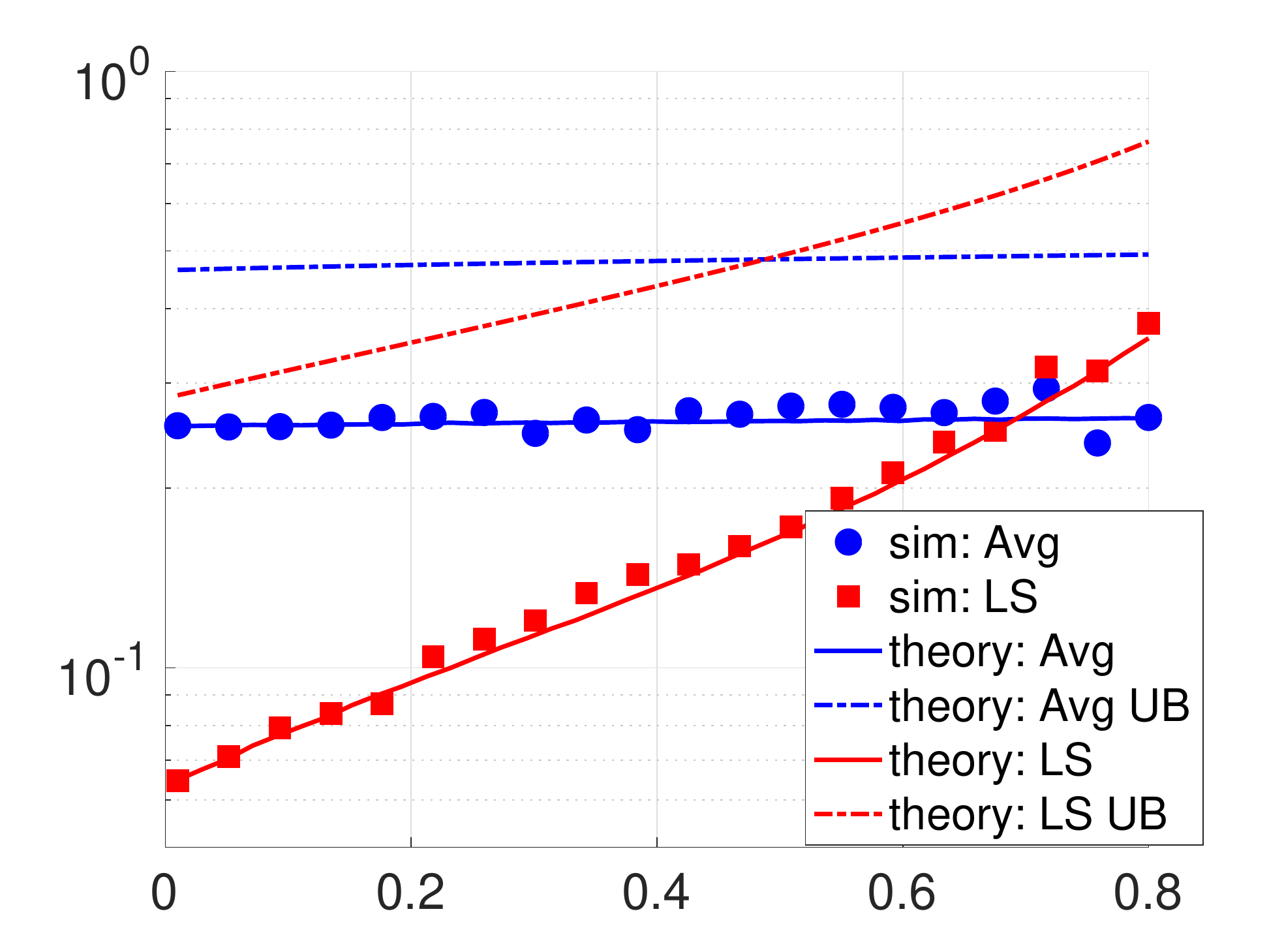}};
	\node at (-2.35,0) [rotate=90,scale=1.]{};
	\node at (0,-1.72) [scale=1.]{$\gamma$};
	\end{tikzpicture}\vspace{-5pt}
		\end{subfigure}\vspace{-0.16cm}
	\begin{subfigure}{1.8in}
	\begin{tikzpicture}
	\node at (0,0) {\includegraphics[scale=0.22]{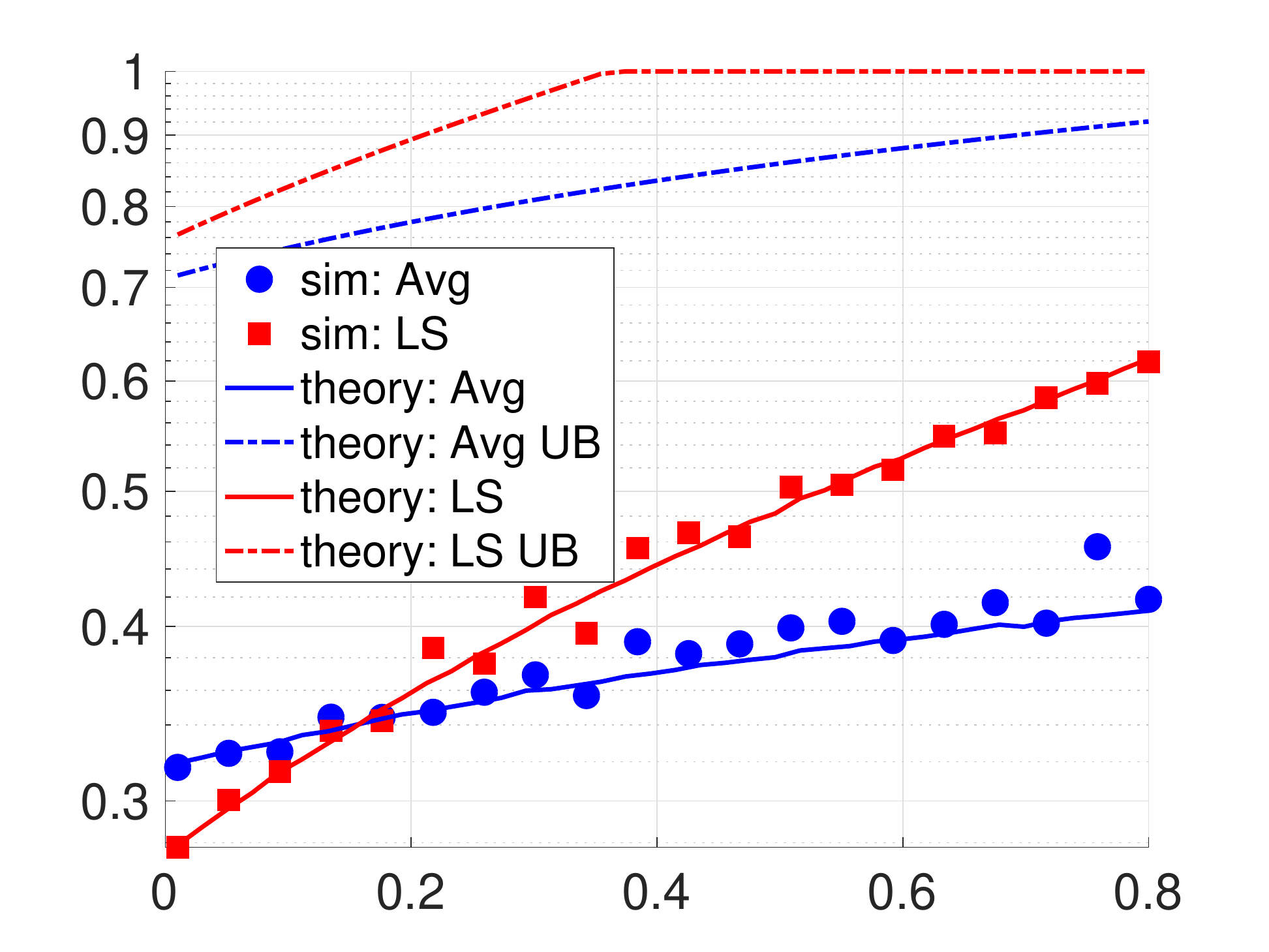}};
	\node at (-2.35,0) [rotate=90,scale=1.]{};
	\node at (0,-1.72) [scale=1.]{$\gamma$};
	\end{tikzpicture}\vspace{-5pt}
	\end{subfigure}
		\caption{Precise theoretical predictions compared to a theoretical upper bound (UB) obtained by Union bound that does \emph{not} require knowledge of the off-diagonal entries of $\Sigmab_{\w,\w}$. GMM with orthogonal means and (a) equal priors, different strengths; (b) different priors, different strengths; (c) different priors, equal strengths. See text for details.
		}\label{fig:UB}
\end{figure}

In this section, we provide further numerical experiments.

First, in Figure \ref{fig:gamma_star} we investigate the question: 
\text{\em{When does least-squares provably outperform averaging?}}
Our Proposition \ref{propo:gamma_star} provides a fundamental transition point in sample complexity above which least-squares is provably better than averaging under MLM. In Figure \ref{fig:gamma_star}, we visualize $\gamma_\star$ as a function of different number of classes as well as different levels of mean energy. Least-squares outperform averaging in the region below the lines displayed in Figure \ref{fig:gamma_star}. Our key message is that least-squares work better when the sample complexity is higher and the problem is less noisy. As the number of classes $k$ increase, the problem becomes more difficult/noisy and we require a larger sample complexity to ensure classifier achieves a similar amount of accuracy as small $k$. Following this intuition, as $k$ increases, $\gamma_\star$ shifts smaller due to larger sample requirement. Similarly energy $\mu$ directly controls the noise level of the problem, i.e.,~larger $\mu$ results in a larger signal-to-noise ratio. Thus, as we increase $\mu$, $\gamma_\star$ increases as well because same test accuracy can be achieved with smaller sample size.

Second, Figure \ref{fig:MLM_class} provides further experiments on the class-wise probabilities of the MLM model with $k=9$ classes for $\gamma=0.1$. Classes 1,2,3 have norms $\twonorm{\mub_i}=15$, while we have quadrupled the norms of classes 7,8,9 and doubled the norms of classes 4,5,6. In scenario (a) the means are orthogonal and in scenario (b) the means are highly correlated. The WLS shown corresponds to the following choice of weights: $\omega_i^2=1/\pib_\ell, \ell\in[k]$, where $\pib_\ell$ is the $\ell$th entry of the vector $\pib$ in \eqref{eq:alphas_gen}. The theoretical predictions for the class-wise error probabilities are computed using formula \eqref{eq:MLM_class}. As was the case for the GMM in Figure \ref{fig12}, we see that WLS creates the flattest class-wise errors. 

Finally, in Figure \ref{fig:UB} we investigate the following question: 
{\em{To what extent pairwise class correlations are necessary for performance prediction?}} Specifically, we consider a GMM setup with $k=5$ classes and orthogonal means under three scenarios:  (a) $\pi_1=\pi_2=\pi_3=\pi_4=\pi_5$, $4\twonorm{\mub_1}=4\twonorm{\mub_2}=2\twonorm{\mub_3}=2\twonorm{\mub_4}=\twonorm{\mub_5}=4\sqrt{3}$; (b)  $4\pi_1=4\pi_2=2\pi_3=2\pi_4=\pi_5$, $4\twonorm{\mub_1}=4\twonorm{\mub_2}=2\twonorm{\mub_3}=2\twonorm{\mub_4}=\twonorm{\mub_5}=4\sqrt{3}$; (c) $4\pi_1=4\pi_2=2\pi_3=2\pi_4=\pi_5$, $\twonorm{\mub_1}=\twonorm{\mub_2}=\twonorm{\mub_3}=\twonorm{\mub_4}=\twonorm{\mub_5}=\sqrt{3}$. The solid lines are exact performance predictions based on our theory for averaging and least-squares estimators. The dashed lines are the theoretical upper bounds, which do \emph{not} require the knowledge of cross-correlations between the classes (i.e.,~off-diagonal entries of $\bSi_{\w,\w}$ are unknown). These bounds are calculated by applying a union bound to  the class-wise probabilities $\P_{e|c}, ~c\in[k]$ in \eqref{eq:4UB} and further appropriately bounding the off-diagonal entries of $\bSi_{\w,\w}$ in terms of  the self-correlations of the classes, i.e.,~only the diagonal entries of $\bSi_{\w,\w}$. Please see Section \ref{sec:UB} for details. Overall, the bounds shown only depend on $\bh, \bSi_{\w,\mub}$ and $\diag{\bSi_{\w,\w}}$, which can all be obtained by studying the properties of isolated least-squares on individual classes without understanding their pairwise relations. 
While this suggests a simpler method to calculate theoretical bounds, there is a visible gap between such upper bounds and exact bounds and this gap is particularly more visible in the third scenario (c), where the bound becomes vacuous for LS. The gap remains visible in scenarios (a) and (b). At least, in these two cases comparing the bounds for averaging and LS to each other reveals the transition in performance gain between the two estimators. However, the cross-point of the curves does not coincide with the true one. This empirical study emphasizes the fact that pairwise correlations are indeed critical for exact asymptotic analysis and naive approaches cannot reproduce, in general, the results of our sharp analysis.

\section{Additional Results on Weighted Least-squares classifiers}
\subsection{WLS for GMM}
\label{GMWLS}
We now focus on characterizing the intercepts/correlation matrices for the WLS classifier.

\begin{theorem}\label{thm:WLS_GM}
Consider data generated according to GMM and $\gamma<1$. Consider a weighted LS classifier with weights $\mtx{D}=\diag{\omega_1,\ldots,\omega_k}$ and let $\eta$ be the unique solution to $\sum_{\ell=1}^k \frac{\pi_\ell \omega_\ell^2}{\omega_\ell^2+\eta}=\gamma$.
Also define
$
\Pb := \diag{\widetilde{\pib}}-\widetilde{\pib}\widetilde{\pib}^T\succeq \zero_{k\times k}$ and $\Deltab:=\sigma^2\Iden_r + \Sigmab\Vb^T\Pb\Vb\Sigmab \succ \zero_{r\times r}$ with the entries of $\widetilde{\pib}$ given by $\widetilde{\pi}_\ell=\frac{1}{\gamma}\frac{\pi_\ell \omega_\ell^2}{\omega_\ell^2+\eta}$. Then, for the WLS linear classifier $\left(\Wh,\bh\right)$ the following asymptotic limits hold
\begin{subequations}
\begin{align}\label{eq:WLS_GM}
\bh &\rP\widetilde{\pib}-\Pb\Vb\Sigmab\Deltab^{-1}\Sigmab\Vb^T\widetilde{\pib}\,,\quad
\Sigmab_{\w,\mub} \rP \Pb\Vb\Sigmab\Deltab^{-1}\Sigmab\Vb^T \,,\\
\Sigmab_{\w,\w} &\rP\frac{\zeta}{\sigma^2} \Pb+ \Pb\Vb\Sigmab \Deltab^{-1} \Big(\Deltab^{-1} - \frac{\zeta}{\sigma^2}\Iden_r\Big) \Sigmab\Vb^T \Pb+\frac{\eta\zeta}{\sigma^2}\mtx{Q}\,.
\end{align}
\end{subequations}
Here, $\zeta:={\gamma}\big/{\left(\eta \,\sum_{\ell=1}^k\frac{\pi_\ell\omega_\ell^2}{(\omega_\ell^2+\eta)^2}\right)}$ and $\mtx{Q}\in\R^{k\times k}$ is a known matrix depending on various problem parameters. Its precise value is given in \eqref{eq:Qmat_app}.

\end{theorem} 
Surprisingly, the effect of the weights is essentially equivalent to adjusting the class priors from $\vct{\pi}$ to $\widetilde{\vct{\pi}}$ defined in the theorem (modulo the extra additive term in the cross correlation matrix $\Sigmab_{\w,\w} $). This shows that weighted LS has similar performance to an un-weighted LS applied to a model with different class priors $\widetilde{\vct{\pi}}$. This characterization allows us to precisely understand how different weighting schemes can alter test accuracy for rare/minority classes.

\subsection{WLS for MLM}\label{sec:WLS_MLM_main}

Theorem \ref{thm:WLS_MLM} predicts the asymptotic performance of \emph{weighted} least-squares for data generated according to MLM.

\begin{theorem}\label{thm:WLS_MLM}
Consider data generated according to MLM and $\gamma<1$. Consider a weighted LS classifier with weights $\mtx{D}=\diag{\omega_1,\ldots,\omega_k}$ and let $\eta$ be the unique solution to $\sum_{\ell=1}^k \frac{\pi_\ell \omega_\ell^2}{\omega_\ell^2+\eta}=\gamma$. 
Also define vector $\nub\in\R^{k}$ with entries given by ${\nub}_\ell=\frac{1}{\gamma}\frac{\omega_\ell^2}{\omega_\ell^2+\eta}$ and matrix 
\bea\label{eq:Delta_WLS_thm}
\Deltab = \E\left[ \left(\nub^T\vb\right) \g\g^T \right] - \Sigmab\Vb^T\left(\diag{\pib}-\Pib\right)\nub\nub^T\left(\diag{\pib}-\Pib\right)\Vb\Sigmab \succ \zero_{r\times r},
\end{align}
where $\vb\in\R^k$ is a random vector with entries $V_\ell={e^{\eb_\ell^T\Vb\Sigmab\g}}\big/{\sum_{\ell^\prime\in[k]}e^{\eb_{\ell^\prime}\Vb\Sigmab\g}}$ for $\g\sim\Nn(\zero,\Iden_r)$. 
Then, for the WLS linear classifier $\left(\Wh,\bh\right)$ the following asymptotic limits hold
\begin{subequations}\label{eq:WLS_MLM_thm}
\begin{align}
\bh &\rP\diag{\nub}\pib-\diag{\nub}\left(\Iden_k-\pib\nub^T\right)\left(\diag{\pib}-\Pib\right)\V\Sigmab\Deltab^{-1}\Sigmab\Vb^T\left(\diag{\pib}-\Pib\right)\nub. \label{eq:b_WLS_thm}
\\
\label{eq:Sigmawmu_WLS_fin_thm}
\Sigmab_{\w,\mub} &\rP \diag{\nub}\left(\Iden_k-\pib\nub^T\right)\left(\diag{\pib}-\Pib\right)\V\Sigmab\Deltab^{-1}\Sigmab\Vb^T.
\end{align}
\end{subequations}
The corresponding formula for the asymptotic limit of the cross-correlation matrix $\Sigmab_{\w,\w}$ is given in \eqref{eq:ww_WLS} in Section \ref{sec:WLS_MLM}.
\end{theorem}

Of course, the theorem above includes Theorem \ref{propo:LS_log} as a special case. Indeed, we show how setting  $\omega_\ell=1,~\ell\in[k]$ recovers the solution for (un-weighted) LS. First, solving for $\eta$ simply gives 
$\eta = \frac{1}{\gamma} - 1.$
Thus, $\nub=\one_k$. Also, observe in \eqref{eq:alphas_gen} that $\left(\diag{\pib}-\Pib\right)\one_k = \zero$ and $\one^T\vb=1$. Thus, \eqref{eq:Delta_WLS_thm} reduces to
$\Deltab = \E[\g\g^T] = \Iden_r$. With these, it can be readily checked that \eqref{eq:b_WLS_thm} and \eqref{eq:Sigmawmu_WLS_fin_thm} simplify to the expressions in \eqref{eq:ls_soft_2}. 

The term $\E\left[ \left(\nub^T\vb\right) \g\g^T \right]$ in \eqref{eq:Delta_WLS_thm} can be computed using Monte Carlo sampling. It is also possible to slightly simplify the calcuations involved using Gaussian integration by parts as shown in Lemma \ref{lem:IBP}. As mentioned, the formula that predicts $\Sigma_{\w,\w}$ is given in  \eqref{eq:ww_WLS}. While somewhat more complicated than formulae \eqref{eq:WLS_MLM_thm}, the expression that we provide is also explicit. Numerical simulations shown in Figure \ref{fig:MLM_class} in Section \ref{sec:additionalnum} validate the accuracy of the theoretical predictions of the theorem.

\section{Preliminaries}
In this section we gather a few preliminary results that will be used later on in our proofs.
\subsection{Slepian's inequality}
\begin{lemma}[Slepian's inequality \cite{ledoux}]\label{lem:Slepian}
Let $\g\sim\Nn(\zero,\Sb)$ and $\gwt\sim\Nn(\zero,\Rb)$ such that for all $i,j\in[k]$:
$$\Sb_{ii}=\Rb_{ii},\qquad\text{and}\qquad\Sb_{ij}\geq \Rb_{ij}.$$
Then, for any $\tb\in\R^k$ it holds that
$$
\P\left\{ \bigcup_{j\in[k]} \left\{\g_j \geq \tb_j \right\} \right\} \leq \P\left\{ \bigcup_{j\in[k]} \left\{\gwt_j \geq \tb_j \right\} \right\}\,.
$$
Equivalently, letting $\z\sim\Nn(\zero,\Iden_k)$,
$$
1-\P\left\{ \Sb^{1/2}\z\leq \tb \right\} \leq 1-\P\left\{ \Rb^{1/2}\z\leq \tb \right\}.
$$
\end{lemma}

\subsection{Gaussian integration by parts}

We say that a function $F:\R^m\rightarrow\R$ is of moderate growth if for each $c>0$, $$\lim_{\twonorm{\x}\rightarrow\infty}~F(\x)\exp\left(-c\twonorm{\x}^2\right)=0.$$

The following result is a direct application of Gaussian integration by parts; for instance, see \cite[Prop.~8.29]{foucart2013invitation}.
\begin{lemma}[Gaussian integration by parts (GIP) ]\label{lem:IBP0}
Let $\g\sim\Nn(\zero,\Iden_r)$ and function $f:\R^{r}\rightarrow\R$ such that $f$ and all its first and second order partial derivatives are of moderate growth. Then, the following statements are true:

\noindent{(i)}~~$\E\left[f(\g) \g \right] = \E\left[\nabla f(\g)\right].$

\noindent{(ii)}~~$\E\left[f(\g) \g \g^T \right] = \E\left[f(\g)\right]\Iden_r + \E\left[\nabla^2 f(\g)\right].$

\end{lemma}

The following is a corollary of Lemma \ref{lem:IBP0} applied to the soft-max function. 

\begin{lemma}[GIP for the Softmax]\label{lem:IBP}
Let  $\g\sim\Nn(\zero_r,\Iden_r)$ and random vector
 $\vb=[V_1,V_2,\ldots,V_k]^T$ with entries:
\bea\label{eq:Vs}
\vb=\frac{e^{\Vb\Sigmab\g}}{\one_k^Te^{\Vb\Sigmab\g}},~~~~~~V_\ell = \frac{e^{\eb_\ell^T\Vb\Sigmab\g}}{\sum_{j\in[k]}e^{\eb_j^T\Vb\Sigmab\g}},~\ell\in[k].
\end{align}
Further recall the notation of  $\pib$ and $\Pib$ in \eqref{eq:alphas_gen}. The following statements are true:

\noindent{(i)}~~$\E[\vb] = \pib.$

\noindent{(ii)}~~For all $i\in[r],\ell\in[k]$,
$$\E\left[\g_i V_\ell\right] = ( \eb_i^T\Sigmab\Vb^T\eb_\ell ) \,\pib_\ell - \eb_i^T\Sigmab\Vb^T \,  \sum_{j\in[k]}{ \eb_j  \Pib_{ij}},$$ and in matrix form: $$\E\left[\vb\g^T\right] = \left(\diag{\pib}-\Pib\right)\V\Sigmab\,.$$

\noindent{(iii)}~~For all $\ell\in[k]$ let $s_\ell:\R^k\rightarrow\R$ denote the soft-max function:
$
s_\ell(\x) = \frac{e^{x_\ell}}{\sum_{i\in[k]}e^{x_i}}. 
$
Then, 
$$
\E\left[V_\ell \g\g^T\right] = \pi_\ell \Iden_r + \Sigmab^T\Vb^T \E\left[\nabla^2 s_\ell(\Vb\Sigmab\g) \right]\Vb\Sigmab.
$$


\end{lemma}

%
%

\subsection{Block matrix inversion}
\begin{lemma}[Block matrix inversion]
Let $\mtx{T} = \begin{bmatrix}
\A & \bb \\ \bb^T & \delta
\end{bmatrix}
$ be an invertible block matrix. Then
\bea
\mtx{T}^{-1}
\begin{bmatrix}
\fb \\ \epsilon
\end{bmatrix} =
\begin{bmatrix}
\Deltab^{-1}\left(\fb-\frac{\epsilon}{\delta}\bb\right)
\\
\frac{\epsilon}{\delta} - \frac{1}{\delta} \bb^T\Deltab^{-1}\left(\fb-\frac{\epsilon}{\delta}\bb\right)
\end{bmatrix}
\end{align}
where $\Deltab=\A-\frac{1}{\delta}\bb\bb^T\succ\zero$ is the Schur complement.
\end{lemma}

%
\section{Calculating and bounding the missclassification error}
\label{Testerrcal}

\subsection{Proof of \eqref{eq:Pe_GMM_gen} and \eqref{eq:Pe_log_gen}}\label{sec:proof_Pe_gen}

\textbf{GMM.}~~Starting from \eqref{eq:class_error} and using the fact that $\x_i=\mub_Y+\z = \M\eb_Y+\z,~\z\sim\Nn(\mathbf{0},\sigma^2\Iden_k)$, we have that 
\begin{align}
\P_e = \P\Big\{ \arg\max_{j\in[k]} \{ \inp{\wh_j}{\M\eb_Y} + \inp{\wh_j}{\z} + \bh_j \} \neq Y \} \Big\},\nn
\end{align}
or, in matrix-form:
\begin{align}
\P_e = \P\Big\{ \arg\max \{ \Wh\M\eb_Y + \Wh{\z} + \bh \} \neq Y \} \Big\}.\nn
\end{align}
Recall that $\mtx{\Sigma}_{\w,\mub} := \Wh\M$ and note that $\Wh{\z}$ is a zero-mean Gaussian vector with covariance matrix $\sigma^2\Wh\Wh^T = \sigma^2\Sigmab_{\w,\w}$ in order to conclude with the desired formula in \eqref{eq:Pe_GMM_gen}.

\vp
\noindent\textbf{MLM.}~~Recall from \eqref{eq:log_model} that $\x\sim\mathcal{N}(\zero,\Iden_d)$  and $Y$ is distributed such that $\P\{Y = \ell ~|~ \x \} = {e^{\inp{\mub_\ell}{\x}}}\big/{\sum_{j\in[k]}e^{\inp{\mub_j}{\x}}}.$ Let $\g=\Wh\x$ and $\h=\M^T\x$. In this notation, \eqref{eq:class_error} becomes 
$$
\P_e = \P\left\{ \arg\max\,( \g + \bh \,) \neq Y \right\},
$$
with $\P\{Y = \ell ~|~ \x \} = {e^{\h_\ell}}\big/{\sum_{j\in[k]}e^{\h_j}}$. To complete the proof of \eqref{eq:Pe_log_gen}, it is easy to check that $\begin{bmatrix}
\g \\ \h
\end{bmatrix}$ defined above is jointly Gaussian with zero-mean and covariance matrix  $\begin{bmatrix}
\Sigmab_{\w,\w} & \Sigmab_{\w,\mub} \\ \Sigmab_{\w,\mub}^T & \Sigmab_{\mub,\mub}
\end{bmatrix} \, .
$

\subsection{Class-wise and total miss-classification error for GMM}\label{sec:last2}
The class-wise miss-classification error for GMM is given by
\bea
\P_{e|c} 
&= \P\left( \exists j\neq c~:~ \inp{\wh_c-\wh_j}{\z} \leq  \inp{\wh_j-\wh_c}{\mub_c} +  (\bh_j-\bh_c)   \right) \label{eq:4UB}\\
&= 1 - \P\left( \forall j\neq c~:~ \inp{\wh_c-\wh_j}{\z} \geq  \inp{\wh_j-\wh_c}{\mub_c} +  (\bh_j-\bh_c)   \right),\label{eq:Pe1}
\end{align}
where we used that $\x = \mub_c+\z$.
Let $\Sb_c\in\R^{{(k-1)}\times{(k-1)}}$ be a symmetric matrix and  $\tb_c\in\R^{k-1}$ a vector with entries:
\begin{subequations}\label{eq:Sc_GMM}
\begin{align}
[\tb_c]_j&:=\inp{\wh_j-\wh_c}{\mub_c} +  (\bh_j-\bh_c)~~j\neq c \in[k]\\
[\Sb_c]_{ij}&:=\inp{\wh_c-\wh_j}{\wh_c-\wh_i}, ~~i,j\neq c \in[k].
\end{align}
\end{subequations}
Then, we can rewrite \eqref{eq:Pe1} as
\bea\label{eq:condP}
\P_{e|c}:= 1 - \P\Big\{ \Sb_c^{{1}/{2}}\, \z \geq \tb_c \Big\},
\end{align}
where the inequality in the rightmost expression applies entry-wise. 

Further, by using the law of total probability we have
\begin{align*}
\P_{e}=\sum_{c=1}^k \pi_c\P_{e|c}=\sum_{c=1}^k \pi_c\left(1-\P\Big\{ \Sb_c^{{1}/{2}}\, \z \geq \tb_c \Big\}\right).
\end{align*}

\subsection{Class-wise and total miss-classification error for MLM}\label{sec:last1}
In this section, we derive an explicit formula for the class-wise error for MLM. Recall \eqref{eq:Pe_log_gen}:
\begin{align}
\P_e = \P\big\{ \arg\max\text{ }(\,\g + \bh\,) \neq Y(\h)  \big\}\,,\quad\text{where}~~\begin{bmatrix}
\g \\ \h
\end{bmatrix} \sim \Nn\Big( \zero , \begin{bmatrix}
\Sigmab_{\w,\w} & \Sigmab_{\w,\mub} \\ \Sigmab_{\w,\mub}^T & \Sigmab_{\mub,\mub}
\end{bmatrix} \Big)\, , \nn
\end{align}
and $\P\{Y(\h)=\ell\} = {e^{\h_\ell}}/{\sum_{j\in[k]} e^{\h_j}},~~ \ell\in[k].$
 Using Gaussian decomposition we can write $\g=\gwt+ \Sigmab_{\w,\mub}\Sigmab_{\mub,\mub}^\dagger\h$ where $\gwt\sim\Nn(\zero_k,\Sigmab_{\w,\w} - \Sigmab_{\w,\mub}\Sigmab_{\mub,\mub}^\dagger\Sigmab_{\w,\mub})$. Using this, we have
\bea\nn
\P_e  &= \P\big\{ \arg\max\text{ }( \gwt + \Sigmab_{\w,\mub}\Sigmab_{\mub,\mub}^\dagger\h + \bh ) \neq Y(\h)  \big\}\nn
\\ &= \P\big\{ \arg\max\text{ }( \gwt + \Sigmab_{\w,\mub}\Sigmab_{\mub,\mub}^\dagger\h + \bh ) \neq Y(\h)  \big\} \nn
\\ &= \sum_{c\in[k]} \E_{\h,\gwt}\left[ \P\left\{ Y(\h) = c \right\} \cdot\ind\left\{\arg\max\text{ }(\gwt + \Sigmab_{\w,\mub}\Sigmab_{\mub,\mub}^\dagger\h + \bh ) \neq c \right\}\right]
 \nn
\\ &= \sum_{c\in[k]} \E_{\h,\gwt}\left[\frac{e^{\h_c}}{{\sum_{\ell\in[k]}e^{\h_{\ell}}}} \left( 1 - \prod_{j\neq c} \ind\left\{  \gwt_c + [\Sigmab_{\w,\mub}\Sigmab_{\mub,\mub}^\dagger\h]_c + \bh_c \geq  \gwt_j + [\Sigmab_{\w,\mub}\Sigmab_{\mub,\mub}^\dagger\h]_j + \bh_j \right\} \right) \right]
 \nn
\\ &= \sum_{c\in[k]}\E_{\h}\left[ \frac{e^{\h_c}}{{\sum_{\ell\in[k]}e^{\h_{\ell}}}}\left(1-\P_{\z\sim\Nn(\zero,\Iden_{k-1})}\big\{ \Sb_c^{1/2} \z \geq \tb_c(\h)  \big\} \right)\right], \label{eq:condF}
\end{align}
where in the last line 
\begin{subequations}\label{eq:Sc_MLM}
\bea
[\tb_c(\h)]_j&=\bh_{j} - \bh_{c} + [ \Sigmab_{\w,\mub}\Sigmab_{\mub,\mub}^\dagger \h ]_j -  [ \Sigmab_{\w,\mub}\Sigmab_{\mub,\mub}^\dagger \h ]_c ,~~j\neq c \in[k]\\
[\Sb_c]_{i,j}&= \left(\eb_c-\eb_j\right)^T\left(\Sigmab_{\w,\w} - \Sigmab_{\w,\mub} \Sigmab_{\mub,\mub}^\dagger\Sigmab_{\w,\mub}\right)\left(\eb_c-\eb_i\right),~~i,j\neq c\in[k].
\end{align}
\end{subequations}
Further recalling the decomposition
$\P_e = \sum_{c} \mathbb{P}\Big\{\widehat{Y}\neq Y | Y=c\Big\} \P\left\{Y=c\right\}$
and noting that 
\begin{align*}
\mathbb{P}\{Y=c\}=\E_{\vct{x}}\Big[\mathbb{P}\Big\{Y=c | \vct{x}\Big\}\Big]=\E_{\vct{x}}\Bigg[\frac{1}{\left(1+\sum_{j\neq c} e^{(\vct{\mu}_j-\vct{\mu}_c)^T\vct{x}}\right)}\Bigg] = \pib_c
\end{align*}
 we can see from \eqref{eq:condF} that the class-wise error probabilities can be calculated as follows:
\bea
\P_{e|c}  = \mathbb{P}\Big\{ \widehat{Y}\neq Y | Y=c\Big\} = \frac{1}{\pib_c}\,\E_{\h\sim\Nn(\zero_k,\Sigmab_{\mub,\mub})}\left[ \frac{e^{\h_c}}{{\sum_{\ell\in[k]}e^{\h_{\ell}}}}\left(1-\P_{\z\sim\Nn(\zero,\Iden_{k-1})}\left\{ \Sb_c^{1/2} \z \geq \tb_c(\h)  \right\} \right)\right],\label{eq:MLM_class}
\end{align}
where $\pib_c$ is the $c^{\text{th}}$ entry of the vector $\pib$ in \eqref{eq:alphas_gen} and $\Sb_c,\tb_c(\h)$ are defined in \eqref{eq:Sc_MLM}.

\subsection{Evaluating and bounding tail probabilities of multivariate Gaussians}
In Sections \ref{sec:last1} and \ref{sec:last2}, we expressed the class-wise probability of missclassification error for both GMM and MLM in the following convenient form for $\z\sim\Nn(\zero,\Iden_{k-1})$,
\bea
1 - \P\{ \A^{1/2}\z \leq \tb \} =  \P\left\{ \bigcup_{i\in [k-1]} \{\widetilde\ab_i^T\z \geq \tb_i \} \right\} 
\label{eq:2bound_B4}\,.
\end{align}
Here, $\A\succeq \zero\in\R^{(k-1)\times(k-1)}, \tb\in\R^{k-1}$ are appropriate coefficient matrices (see \eqref{eq:condP} and \eqref{eq:condF}) and $\widetilde\ab_i$ denotes the $i$th row of the matrix $\A^{1/2}$. For example, \eqref{eq:2bound_B4} maps to \eqref{eq:condP} for $\A\leftarrow\Sb_c$ and $\tb \leftarrow (-\tb_c)$.

The formulation above is convenient both in our theoretical analysis, as well as, in simulations. In the rest of this section, we briefly discuss some relevant tools that allow to further simplify or bound expressions in the form of \eqref{eq:2bound_B4}.

\subsubsection{A special case: Rank-one update of Identity}
First, we discuss the case where the coefficient matrix $\A$ and vector $\tb$ in \eqref{eq:2bound_B4} take the special form $\A\propto\Iden + \one\one^T$ and $\tb\propto\one$. This special case appears in some of the stylized symmetric problem settings studied in this paper, such as classification problems with orthogonal and equally-balanced means. 

\begin{lemma}\label{lem:rank1}
Let $\A=\Iden_k + \one_k\one_k^T$ and $\g\sim\Nn(\zero,\A)$. Then, for any $t\in\R$,
\begin{align}\label{eq:simple_rank1}
1 - \P\{ \g \leq t \one_k \} = \P\{ G_0 + \max_{i\in[k]} G_i \geq t \},~~G_0,G_1,\ldots,G_{k}\simiid\Nn(0,1)\,.
\end{align}
\end{lemma}
\begin{proof}
For each $i\in[n]$, we can decompose 
$
\g_i = G_0 + G_i,
$
 where $G_0, G_1,\ldots, G_k$ are iid standard normals. Indeed, it can be readily checked from this that $\E[\g_i^2] = 2$ and $\E[\g_i\g_j] = \E[G_0^2] = 1,~i\neq j$, which is consistent with $\g\sim\Nn(\zero,\A)$. 
Thus, we can write
$$
1 - \P\{ \g \leq t \one_k \} =  \P\{ \max_{i\in[k]} \g_i \geq t \} =  \P\{ G_0 + \max_{i\in[k]} G_i \geq t \},
$$
which completes the proof.
\end{proof}

\subsubsection{Slepian's bound}
When the matrix $\A$ does not have the special structure assumed by Lemma \ref{lem:rank1}, it is not possible in general to provide simple expressions as the one in \eqref{eq:simple_rank1}. Yet, it might be possible to obtain upper bounds of the same simple form. Such simple bounds can be useful for theoretical interpretations of otherwise complicated formulae, or can provide efficient means for quick (but, non-tight) implementations. 

In this section, we discuss Slepian's inequality (see \ref{lem:Slepian}) as a useful tool in this direction. 
Assume that
$
a = \min_{i, j\in[k]} \A_{ij} \geq 0.
$
To begin,  note that 
$
\A \geq \left({\rm diag}(\A) - a\Iden\right)\, + a\one\one^T,
$
where the inequality holds element-wise and equality is true for the diagonal elements. Then, one can apply Slepian's Lemma \ref{lem:Slepian} to upper bound the conditional probability of error in \eqref{eq:2bound_B4} with the following simple bound:
\bea
\nn 1 - \P\{ \A^{1/2}\g \leq \tb \} &\leq 1 - \P\left\{ \left(\left({\rm diag}(\A) - a\Iden\right)\, + a\one\one^T\right)^{1/2}\g \leq \tb \right\}\\
&\leq \P\Big\{ \bigcup_{j\in[k]} \big\{\,G_0 + G_{j} \sqrt{{[\A]_{jj}}/{a}-1}  \geq {[\tb]_j}/{a} \,\big\}  \Big\},\qquad G_0,G_1,\ldots,G_{k}\simiid\Nn(0,1)\,.\nn
\end{align}
In the second line above, we used the Gaussian decomposition of Lemma \ref{lem:rank1}.


\subsubsection{Simple bounds for GMM}\label{sec:UB}

\vp
\noindent\textbf{Union bound.}~Of course, it is also possible to apply (a simpler) union bound  to upper bound the tail probability in \eqref{eq:2bound_B4}. Here, we show explicitly the result of applying union bound to the class-wise error probabilities of the GMM. Specifically, consider \eqref{eq:4UB}. An application of the union bound leads to the following:
\bea
\P_{e|c} &= \P\left( \exists j\neq c~:~ \inp{\wh_c-\wh_j}{\z} \leq  \inp{\wh_j-\wh_c}{\mub_c} +  (\bh_j-\bh_c)   \right) \\
&\leq \sum_{j\neq c} \P_{G\sim\Nn(0,1)}\left\{ \twonorm{\wh_j - \wh_c} G \leq \inp{\wh_j-\wh_c}{\mub_c} + (\bh_j-\bh_c)   \right\}\nn\\
&=  \sum_{j\neq c} Q\left( \inp{\frac{\wh_c-\wh_j}{\twonorm{\wh_c - \wh_j}}}{\mub_c} + \frac{\bh_c-\bh_j}{\twonorm{\wh_c - \wh_j}} \right)  \nn\\
&= \sum_{j\neq c} Q\left( \frac{-[\tb_c]_j}{\sqrt{[\Sb_c]_{jj}}}\right)\label{eq:UB33}\\
& \leq (k-1)\cdot Q\left( d_{\min} \right),\nn
\end{align}
where in \eqref{eq:UB33} $\Sb_c,\tb_c$ are defined in \eqref{eq:Sc_GMM} and in the last line we denote $d_{\min}:=\min_{j\neq c} \big\{{-[\tb_c]_j}\Big/{\sqrt{[\Sb_c]_{j,j}}}\big\}.$ 

\vp
\noindent\textbf{Union bound without knowledge of cross-correlations $\inp{\w_i}{\w_j},~i\neq j$.}~It is worth noting that the upper bound in \eqref{eq:UB33}  requires knowledge of the cross-correlations $\inp{\wh_j}{\wh_c},~j\neq c$, i.e., of the off-diagonal entries of $\Sigmab_{\w,\w}$. Thankfully, our analysis allows predicting these values. For comparison, we ask wether it is possible to further upper bound the class-wise error probability if only the diagonal entries of $\Sigmab_{\w,\w}$ (i.e., the norms $\twonorm{\wh_j}, j\in[k]$) were known. A simple answer to this questions is as follows. Observe that $\sqrt{[\Sb_c]_{jj}} = \twonorm{\wh_c-\wh_j}\leq \twonorm{\wh_c}+ \twonorm{\wh_j}.$ Thus, if $[\tb_c]_j<0$ then the $j$th term in \eqref{eq:UB33} is further upper bounded by $Q\left( \frac{-[\tb_c]_j}{\twonorm{\wh_c}+ \twonorm{\wh_j}}\right)$:
$$
\P_{e|c} \leq \begin{cases} 
\sum_{j\neq c} Q\left( \frac{-[\tb_c]_j}{\twonorm{\wh_c}+ \twonorm{\wh_j}}\right)& \text{if } \tb_c\leq \zero,\\
1 &\text{otherwise.}
\end{cases}
$$
 Unfortunately, this bound becomes non-trivial for the class-wise probability of error only if $$[\tb_c]_j<0 \Leftrightarrow \inp{\wh_c}{\mub_c} + \bh_c \geq \inp{\wh_j}{\mub_c} + \bh_j \text{ for all } j\neq c.$$
 Intuitively, this assumes a regime wherethe weight vector $\wh_c$ corresponding to class $c$ aligns better with the corresponding mean vector $\mub_c$ than the rest of the weight vectors $\wh_j,~j\neq c.$
  This emphasizes the important role of the cross-correlation matrix $\Sigmab_{\w,\w}$ (including the off-diagonals) for accurate performance prediction. For an illustration, we have implemented this bound and have compared it to our sharp predictions in Figure \ref{fig:UB}.

\vp
\noindent\textbf{Oracle lower bound.}~ For completeness, we briefly discuss an oracle lower bound for the class-wise probability of error in GMM. Specifically,  assume that the means $\mub_i,~i\in[n]$ are known. Then the problem of classifying a new sample $\x$ is a $k$-ary hypothesis testing problem with Gaussian conditionals. Denote $\P_{\rm genie , Bayes}$ the Bayes error of this hypothesis testing problem. Clearly $\P_{\rm genie , Bayes}$ is a lower bound on the error of any classifier that is trained on data. The paper \cite{wisler2016empirically} further lower bounds $\P_{\rm genie , Bayes}$ in terms of the Bayesian probability of errors between every two classes as follows:
\bea
\P_e &\geq \P_{\rm genie , Bayes} \geq \frac{2}{k}\sum_{i\neq j}\pi_i  \P_{\rm genie , ij} \\
&= \frac{2}{k}\sum_{i\neq j}\frac{\pi_i}{\pi_i+\pi_j} \Big\{ \pi_i \cdot Q\left( \frac{\twonorm{\mub_i-\mub_j}}{2} + \frac{\log(\pi_i/\pi_j)}{\twonorm{\mub_i-\mub_j}} \right) + \pi_j \cdot Q\left( \frac{\twonorm{\mub_i-\mub_j}}{2} - \frac{\log(\pi_i/\pi_j)}{\twonorm{\mub_i-\mub_j}} \right)  \Big\}.\nn
\end{align}
where $\P_{\rm genie , ij}$ is the Bayesian error between classes $i$ and $j$ with priors $\frac{\pi_i}{\pi_i+\pi_j}$ and $\frac{\pi_j}{\pi_i+\pi_j}$. For the last equality we have used the well-known formula for the Bayesian probability of binary Gaussian hypothesis testing. In the case of equal-priors the genie lower bound above simplifies to
\bea\label{eq:genie_LB}
\P_e \geq \frac{2}{k}\sum_{i\neq j}{\pi_i}\cdot Q\left( \frac{\twonorm{\mub_i-\mub_j}}{2}  \right) .
\end{align}
Note that, in contrast to \eqref{eq:genie_LB}, our analysis allows for precise evaluations of the missclassification error $\P_e$.

\section{The Class-averaging estimator}

\subsection{Proofs for GMM}
\subsubsection{GMM: Proof of Proposition \ref{thm:ave_GM}}
The first statement \eqref{eq:ave01} follows directly from the fact that $\frac{1}{n}\one^T\Y_i = \frac{n_i}{n} \rP \pi_i.$ For the next two statements note that
\begin{align}\label{eq:ave_pf_kfixed}
\wh_i = \frac{1}{n}\M\Y \Y_i + \frac{1}{n} \Zb \Y_i =\frac{1}{n}\sum_{j=1}^k \mub_j(\Y_j^T\Y_i) + \frac{1}{n} \Zb \Y_i = \frac{\twonorm{\Y_i}^2}{n}\mub_i + \frac{1}{n} \Zb \Y_i,
\end{align}
where in the last line we used orthogonality of the rows $\Y_j$ of the matrix $\Y$:
\bea\label{eq:ortho}
\inp{\Y_i}{\Y_j} = 0, \forall i\neq j\in[k].
\end{align}
 To conclude simply use the facts that for all $i\in[k]$: 
 \begin{enumerate}
\item[(i)] $\frac{\twonorm{\Y_i}^2}{n}=\frac{n_i}{n}\rP\pi_i$.
 \item[(ii)] $\Zb \Y_i \sim \sigma\twonorm{\Y_i}\g_i$ with $\g_i\simiid\Nn(\zero,\Iden_d)$ because of \eqref{eq:ortho}.
 \item[(iii)] $\frac{\twonorm{\g_i}}{\sqrt{n}}\rP \sqrt{\gamma}$ and $\frac{1}{\sqrt{n}}\inp{\g_i}{\mub_j}\rP 0$.
 \end{enumerate}

\subsection{Proofs for MLM}

\subsubsection{Proof of Proposition \ref{propo:ave_log}}
Let us define $\g\sim\Nn(\zero,\Iden_r)$ and random vector
 $\vb=[V_1,V_2,\ldots,V_k]^T$ with entries:
\bea\label{eq:Vs}
\vb=\frac{e^{\Vb\Sigmab\g}}{\one_k^Te^{\Vb\Sigmab\g}},~~~~~~V_i = \frac{e^{\eb_i^T\Vb\Sigmab\g}}{\sum_{j\in[k]}e^{\eb_j^T\Vb\Sigmab\g}},~i\in[k].
\end{align}
We will prove the following three statements:
\begin{subequations}\label{eq:ave_soft_app}
\begin{align}
\bh &\rP \E\left[\vb\right]
\label{eq:ave_soft_app_1}  \\
\Sigmab_{\w,\mub} &\rP \E\left[\vb\g^T\right]\Sigmab\Vb^T
\label{eq:ave_soft_app_2}
\\
\Sigmab_{\w,\w} &\rP \gamma\cdot\diag{\E\left[\vb\right]} + \E\left[\vb\g^T\right]\cdot \E\left[\g\vb^T\right]
\label{eq:ave_soft_app_3}.
\end{align}
\end{subequations}
These lead to \eqref{eq:ave_soft} using Lemma \ref{lem:IBP}. Therefore, in what follows, we prove \eqref{eq:ave_soft_app}

For the intercepts $\bh_\ell,~\ell\in[k]$ it holds that
$$
\bh_\ell = \frac{1}{n}\one^T\Y_\ell \rP \P\{Y=\ell\} = \E\left[ \frac{e^{\h_\ell}}{\sum_{j\in[k]}e^{\h_j}} \right],
$$
where $\h\sim\Nn\left(\zero_k,\Sigmab_{\mub\mub}\right)$. To deduce the first statement in \eqref{eq:ave_soft_app_1}, note that $\h\eqD\Vb\Sigmab\g$.

Continuing with the vectors $\wh_\ell,~\ell\in[k]$, recall that $\w_\ell=\frac{1}{n}\X\Y_\ell = \frac{1}{n}\sum_{i_\ell\in[n]}\x_{i_\ell}[\Y_\ell]_{i_\ell}$. Consider the singular decomposition  
$$
\mtx{M}=\mtx{U}\mtx{\Sigma}\mtx{V}^T = 
\begin{bmatrix}\ub_1 & \ub_2 & \ldots & \ub_r \end{bmatrix} \diag{\sigma_1,\sigma_2,\ldots,\sigma_r} \begin{bmatrix}\vb^T_1 \\ \vb_2^T \\ \ldots \\ \vb_r^T \end{bmatrix} ,
$$
with $\mtx{U}\in\R^{d\times r}$, $\mtx{\Sigma}\in\R^{r\times r}$, and $\mtx{V}\in\R^{k\times r}$ where $r=$rank$(\mtx{M})\le k$. Decompose $\X\in\R^{d\times n}$ as $\X=\Ub\Ub^T\X + \Pb^\perp\X$ with $\Pb^\perp = \Iden_d-\Ub\Ub^T.$ With this notation we compute
\begin{align}
\inp{\w_\ell}{\mub_c}=\frac{1}{n}\sum_{i\in[n]}\x^T_{i}\mub_c[\Y_\ell]_{i}  &= \frac{1}{n}\sum_{j=1}^r\sum_{i\in[n]}(\x^T_{i}\ub_j)\cdot(\mub_c^T\ub_j)[\Y_\ell]_{i} + \frac{1}{n}\sum_{i\in[n]}\cdot(\mub_c^T\Pb^\perp\x_i)[\Y_\ell]_{i}  \nn\\
&\rP \sum_{j=1}^{r}\left(\eb_c^T\Vb\Sigmab\eb_j\right)\E\left[ \g_j  \frac{\left(e^{\eb_\ell^T\Vb\Sigmab\g}\right)}{\sum_{\ell^\prime\in[k]}e^{\eb_{\ell^\prime}^T\Vb\Sigmab\g}}\right] \nn\\
 &= \sum_{j=1}^{r}\E[V_\ell\g_j] \, \left(\eb_j^T\Sigmab\Vb^T\eb_c\right).
\end{align}
Here, we have recognized that for every $i\in[n]: \Ub^T\x_i\sim\g$, and also, conditioned on $\x_i$: $[\Y_\ell]_i\sim{\rm Bern}\left(e^{\mub_\ell^T\x_i}\big/\sum_{\ell^\prime}e^{\mub_{\ell^\prime}^T\x_i}\right)$ and $\mub_{\ell}^T\x_i=\eb_{\ell}^T\Vb\Sigmab\Ub^T\x_i\sim\eb_{\ell}^T\Vb\Sigmab\g, \ell\in[k]$.
This shows the second statement in \eqref{eq:ave_soft_app_2} when expressed in matrix form. 

We proceed similarly with the proof of the last statement in \eqref{eq:ave_soft_app_3} as follows:
\begin{align}
&\inp{\w_\ell}{\w_c}=\frac{1}{n^2}\sum_{i_\ell\in[n],i_c\in[n]}\x^T_{i_\ell}\x_{i_c}[\Y_\ell]_{i_\ell}[\Y_c]_{i_c}  \nn
\\
&=\frac{1}{n^2}\sum_{j=1}^r\sum_{i_\ell\in[n],i_c\in[n]}(\x^T_{i_\ell}\ub_{j})(\x^T_{i_c}\ub_{j})[\Y_\ell]_{i_\ell}[\Y_c]_{i_c} + \frac{1}{n^2}\sum_{i_\ell\in[n],i_c\in[n]}\left(\Pb^\perp\x_{i_\ell}\right)^T\left(\Pb^\perp\x_{i_c}\right)[\Y_\ell]_{i_\ell}[\Y_c]_{i_c} \nn
\end{align}
For $i_\ell=i_c=i\in[n]$ note that
\begin{align}
\frac{1}{n^2}\sum_{i\in[n]}\sum_{j=1}^r(\x^T_{i}\ub_{j})(\x^T_{i}\ub_{j})[\Y_\ell]_{i}[\Y_c]_{i} &\rP 0, \nn
\end{align}
while, for $i_\ell\neq i_c$,
\begin{align}
\frac{1}{n^2}\sum_{i_\ell\neq i_c\in[n]}\sum_{j=1}^r(\x^T_{i_\ell}\ub_{j})[\Y_\ell]_{i_\ell} (\x^T_{i_c}\ub_{j})[\Y_c]_{i_c} &\rP \sum_{j=1}^r\E\left[\g_j\cdot\frac{\left(e^{\eb_\ell^T\Vb\Sigmab\g}\right)}{\sum_{\ell^\prime\in[k]}e^{\eb_{\ell^\prime}^T\Vb\Sigmab\g}}\right] \E\left[\g_j\cdot\frac{\left(e^{\eb_c^T\Vb\Sigmab\g}\right)}{\sum_{\ell^\prime\in[k]}e^{\eb_{\ell^\prime}^T\Vb\Sigmab\g}}\right] \nn\\
& = \sum_{j=1}^{r} \E[V_\ell \g_j] \E[\g_j V_c] = \eb_\ell^T \Eb[\vb\g^T] \cdot \Eb[\g\vb^T] \eb_c^T \nn
\end{align}
Furthermore,
\begin{align}
\frac{1}{n}\sum_{i\in[n]}\twonorm{\Pb^\perp\x_{i}}^2[\Y_\ell]_{i}[\Y_c]_{i}^2\rP \gamma\cdot \ind_{\ell,c} \cdot\E\left[\left(\frac{e^{\eb_\ell^T\Vb\Sigmab\g}}{\sum_{\ell^\prime\in[k]}e^{\eb_{\ell^\prime}^T\Vb\Sigmab\g}}\right)\right]  = \gamma\cdot\eb_\ell^T \diag{\E[\vb]} \eb_c.\nn
\end{align}
Combining the last two displays results in \eqref{eq:ave_soft_app_2}, as desired.

\subsubsection{Orthogonal means}
Here, we specialize the general result of Proposition \ref{propo:ave_log} to the special case of orthogonal   means: $\inp{\mub_i}{\mub_j}=0,\forall i\neq j$. Recall the notation $\mu_i=\twonorm{\mub_i}, i \in[k]$. Then, in this case the parameters in \eqref{eq:alphas_gen} are simply given by the following 
\bea\label{eq:alphas}
\pib_i := \E\big[\frac{e^{\mu_i G_i}}{\sum_{\ell\in[k]}e^{\mu_\ell G_\ell}} \big], i\in[k]\quad\text{and}\quad \Pib_{ij} := \E\big[\frac{e^{\mu_i G_i}e^{\mu_j G_j}}{\left(\sum_{\ell\in[k]}e^{\mu_\ell G_\ell}\right)^2} \big], ~i,j\in[k].
\end{align}

%
%
Specifically, \eqref{eq:simple_log_Pe_ave} can be equivalently expressed as
\begin{align}
\P_{e,\rm Avg} &\rP \P\big( \arg\max_{\ell\in [k]} \left\{ \gamma\cdot\diag{\pib}\cdot\gwt + \left(\diag{\pib}-\Pib\right)\cdot \Sigmab \g \right\} \neq Y(\g) \big)\label{eq:ave_ortho_equal_Pe}\\
&= \P\big( \bigcup_{j\neq Y} \left\{ \gamma \cdot \pib_{\ell} \cdot \gwt_\ell \geq \gamma \cdot \pib_{Y} \cdot \gwt_Y + (\eb_Y-\eb_\ell)^T\left(\diag{\pib}-\Pib\right)\Sigmab \g + (\pib_Y-\pib_\ell)\right\} \big), \label{eq:union}
\end{align}
where $\g,\gwt\simiid\Nn(\zero,\Iden_k)$, $\P\left(Y(\g)=c\right) = \frac{e^{\mu_c \g_c}}{\sum_{\ell\in[k]}e^{\mu_\ell \g_\ell}}$ and $\Sigmab=\diag{\mu_1,\ldots,\mu_k}$.

\section{On the Bayes risk of GMM: Proof of Proposition \ref{propo:Bayes}}
Without loss of generality in this proof we assume $\sigma=1$. The general result follows by simply replacing $(\mu,\sigma)$ with $(\frac{\mu}{\sigma},1)$ and using the proof for $\sigma=1$. Recall that the feature vectors $\x_1,\ldots,\x_n$ of the training data set are given by:
$$
\x_i = \M\y_i + \z_i,\quad i\in[n],
$$
where the matrix of means $\M\in\R^{d\times k}$ has iid Gaussian entries with variance $\mu^2/d$, $\z_i\simiid\Nn(\zero,\Iden_d)$ and $\y_i\simiid\rm{Unif}\left(\e_1,\ldots,\e_k\right)$ with $\e_j$ denoting the $j^{\text th}$ canonical vector in $\R^k$. By definition here the Bayes estimator is the maximum-likelihood (ML) estimator. By applying the law of total probability and by successive application of the Bayes rule we have the following chain of reformulations of the ML:
\begin{align}
\yh_{n+1} &= \arg\max_{\eb_j,~j\in[k]} P\left\{ \y=\eb_j~|~\X,\Y,\x_{n+1}\right\} \nn\\
&= \arg\max_{\eb_j,~j\in[k]} \int P\left\{ \y=\eb_j~|~\M,\X,\Y,\x_{n+1}\right\} P\left\{\M~|~\X,\Y,\x_{n+1}\right\} \,\mathrm{d}\M\nn\\
&= \arg\max_{\eb_j,~j\in[k]} \int \frac{P\left\{ \x_{n+1}~|~\y=\eb_j,\M,\X,\Y\right\}\cdot P\left\{ \y=\eb_j~|~\M,\X,\Y\right\}}{P\left\{\x_{n+1}|\M,\X,\Y\right\}} P\left\{\M~|~\X,\Y,\x_{n+1}\right\} \,\mathrm{d}\M\nn\\
&= \arg\max_{\eb_j,~j\in[k]} \int {P\left\{ \x_{n+1}~|~\y=\eb_j,\M\right\}} \,\frac{P\left\{\M~|~\X,\Y,\x_{n+1}\right\}}{P\left\{\x_{n+1}|\M,\X,\Y\right\}} \,\mathrm{d}\M\label{eq:noM1}\\
&= \arg\max_{\eb_j,~j\in[k]} \int {P\left\{ \x_{n+1}~|~\y=\eb_j,\M\right\}} \,\frac{P\left\{\M~|~\X,\Y\right\}}{P\left\{\x_{n+1}|\X,\Y\right\}} \,\mathrm{d}\M\nn\\
&= \arg\max_{\eb_j,~j\in[k]} \int {P\left\{ \x_{n+1}~|~\y=\eb_j,\M\right\}} \,{P\left\{\M~|~\X,\Y\right\}} \,\mathrm{d}\M\label{eq:noM2}\\
&= \arg\max_{\eb_j,~j\in[k]} \int {P\left\{ \x_{n+1}~|~\y=\eb_j,\M\right\}} \,{P\left\{\X~|~\M,\Y\right\}}\,P\{\M\} \,\mathrm{d}\M\label{eq:ML}
\end{align}
To arrive in \eqref{eq:noM1} we used that $P\left( \y=\eb_j~|~\M,\X,\Y\right) =\pi,~\forall j\in[k]$ and $P\left( \x_{n+1}~|~\y=\eb_j,\M,\X,\Y\right)=P\left( \x_{n+1}~|~\y=\eb_j,\M\right)$. Also, \eqref{eq:noM2} follows by recognizing that $P(\x_{n+1}~|~\X,\Y)>0$ is independent of the variable of integration $\M$ and of the optimization variable $j$. For the same reasons, in \eqref{eq:ML} we have ignored the normalizing term $P(\X|\Y)$.


Recalling that $\z_{n+1}\sim\Nn(\zero,\Iden_d)$, we have that $P\left( \x_{n+1}~|~\y=\eb_j,\M\right)\propto\exp\left(-\twonorm{\x_{n+1} - \mub_j}^2\big/2\right)$ where $\propto$ hides constant positive terms. 
Moreover, the posterior probability of the mean matrix given the training data is given by
\bea
P\left(\X~|~\M,\Y\right)\cdot P\left(\M\right) &\propto \exp\left(-\frac{\twonorm{\X-\M\Y}^2}{2}\right)\cdot \exp\left(-\frac{\twonorm{\M}^2}{2(\mu^2/d)}\right) \nn\\
&\propto\prod_{c=1}^k  \left\{\exp\left(-\frac{\twonorm{\mub_c}^2}{2(\mu^2/d)}\right) \cdot\prod_{i\in\Cc_c} \exp\left(-\frac{\twonorm{\x_i-\mub_c}^2}{2}\right)\right\},
\end{align}
where we denote by $\Cc_c$ the collection of training samples that belong to class $c\in[k]$, i.e. $\Cc_c=\{i\in[n]~|~\y_i=\eb_c\}.$ 

With these the objective function of the ML rule in \eqref{eq:ML} becomes:
\begin{align}
\yh_{n+1} = \arg \max_{j\in[k]}~\Ic(j,\Cc_{j}\cap\{n+1\})\cdot\prod_{\substack{c=1 \\ c\neq j}}^k\Ic(c,\Cc_c)\label{eq:Bayes1}\,,
\end{align}
where for $\ell\in[k]$ and a subset $\Ac\subset[n+1]$ we denote
$$
\Ic(\ell,\Ac) := \int \mathrm{d}{\mub_\ell}\,\exp\left(-\frac{\twonorm{\mub_\ell}^2}{2(\mu^2/d)}\right) \cdot \exp\left(-\sum_{i\in\Ac}\frac{\twonorm{\x_i-\mub_\ell}^2}{2}\right)\,.
$$
By completing the squares and invoking a gaussian integral it can be shown that
\bea
\Ic(\ell,\Ac) &:= \sqrt{\frac{(d/\mu^2+|\Ac|)}{(2{\varpi})^d}}\exp\left(-\frac{\left(1-\frac{1}{d/\mu^2+|\Ac|}\right)}{2}\sum_{i\in\Ac}\twonorm{\x_i}^2 + \frac{1}{2\left(d/\mu^2+|\Ac|\right)}\sum_{i\in\Ac}\inp{\x_i}{\sum_{\substack{j\in\Ac\\j\neq i}}\x_j} \right) \nn \\
&:= \sqrt{\frac{(d/\mu^2+|\Ac|)}{(2{\varpi})^d}}\exp\left(-\frac{1}{2\left(\frac{d}{\mu^2}+{|\Ac|}\right)}\left( \left(\frac{d}{\mu^2}+|\Ac| - 1\right) \sum_{i\in\Ac}\twonorm{\x_i}^2 - \sum_{i\in\Ac}\inp{\x_i}{\sum_{\substack{j\in\Ac\\j\neq i}}\x_j} \right) \right) \nn\\
&:= \sqrt{\frac{(d/\mu^2+|\Ac|)}{(2{\varpi})^d}}\exp\left(-\frac{1}{2\left(\frac{d/n}{\mu^2}+{|\Ac|/n}\right)}\left( \left(\frac{d/n}{\mu^2}+\frac{|\Ac|}{n} - \frac{1}{n}\right) \sum_{i\in\Ac}\twonorm{\x_i}^2 - \frac{1}{n}\sum_{i\in\Ac}\inp{\x_i}{\sum_{\substack{j\in\Ac\\j\neq i}}\x_j} \right) \right) \nn.
\end{align}
Using this in \eqref{eq:Bayes1} we have that
\bea
\yh_{n+1} = \arg\max_{j\in[k]} ~\Ic(j) \cdot \exp\left(-\frac{1}{2\left(\frac{d/n}{\mu^2}+\frac{n_j+1}{n}\right)}\left( \left(\frac{d/n}{\mu^2}+\frac{n_j}{n}\right) \twonorm{\x_{n+1}}^2 - \frac{2}{n}\inp{\x_{n+1}}{\sum_{\substack{\ell\in\Cc_j}}\x_\ell} \right) \right), \label{eq:Bayes2}
\end{align}
where $\xi(n_c) := \frac{d/n}{\mu^2}+\frac{n_c}{n},~c\in[k]$ and 
\bea
\Ic(j):= \left\{\prod_{\substack{c=1\\ c\neq j}}^{k}e^{-\frac{1}{2\xi(n_c)}\left( \left(\xi(n_c) - \frac{1}{n}\right) \sum_{i\in\Cc_c}\twonorm{\x_i}^2 - \frac{1}{n}\sum_{i\in\Cc_c}\inp{\x_i}{\sum_{\substack{\ell\in\Cc_c\\\ell\neq i}}\x_\ell} \right) }\right\}\cdot e^{-\frac{1}{2\left(\xi(n_{j})+\frac{1}{n}\right)}\left( \xi(n_{j}) \sum_{i\in\Cc_j}\twonorm{\x_i}^2 - \frac{1}{n}\sum_{i\in\Cc_j}\inp{\x_i}{\sum_{\substack{\ell\in\Cc_j\\\ell\neq i}}\x_\ell} \right)} \,. \nn
\end{align}
We conclude that
\bea
\yh_{n+1} &= \arg\max_{j\in[k]} ~\log\left(\Ic(j)\right)-\frac{1}{2\left(\frac{d/n}{\mu^2}+\frac{n_j+1}{n}\right)}\left\{\left(\frac{d/n}{\mu^2}+\frac{n_j}{n}\right) \twonorm{\x_{n+1}}^2 - \frac{2}{n}\inp{\x_{n+1}}{\sum_{\substack{\ell\in\Cc_j}}\x_\ell} \right\}\nn \\
&= \arg\max_{j\in[k]} ~\log\left(\Ic(j)\right) +  \frac{1}{2\left(\frac{d/n}{\mu^2}+\frac{n_j+1}{n}\right)}\left\{  \frac{2}{n}\inp{\x_{n+1}}{\sum_{\substack{\ell\in\Cc_j}}\x_\ell} \right\}.
 \label{eq:Bayes3}
\end{align}
Next, we evaluate the objective in \eqref{eq:Bayes3} in the asymptotic limit $n,d\rightarrow\infty, n/d=\gamma$. First, since $n_c/n\rP \pi$, note that $\Ic(j)-\Ic(\ell)\rP 0$ for all $\ell,j\in[k]$. Moreover, note that
\bea
\frac{1}{n}\inp{\x_{n+1}}{\sum_{\ell\in \Cc_j}\x_\ell} &= \frac{1}{n}\inp{\M\y_{n+1} + \z_{n+1}}{ n_j\,\mub_j + \sum_{\ell\in \Cc_j}\z_\ell}\nn\\
&= \frac{n_j}{n}\inp{\M\y_{n+1}}{\mub_j}  + \frac{n_j}{n}\inp{\z_{n+1}}{\mub_j}  + \frac{1}{n}\sum_{\ell\in\Cc_j}\inp{\M\y_{n+1}}{\z_\ell} +  \frac{1}{n}\sum_{\ell\in\Cc_j}\inp{\z_{n+1}}{\z_\ell} \label{eq:fourTerms}
\end{align}
For each one of the four terms in \eqref{eq:fourTerms}, we have the following by the CLT:
\bea
&\frac{n_j}{n}\inp{\M\y_{n+1}}{\mub_\ell} \rP \pi \mu^2 \inp{\y_{n+1}}{\eb_j} \nn\\
& \frac{n_j}{n}\inp{\z_{n+1}}{\mub_j} \rD \Nn(0,\pi^2 r^2)\nn\\
&\frac{1}{n}\sum_{\ell\in\Cc_j}\inp{\M\y_{n+1}}{\z_\ell} \rP 0 \nn\\
&\frac{1}{n}\sum_{\ell\in\Cc_j}\inp{\z_{n+1}}{\z_\ell} \rD \Nn(0,\pi \gamma), \nn
\end{align}
where in the last line we used the fact that $\frac{1}{\sqrt{n_j}}\sum_{\ell\in\Cc_j}\frac{\inp{\z_{n+1}}{\z_\ell}}{\sqrt{n}}\rD \Nn(0,\gamma)$. 

Therefore, in the asymptotic limit, the Bayes estimator is the solution to:
\bea
\yh_{n+1} &= \arg\max_{\eb_j, j\in[k]} \pi \mu^2 \inp{\y_{n+1}}{\eb_j} + \sqrt{\pi\left(\pi \mu^2 + \gamma\right)}\, G_j,\quad G_1,\ldots,G_k\simiid\Nn(0,1).
\end{align}
As such, the probability of error is
\bea
\P_e &= \P\left\{\yh_{n+1}\neq\y_{n+1}\right\} = \P\left\{  \pi \mu^2 + \sqrt{\pi\left(\pi \mu^2 + \gamma\right)}\, G_0 \leq \max_{\ell\in[k-1]} \sqrt{\pi\left(\pi \mu^2 + \gamma\right)}\, G_\ell  \right\}\nn\\
&= \P\left\{ G_0 + \max_{\ell\in[k-1]}  G_\ell \geq \mu^2\sqrt{\frac{\pi}{\pi \mu^2 + \gamma}}  \right\}.
\end{align}


\section{Least-squares for GMM}\label{sec:proof_LS_GM}

\subsection{Proof of Theorem \ref{thm:LS_GM}}\label{sec:proof_LS_GM_main}

\subsubsection{Computing $\Sigma_{w,\mu}$}\label{sec:MLM_proof_1}
The LS classifier solves:
\begin{align*}
\underset{\mtx{W}\in\R^{k\times d},\text{ }\vct{b}\in\R^k}{\min}\quad \frac{1}{2n}\fronorm{\mtx{W}\mtx{X}+\vct{b}\vct{1}_n^T-\mtx{Y}}^2=&\sum_{\ell=1}^k \min_{\vct{w}_\ell, b_\ell}\text{ }\frac{1}{2n}\twonorm{\mtx{X}^T\vct{w}_\ell+b_\ell\vct{1}_n-\mtx{Y}_\ell}^2\\
=&\sum_{\ell=1}^k \min_{\vct{w}_\ell, b_\ell}\text{ }\frac{1}{2n}\twonorm{\mtx{Y}^T\mtx{M}^T\vct{w}_\ell+ \mtx{Z}^T\vct{w}_\ell+b_\ell\vct{1}_n-\mtx{Y}_\ell}^2\,.
\end{align*}

Define 
\begin{align}
\mathcal{L}_{PO}\left(\vct{w}_\ell,b_\ell\right):=\frac{1}{2n}\twonorm{\mtx{Y}^T\mtx{M}^T\vct{w}_\ell+ \mtx{Z}^T\vct{w}_\ell+b_\ell\vct{1}_n-\mtx{Y}_\ell}^2\label{eq:LS_PO}\,.
\end{align}

\noindent\textbf{Identifying the AO.}~~To continue further note that by duality we have
\begin{align*}
\min_{\vct{w}_\ell, b_\ell}\text{ }\mathcal{L}_{PO}\left(\vct{w}_\ell,b_\ell\right) =\min_{\vct{w}_\ell, b_\ell}\text{ }\max_{\vct{s}}\text{ }\frac{1}{n}\left(\vct{s}^T\mtx{Y}^T\mtx{M}^T\vct{w}_\ell+ \vct{s}^T\mtx{Z}^T\vct{w}_\ell+b_\ell\vct{s}^T\vct{1}_n-\vct{s}^T\mtx{Y}_\ell-\frac{\twonorm{\vct{s}}^2}{2}\right)\,.
\end{align*}
Note that the above is jointly convex in $(\vct{w}_\ell, b_\ell)$ and concave in $\vct{s}$ and the Gaussian matrix $\Zb$ is independent of everything else. Thus, the objective is in the form of \eqref{eq:PO_obj} and so we consider the corresponding Auxiliary Optimization (AO) problem:
\begin{align*}
\min_{\vct{w}_\ell, b_\ell}\text{ }\max_{\vct{s}}\text{ }\frac{1}{n}\left(\vct{s}^T\mtx{Y}^T\mtx{M}^T\vct{w}_\ell+ \sigma\twonorm{\vct{w}_\ell}\vct{g}^T\vct{s}+ \sigma\twonorm{\vct{s}}\vct{h}^T\vct{w}_\ell+b_\ell\vct{s}^T\vct{1}_n-\vct{s}^T\mtx{Y}_\ell-\frac{\twonorm{\vct{s}}^2}{2}\right),
\end{align*}
where $\vct{g}\in\R^n$ and $\vct{h}\in\R^d$ are independent Gaussian random vectors with i.i.d.~$\mathcal{N}(0,1)$ entries. Maximizing over the direction of $\vct{s}$ and setting its norm $\beta=\twonorm{\vct{s}}$ we arrive at
\begin{align*}
\min_{\vct{w}_\ell, b_\ell}\text{ }\max_{\beta\ge 0}\text{ } &\frac{1}{n}\left(\beta\twonorm{ \sigma\twonorm{\vct{w}_\ell}\vct{g}+\mtx{Y}^T\mtx{M}^T\vct{w}_\ell+b_\ell\vct{1}_n-\mtx{Y}_\ell}+ \beta\sigma\vct{h}^T\vct{w}_\ell-\frac{\beta^2}{2}\right)\\
&\quad\quad\quad\quad\quad\quad=\min_{\vct{w}_\ell, b_\ell}\text{ }\frac{1}{2n}\left(\twonorm{ \sigma\twonorm{\vct{w}_\ell}\vct{g}+\mtx{Y}^T\mtx{M}^T\vct{w}_\ell+b_\ell\vct{1}_n-\mtx{Y}_\ell}+ \sigma\vct{h}^T\vct{w}_\ell\right)_{+}^2\\
&\quad\quad\quad\quad\quad\quad=\frac{1}{2}\left(\min_{\vct{w}_\ell, b_\ell}\text{ } \frac{1}{\sqrt{n}}\twonorm{ \sigma\twonorm{\vct{w}_\ell}\vct{g}+\mtx{Y}^T\mtx{M}^T\vct{w}_\ell+b_\ell\vct{1}_n-\mtx{Y}_\ell}+\frac{1}{\sqrt{n}}\sigma \vct{h}^T\vct{w}_\ell\right)_{+}^2
\end{align*}

\noindent\textbf{Scalarization of the AO.}~~
For convenience, define
\begin{align}\label{eq:AO_main_GM}
\bar\phi_{AO,\ell}:= 
\min_{\vct{w}_\ell, b_\ell}\text{ } \frac{1}{\sqrt{n}}\twonorm{ \sigma\twonorm{\vct{w}_\ell}\vct{g}+\mtx{Y}^T\mtx{M}^T\vct{w}_\ell+b_\ell\vct{1}_n-\mtx{Y}_\ell}+\frac{\sigma}{\sqrt{n}} \vct{h}^T\vct{w}_\ell.
\end{align}

To continue, consider the singular value decomposition  
\begin{align}\label{eq:M_SVD}
\mtx{M}=\mtx{U}\mtx{\Sigma}\mtx{V}^T = 
\begin{bmatrix}\ub_1 & \ub_2 & \ldots & \ub_r \end{bmatrix} \diag{\sigma_1,\sigma_2,\ldots,\sigma_r} \begin{bmatrix}\vb^T_1 \\ \vb_2^T \\ \ldots \\ \vb_r^T \end{bmatrix} ,
\end{align}
with $\mtx{U}\in\R^{d\times r}$, $\mtx{\Sigma}\in\R^{r\times r}$, and $\mtx{V}\in\R^{k\times r}$ where $r=$rank$(\mtx{M})\le k$. We further decompose $\w_\ell$ in its projections on the orthogonal columns $\ub_1,\ldots,\ub_r$ of $\Ub$:
$$
\w_\ell = \sum_{i=1}^{r}\alpha_i \ub_i + \alpha_0\w_\ell^\perp, 
$$
where $\twonorm{\w_\ell^\perp}=1$ and $\Ub^T\w_\ell^\perp = \mathbf{0}$, $\alpha_0\geq 0$ and  we denote
\bea\label{eq:alphab_def}
\alpha_i := \ub_i^T\w_\ell, i\in[r].
\end{align}
We also define $\vct{\alpha}=\begin{bmatrix} \alpha_1 & \alpha_2 & \ldots & \alpha_k\end{bmatrix}^T$.
In this notation, we have
\begin{align}
\bar\phi_{AO,\ell}(\g,\h)&:= \min_{\alpha_0\geq 0,\text{ }\vct{\alpha}\in\R^r,\text{ }b_\ell}\text{ } \frac{1}{\sqrt{n}}\twonorm{ \sigma\sqrt{\alpha_0^2 + \twonorm{\vct{\alpha}}^2}\,\vct{g}+\mtx{Y}^T\Vb\mtx{\Sigma}\vct{\alpha}  +b_\ell\vct{1}_n-\mtx{Y}_\ell}\nonumber\\
&\quad\quad\quad\quad\quad\quad\quad\quad+ \sum_{i=1}^{r}{\alpha_i\sigma\frac{\vct{h}^T\ub_i}{\sqrt{n}}} + \frac{\alpha_0\sigma}{\sqrt{n}}\min_{\w_\ell^\perp}\left(\h^T\w_\ell^\perp\right) \nn\\
&
= \min_{\alpha_0\geq 0,\text{ }\vct{\alpha}\in\R^r,\text{ }b_\ell}\text{ } \frac{1}{\sqrt{n}}\twonorm{ \sigma\sqrt{\alpha_0^2 + \twonorm{\vct{\alpha}}^2}\,\vct{g}+\mtx{Y}^T\Vb\mtx{\Sigma}\vct{\alpha}   +b_\ell\vct{1}_n-\mtx{Y}_\ell}\nonumber\\
&\quad\quad\quad\quad\quad\quad\quad\quad+ \sum_{i=1}^{r}{\alpha_i\sigma\frac{\vct{h}^T\ub_i}{\sqrt{n}}} -\alpha_0\sigma\frac{\twonorm{\h^\perp}}{\sqrt{n}},\label{eq:scal_AO}
\end{align}
where in the second line we denote $\h^\perp$ the projection of $\h$ onto the complement subspace of the span of $\ub_1,\ldots,\ub_r$ and we recalled that $\twonorm{\w^\perp_\ell} = 1$ and $\inp{\w^\perp_\ell}{\ub_i}=0, i\in[r]$.

\vp
\noindent\textbf{Convergence of the AO.}~~First, note that 
\begin{align*}
\frac{1}{n}\twonorm{\mtx{Y}^T\Vb\mtx{\Sigma}\vct{\alpha}   +b_\ell\vct{1}_n-\mtx{Y}_\ell}^2=&\frac{1}{n}\twonorm{\mtx{Y}^T\left(\Vb\mtx{\Sigma}\vct{\alpha}-\vct{e}_\ell\right)   +b_\ell\vct{1}_n}^2\\
=&\frac{1}{n}\twonorm{\mtx{Y}^T\left(\Vb\mtx{\Sigma}\vct{\alpha}-\vct{e}_\ell\right)}^2+b_\ell^2+\frac{2}{n}b_\ell\vct{1}_n^T\mtx{Y}^T\left(\Vb\mtx{\Sigma}\vct{\alpha}-\vct{e}_\ell\right)\\
=&\text{trace}\left(\left(\Vb\mtx{\Sigma}\vct{\alpha}-\vct{e}_\ell\right)^T\text{diag}\left(\frac{n_1}{n} , \frac{n_2}{n} , \ldots , \frac{n_k}{n}\right)\left(\Vb\mtx{\Sigma}\vct{\alpha}-\vct{e}_\ell\right)\right)\\
&+b_\ell^2+2b_\ell\begin{bmatrix}\frac{n_1}{n} & \frac{n_2}{n}& \ldots & \frac{n_k}{n}\end{bmatrix}\left(\Vb\mtx{\Sigma}\vct{\alpha}-\vct{e}_\ell\right)
\end{align*}
Thus
\begin{align*}
\frac{1}{n}\twonorm{\mtx{Y}^T\Vb\mtx{\Sigma}\vct{\alpha}   +b_\ell\vct{1}_n-\mtx{Y}_\ell}^2\text{ }\rP\text{ }&\text{trace}\left(\left(\Vb\mtx{\Sigma}\vct{\alpha}-\vct{e}_\ell\right)^T\text{diag}\left(\vct{\pi}\right)\left(\Vb\mtx{\Sigma}\vct{\alpha}-\vct{e}_\ell\right)\right)\\
&+b_\ell^2+2b_\ell\vct{\pi}^T\left(\Vb\mtx{\Sigma}\vct{\alpha}-\vct{e}_\ell\right)\\
=&\alphab^T\left( \Sigmab\Vb^T\diag{\pib}\Vb\Sigmab\right)\alphab-2\pi_\ell\alphab^T\mtx{\Sigma}\Vb^T\vct{e}_\ell+2b_\ell\alphab^T\mtx{\Sigma}\Vb^T\vct{\pi}\\
&+b_\ell^2-2b_\ell\pi_\ell+\pi_\ell\,.
\end{align*}
At this point, observe that we have reduced the AO to an optimization problem over only $r+2$ scalar variables. Using the law of large numbers, the fact that $\twonorm{\h^\perp}$ concentrates around $\sqrt{d-r}$ and $(d-r)/n\rP\gamma$, as well as the limit calculation above, it is not hard to see that for fixed $\alpha_0,b_\ell$ and $\alphab=[\alpha_1,\ldots,\alpha_r]^T\in\R^r$, the objective function in \eqref{eq:scal_AO} converges to the following:
\begin{align}
&\Dc_\ell(\alpha_0,\alphab,b_\ell)\nonumber\\
&:=
\sqrt{\alpha_0^2 \sigma^2+ \alphab^T\left(\sigma^2\Iden_r + \Sigmab\Vb^T\diag{\pib}\Vb\Sigmab\right)\alphab - 2\alphab^T\left(\pi_\ell\mtx{\Sigma}\Vb^T\vct{e}_\ell-b_\ell\mtx{\Sigma}\Vb^T\vct{\pi}\right)+b_\ell^2-2 b_\ell\pi_\ell + \pi_\ell}\nonumber\\
&\quad\quad - \alpha_0\sigma\sqrt{\gamma}
\label{eq:det_1},
\end{align}

We will show in the next paragraph that the argument inside the square-root in \eqref{eq:det_1} is a convex quadratic over $(\alpha_0,\alphab,b_\ell)$ (see \eqref{eq:det_2}). Thus, the function $\Dc_\ell(\alpha_0,\alphab,b_\ell)$ is jointly convex. 
Using uniform convergence of convex functions over compact sets \cite[Cor..~II.1]{AG1982}, we arrive at
\begin{align}
\bar\phi_{AO,\ell}(\g,\h) \rP \min_{\alpha_0\geq 0,\alphab,b_\ell} \Dc_\ell(\alpha_0,\alphab,b_\ell).\label{eq:uniAO}
\end{align}

\noindent\textbf{Deterministic Analysis.}~~Here, we analyze the deterministic scalar minimization on the RHS of \eqref{eq:uniAO}. Define
\begin{align}\label{eq:Amat}
\A:=\begin{bmatrix} \sigma^2\Iden_r+\Sigmab\Vb^T\diag{\pib}\Vb\Sigmab & \Sigmab\Vb^T\pib \\ \pib^T\Vb\Sigmab & 1 \end{bmatrix}\quad\text{and}\quad \cb_\ell = \begin{bmatrix} \Sigmab\Vb^T\eb_\ell \\ 1 \end{bmatrix}\,,
\end{align}
and observe that we can write
\begin{align}
\Dc_\ell(\alpha_0,\alphab,b_\ell) = \sqrt{\alpha_0^2\sigma^2 + \pi_\ell +  \begin{bmatrix} \alphab^T & b_\ell \end{bmatrix} \A  \begin{bmatrix} \alphab \\ b_\ell  \end{bmatrix} - 2 \pi_\ell \cb_\ell^T \begin{bmatrix} \alphab \\ b_\ell  \end{bmatrix}} - \alpha_0\sigma\sqrt{\gamma}\label{eq:det_21}\,.
\end{align}
First, note that the matrix $\A$ is positive definite. This can be checked by computing the Schur complement of $\A$:
\begin{align}\label{eq:Deltab_def}
\Deltab:=\sigma^2\Iden_r+\Sigmab\Vb^T\Pb\Vb\Sigmab:=\sigma^2\Iden_r+\Sigmab\Vb^T\left(\diag{\pib}-\pib\pib^T\right)\Vb\Sigmab \succ \zero_{r\times r}.
\end{align}
Positive definiteness above holds because $\Pb:=\left(\diag{\pib}-\pib\pib^T\right)\succeq \zero_{k\times k}$.
 Thus the term under the square-root in \eqref{eq:det_21} is a strictly convex quadratic. Thus, $\Dc_\ell$ is jointly convex in its arguments.

To simplify the RHS of \eqref{eq:uniAO} we proceed by minimizing $\Dc_\ell(\alpha_0,\alphab,b_\ell)$ over $(\alphab,b_\ell)$ which from Lemma \eqref{lem:IBP} is equal to
\bea
\begin{bmatrix} \widehat\alphab \\ \widehat b_\ell \end{bmatrix} = \pi_\ell \A^{-1} \cb_\ell &= \pi_\ell \begin{bmatrix} \Iden & \zero \\ -\pib^T\Vb\Sigmab & 1 \end{bmatrix}\begin{bmatrix} \Deltab^{-1} & \zero \\ \zero^T & 1 \end{bmatrix} \begin{bmatrix} \Iden & -\Sigmab\Vb^T\pib \\ \zero^T & 1 \end{bmatrix}\begin{bmatrix} \Sigmab\Vb^T\eb_\ell \\ 1\end{bmatrix}\nn\\
&= \pi_\ell \begin{bmatrix} \Iden & \zero \\ -\pib^T\Vb\Sigmab & 1 \end{bmatrix}\begin{bmatrix} \Deltab^{-1} & \zero \\ \zero^T & 1 \end{bmatrix} \begin{bmatrix} -\Sigmab\Vb^T \left(\pib - \eb_\ell\right) \\ 1 \end{bmatrix} \nn\\
&= \pi_\ell \begin{bmatrix} -\Deltab^{-1}\Sigmab\Vb^T(\pib-\eb_\ell) \\ 1+\pib^T\Vb\Sigmab\Deltab^{-1}\Sigmab\Vb^T(\pib-\eb_\ell) \end{bmatrix}\label{eq:alphab_1}\,.
\end{align}
Thus, the minimum value attained is
\begin{align}
-\pi_\ell^2  \begin{bmatrix} -\left(\pib - \eb_\ell\right)^T\Vb\Sigmab && 1 \end{bmatrix} \begin{bmatrix} \Deltab^{-1} & \zero \\ \zero^T & 1 \end{bmatrix}\begin{bmatrix} -\Sigmab\Vb^T \left(\pib - \eb_\ell\right) \\ 1 \end{bmatrix} &= - \pi_\ell^2\left(1+\left(\pib - \eb_\ell\right)^T\Vb\Sigmab \Deltab^{-1} \Sigmab\Vb^T \left(\pib - \eb_\ell\right)\right).  \nn
\end{align}
Using the above, \eqref{eq:uniAO} reduces to 
\begin{align}
\bar\phi_{AO,\ell}(\g,\h) \rP \min_{\alpha_0\geq 0} \sqrt{\alpha_0^2\sigma^2 + \pi_\ell - \pi_\ell^2\left(1+\left(\pib - \eb_\ell\right)^T\Vb\Sigmab \Deltab^{-1} \Sigmab\Vb^T \left(\pib - \eb_\ell\right)\right)} - \alpha_0\sigma\sqrt{\gamma}.
\label{eq:uniAO2}
\end{align}
Setting the derivative with respect to $\alpha_0$ to zero we arrive at
\begin{align*}
\frac{\alpha_0\sigma^2}{\sqrt{\alpha_0^2\sigma^2 + \pi_\ell - \pi_\ell^2\left(1+\left(\pib - \eb_\ell\right)^T\Vb\Sigmab \Deltab^{-1} \Sigmab\Vb^T \left(\pib - \eb_\ell\right)\right)}}=\sigma\sqrt{\gamma}.
\end{align*}
Thus,
\begin{align}
\widehat\alpha_0 =\frac{1}{\sigma} \sqrt\frac{\gamma}{1-\gamma}\sqrt{\pi_\ell\left(1-\pi_\ell\right)- \pi_\ell^2\left(\pib - \eb_\ell\right)^T\Vb\Sigmab \Deltab^{-1} \Sigmab\Vb^T \left(\pib - \eb_\ell\right)}.
\end{align}
Plugging the latter into \eqref{eq:uniAO2} we arrive at
\begin{align*}
\bar\phi_{AO,\ell}(\g,\h) \rP \sqrt{1-\gamma }\sqrt{\pi_\ell\left(1-\pi_\ell\right)- \pi_\ell^2\left(\pib - \eb_\ell\right)^T\Vb\Sigmab \Deltab^{-1} \Sigmab\Vb^T \left(\pib - \eb_\ell\right)}\,.
\end{align*}

\vp
\noindent\textbf{Asymptotic predictions.}~~First, from \eqref{eq:alphab_1} the bias term converges as follows:
$$
\widehat{b}_\ell \rP \pi_\ell\left(1+\pib^T\Vb\Sigmab\Deltab^{-1}\Sigmab\Vb^T(\pib-\eb_\ell)\right).
$$
Thus,
\begin{align*}
\widehat{\vct{b}}\rP\diag{\pib}\left(\one_k+\left(\pib\one_k^T-\Iden_k\right)\Vb\Sigmab\Deltab^{-1}\Sigmab\Vb^T\pib \right)\,.
\end{align*}
 Recall from \eqref{eq:alphab_def} that $\alphab = \Ub^T\w_\ell$. Thus, the correlations $\inp{\mub_i}{\w_\ell},~i\in[k]$ converge as follows:
\begin{align}
\M^T\w_\ell = \Vb\Sigmab\Ub^T\w_\ell \rP \Vb\Sigmab\widehat\alphab = -{\pi_\ell}\Vb\Sigmab\Deltab^{-1}\Sigmab\Vb^T\left(\pib-\eb_\ell\right).\quad
\end{align}
Here, convergence applies element-wise to the entries of the involved random vectors. Moreover, from the analysis above we can predict the limit of the norm $\twonorm{\w_\ell}$. For this, note that $\twonorm{\w_\ell}^2 = \widehat{\alpha}_0^2 + \widehat{\alphab}^T\widehat{\alphab}$. Thus,
\begin{align}\label{eq:norm_lim}
\twonorm{\w_\ell}^2 \rP \frac{\gamma}{(1-\gamma)\sigma^2}\pi_\ell(1-\pi_\ell)  + \pi_\ell^2 \left(\pib - \eb_\ell\right)^T\Vb\Sigmab \Deltab^{-1} \left(\Deltab^{-1} - \frac{\gamma}{(1-\gamma)\sigma^2}\Iden_{r}\right) \Sigmab\Vb^T \left(\pib - \eb_\ell\right) \,.
\end{align}

\subsubsection{Computing $\Sigma_{w,w}$}\label{sec:cross_LS_GM}
In the previous section we used the CGMT to predict the bias $\widehat{b}_\ell$, the correlations $\inp{\mub_i}{\wh_\ell},~i[k]$ and the norm $\twonorm{\wh_\ell}$ for all $\ell\in[k]$ members of the multi-output classifier. Here, we show how to compute the limits of the cross-correlations $\inp{\wh_\ell}{\wh_j}, \ell\neq j\in [k]$.

\begin{lemma}\label{lem:joint_LS}
For $\ell\neq j\in[k]$, let $\wh_\ell$ $\wh_j$ be solutions to the least-squares minimization
$$
(\wh_\ell,\wh_j,\widehat{b}_\ell,\widehat{b}_j) =  \arg\min_{\w_\ell, \w_j, b_\ell, b_j} \left\{ \frac{1}{2n}\twonorm{\Y_{\ell} - \X^T \w_\ell - b_\ell\one_n }^2 + \frac{1}{2n}\twonorm{\Y_{j} - \X^T \w_j - b_j\one_n }^2\right\}.
$$
 Denote $\wh_{\ell,j}:=\wh_\ell+\wh_j$ and $\widehat{b}_{\ell,j}:=\widehat{b}_\ell+\widehat{b}_j$. Then, $(\wh_{\ell,j},\widehat{b}_{\ell,j})$ is a minimizer in the following least-squares problem:
\begin{align}\label{eq:joint_LS}
(\wh_{\ell,j},\widehat{b}_{\ell,j}) = \arg\min_{\w, b} \frac{1}{2n}\twonorm{\Y_{\ell} + \Y_{j} - \X^T \w - b\one_n }^2\,.
\end{align}
\end{lemma}
\begin{proof}
Clearly the minimization in \eqref{eq:joint_LS} is convex. Thus, it suffices to prove that $\wh_\ell+\wh_j$ satisfies the KKT conditions. First, by optimality of $\wh_\ell$, we have that 
$$
\X\left(\Y_\ell-\X^T\wh_\ell-\widehat{b}_\ell\one_n\right) = 0
$$
Similarly, for $\wh_j$:
$$
\X\left(\Y_j-\X^T\wh_j-\widehat{b}_j\one_n\right) = 0.
$$
Adding the equations on the above displays we find that
$$
\X\left(\Y_\ell+\Y_j-\X^T(\wh_j+\wh_\ell)-(\widehat{b}_j+\widehat{b}_\ell)\one_n\right) = 0.
$$
Recognize that this coincides with the optimality condition for \eqref{eq:joint_LS}. Thus, the proof is complete.
\end{proof}

Thanks to Lemma \ref{lem:joint_LS}, we can use the CGMT to characterize the limiting behavior of $\twonorm{\wh_\ell+\wh_j}$. Observe that this immediately gives the limit of $\inp{\wh_\ell}{\wh_j}$ since 
\begin{align}\label{eq:inp_from_norms}
\inp{\wh_\ell}{\wh_j} = \frac{\twonorm{\wh_\ell+\wh_j}^2-\twonorm{\wh_\ell}^2-\twonorm{\wh_j}^2}{2}.
\end{align} 

The analysis of \eqref{eq:joint_LS} is very similar to that of \eqref{eq:LS_PO}; thus, most details are omitted. Similar to \eqref{eq:scal_AO} we can relate \eqref{eq:joint_LS} with the following AO problem:
\begin{align}
\bar\phi_{AO,\ell,j}(\g,\h)&:=  \min_{\beta_0\geq 0, \vct{\beta}\in\R^r, b_{\ell,j}}\text{ } \frac{1}{\sqrt{n}}\twonorm{ \sigma\sqrt{\beta_0^2 + \twonorm{\vct{\beta}}^2}\,\vct{g}+\mtx{Y}^T\mtx{V}\mtx{\Sigma}\vct{\beta}  +b_{\ell,j}\vct{1}_n-\mtx{Y}_\ell - \Y_j}\nonumber\\
&\quad\quad\quad\quad\quad\quad\quad+ \sigma\sum_{i=1}^{r}{\beta_i\frac{\vct{h}^T\ub_i}{\sqrt{n}}} -\sigma\beta_0\frac{\twonorm{\h^\perp}}{\sqrt{n}},\label{eq:scal_AO_ij}
\end{align}
where we have decomposed 
$$
\w_{\ell,j} = \sum_{i=1}^{r}\beta_i \ub_i + \beta_0\w_{\ell,j}^\perp, 
$$
with $\twonorm{\w_{\ell,j}^\perp}=1$ and $\Ub^T\w_{\ell,j}^\perp=\zero_r$.

Using a calculation similar to the one leading to \eqref{eq:det_21} we can show that \eqref{eq:scal_AO_ij} converges point-wise in $\beta_0,\betab=\begin{bmatrix}\beta_1,\ldots,\beta_r\end{bmatrix}, b_{\ell,j}$ to the following:
\begin{align}
\Dc_\ell(\beta_0,\betab,b_{\ell,j}) = \sqrt{\beta_0^2\sigma^2 + \pi_\ell + \pi_j +  \begin{bmatrix} \betab^T & b_{\ell,j} \end{bmatrix} \A  \begin{bmatrix} \betab \\ b_{\ell,j}  \end{bmatrix} - 2  \db_{\ell,j}^T \begin{bmatrix} \betab \\ b_{\ell,j}  \end{bmatrix}} - \beta_0\sigma\sqrt{\gamma}\label{eq:det_2},
\end{align}
where $\A$ is as in \eqref{eq:Amat} and we have further defined
$$
\db_{\ell,j} := \begin{bmatrix} \pi_\ell\Sigmab\Vb^T\eb_\ell + \pi_j\Sigmab\Vb^T\eb_j \\ \pi_\ell+\pi_j \end{bmatrix}.
$$
Thus,  similar to \eqref{eq:alphab_1} we can compute the minimizer of the deterministic objective in \eqref{eq:det_2}:
\bea
\begin{bmatrix} \widehat\betab \\ \widehat b_{\ell,j} \end{bmatrix} = \begin{bmatrix} - \Deltab^{-1}\Sigmab\Vb^T\left(\pi_\ell(\pib-\eb_\ell) + \pi_j (\pib-\eb_j)\right) \\ \pi_\ell + \pi_j +\pib^T\Vb\Sigmab\Deltab^{-1}\Sigmab\Vb^T\left(\pi_\ell(\pib-\eb_\ell) +\pi_j(\pib-\eb_j) \right) \end{bmatrix}\,,
\end{align}
and 
\begin{align}
\widehat\beta_0 =\frac{1}{\sigma} \sqrt\frac{\gamma}{1-\gamma}\sqrt{\pi_\ell+\pi_j - (\pi_\ell+\pi_j)^2 - \left(\pi_\ell\left(\pib - \eb_\ell\right)+\pi_j\left(\pib - \eb_j\right)\right)^T\Vb\Sigmab \Deltab^{-1} \Sigmab\Vb^T \left(\pi_\ell\left(\pib - \eb_\ell\right)+\pi_j\left(\pib - \eb_j\right)\right)},
\end{align}
where recall that $\Deltab$ is as in \eqref{eq:Deltab_def}.

From the CGMT, we have that 
$\twonorm{\wh_{\ell}+\wh_{j}}^2 \rP \widehat\betab_0^2+\twonorm{\betab}^2.$ 
Combining this with the calculations above, we conclude that
\begin{align}\label{eq:joint_norm_lim}
&\twonorm{\wh_{\ell}+\wh_{j}}^2\rP
 \frac{\gamma}{(1-\gamma)\sigma^2}\left(\pi_\ell+\pi_j\right)\left(1-\pi_\ell-\pi_j\right)\nn \\&+  \left(\pi_\ell\left(\pib - \eb_\ell\right)+\pi_j\left(\pib - \eb_j\right)\right)^T\Vb\Sigmab \Deltab^{-1} \left(\Deltab^{-1} - \frac{\gamma}{(1-\gamma)\sigma^2}\Iden_r\right) \Sigmab\Vb^T \left(\pi_\ell\left(\pib - \eb_\ell\right)+\pi_j\left(\pib - \eb_j\right)\right)
\end{align}

Finally, using \eqref{eq:joint_norm_lim} and \eqref{eq:norm_lim} in \eqref{eq:inp_from_norms} it follows that
\begin{align}
\inp{\w_\ell}{\w_j} \rP \pi_\ell\pi_j\left( - \frac{\gamma}{(1-\gamma)\sigma^2}  +  \left(\pib - \eb_\ell\right)^T\Vb\Sigmab \Deltab^{-1} \left(\Deltab^{-1} - \frac{\gamma}{(1-\gamma)\sigma^2}\Iden_r\right) \Sigmab\Vb^T \left(\pib - \eb_j\right) \right).
\end{align}

\subsection{Orthogonal means}\label{sec:ortho_app}
Here, we specialize the asymptotic predictions of Theorem \ref{thm:LS_GM} to the case of orthogonal means $\inp{\mub_i}{\mub_j}=0,~i\neq j$. 
\begin{corollary}[Orthogonal means]\label{cor:LS_ortho}
Consider the case of orthogonal means, i.e. $\inp{\mub_i}{\mub_j}=0,\forall i\neq j$ and $\gamma<1$ with Euclidean norms given by $\mu_i=\twonorm{\vct{\mu}_i}$. Define the following parameters for $i\in[k]$:
$$
\rho_i:={\pi_i\sigma^2}\big/\left({\sigma^2+\pi_i \mu_i^2}\right) \quad\text{and}\quad \beta_i= {\rho_i\sigma^2}\big/ \Big({\sigma^2-\sum_{i=1}^k \pi_i \rho_i \mu_i^2}\Big).$$
Then, the following asymptotic limits hold for the least-squares classifier, for all $i,j\in[k]$:
\begin{subequations}\label{eq:LS_ortho_form}
\begin{align}
\bh_i &\rP \beta_i\,,\qquad
\inp{\wh_i}{\mub_j} \rP 
\frac{1}{\sigma}(\mathbb{1}_{ij} - \beta_i)  \rho_j\mu_j   \,,
 \\
\inp{\wh_i}{\wh_j} &\rP 
\frac{1}{\sigma^4}\beta_i\beta_j \sum_{\ell=1}^k \rho_\ell^2 \mu_\ell^2 - \frac{1}{\sigma^4}\beta_i \rho_j^2 \mu_j^2  - \frac{1}{\sigma^4}\beta_j \rho_i^2 \mu_i^2  - \frac{\gamma\beta_i\rho_j}{(1-\gamma)\sigma^2} + \frac{\mathbb{1}_{ij}}{\sigma^2}\big(  \frac{\gamma}{(1-\gamma)} \rho_i + \frac{1}{\sigma^2}\rho_i^2 \mu_i^2 \big) 
\end{align}
\end{subequations}
Furthermore, if the means have equal norms $\mu:=\mu_i$ and the classes are balanced: $\pi_i=1/k,~i\in[k]$, then, setting $u_{\rm LS}:=\frac{\mu^2}{\sigma}\sqrt{ \frac{1-\gamma}{\mu^2+k{\gamma\sigma^2}}}$,  it holds that
\begin{align}\label{eq:uls}
\P_e = \P\big\{  G_0 + \max_{j\in[k-1]} G_j \geq u_{\rm LS} \big\},\quad G_0,G_1,\ldots,G_{k-1}\simiid\Nn(0,1)\,.
\end{align}
\end{corollary}

\begin{proof} This is a direct corollary of Theorem \ref{thm:LS_GM}. Indeed,  \eqref{eq:LS_ortho_form} can be derived from  \eqref{eq:LS_GM} after substituting $\Vb=\Iden_k, \Sigmab=\diag{\mu_1,\mu_2,\ldots,\mu_k}$ and some algebra steps that we omit for brevity. 

Instead, we outline below how to conclude \eqref{eq:uls} from \eqref{eq:LS_ortho_form}. 
Assume that $\mu_i=\mu, \forall i\in[k]$ and $\pi_i=\pi=1/k,\forall i\in[k]$.  Recall from \eqref{eq:condPintro} that 
$
\P\left( \text{error} \,|\, \y = \e_c \right) = 1 - \P\left( \Sb_c^{{1}/{2}}\, \z > \tb \right),
$
and using \eqref{eq:LS_ortho_form} it can be checked that
$$\Sb_c=\frac{\pi}{1+\pi \mu^2}\left(\frac{\pi \mu^2}{1+\pi \mu^2}+\frac{\gamma}{1-\gamma}\right)(\Iden_k + \one_k\one_k^T)
\qquad\text{and}\qquad
\tb = -\frac{\pi \mu^2}{1+\pi \mu^2}\one.
$$ 
Thus, setting 
\bea\label{eq:uls2}
u_{\rm LS}:=\mu^2\sqrt{ \frac{\pi}{ \pi \mu^2 + \left(\frac{\gamma}{1-\gamma}\right) \left({1+\pi \mu^2}\right)}}=\mu^2\sqrt{ \frac{1-\gamma}{\mu^2+\gamma/\pi}},
\end{align}
and applying Lemma \eqref{lem:rank1}, the probability of error is given by the advertised expression.
\end{proof}

\section{Least-squares for MLM}

\subsection{Proof of Theorem \ref{propo:LS_log}}

\subsubsection{Computing $\Sigma_{w,\mu}$}
Assume that $\X,\Y$ are generated from the MLM. 

Fix any $\ell\in[k]$. The classifier parameters $\wh_\ell, \bh_\ell$ minimize the following objective function
$
\mathcal{L}_{PO}\left(\vct{w}_\ell,b_\ell\right):=\frac{1}{2n}\twonorm{\X^T\vct{w}_\ell+b_\ell\vct{1}_n-\mtx{Y}_\ell}^2.
$

\vp
\noindent\textbf{Identifying the AO.}~~
To continue further note that by duality we have
\begin{align}
\min_{\vct{w}_\ell, b_\ell}\text{ }\mathcal{L}_{PO}\left(\vct{w}_\ell,b_\ell\right) =\min_{\vct{w}_\ell, b_\ell}\text{ }\max_{\vct{s}}\text{ }\frac{1}{n}\left(\vct{s}^T\X^T\vct{w}_\ell+b_\ell\vct{s}^T\vct{1}_n-\vct{s}^T\mtx{Y}_\ell-\frac{\twonorm{\vct{s}}^2}{2}\right) \label{eq:LS_log_PO}\,,
\end{align}
and the optimization is jointly convex in $(\vct{w}_\ell, b_\ell)$ and concave in $\vct{s}$. Here, note that $\Y_\ell$ depends on the Gaussian matrix $\X$. Thus, before applying the CGMT, we need to break this dependence as follows.  Consider the singular value decomposition  
\begin{align}\label{eq:M_SVD}
\mtx{M}=\mtx{U}\mtx{\Sigma}\mtx{V}^T = 
\begin{bmatrix}\ub_1 & \ub_2 & \ldots & \ub_r \end{bmatrix} \diag{\sigma_1,\sigma_2,\ldots,\sigma_r} \begin{bmatrix}\vb^T_1 \\ \vb_2^T \\ \ldots \\ \vb_r^T \end{bmatrix} ,
\end{align}
with $\mtx{U}\in\R^{d\times r}$, $\mtx{\Sigma}\in\R^{r\times r}$, and $\mtx{V}\in\R^{k\times r}$ where $r=$rank$(\mtx{M})\le k$. For every $i\in[n]$, we decompose $\x_i$ in its projection on the subspace spanned orthogonal columns $\ub_1,\ldots,\ub_r$ as follows:
$$
\x_i = \Ub\Ub^T\X_i + \Pb^\perp\X_i = \Ub\gt_i + \Pb^\perp\x_i, 
$$
where $\Pb^\perp = \Iden_r-\Ub\Ub^T$, and we denote
\bea\label{eq:gt_log}
\Gt:=\begin{bmatrix} \gt_1 & \gt_2 & \ldots & \gt_n \end{bmatrix},\quad \gt_i := \Ub^T\x_i\in\R^r, i\in[n].
\end{align}
Recalling that $\x_i\sim\Nn(\zero,\Iden_d)$ note that
\begin{align}\label{eq:gt_ortho}
\gt_i\sim\Nn(\zero,\Iden_r) \quad\text{and}\quad \gt_i \perp \Pb^\perp\x_i.
\end{align}
Further recall that for all $i\in[n]$, conditioned on $\x_i$
\begin{align}\label{eq:Y_to_gt}
[\Y_\ell]_i\sim{\rm Bern}\left( \frac{e^{\mub_\ell^T\x_i}}{\sum_{\ell^\prime\in[k]}e^{\mub_{\ell^{\prime}}^T\x_i}} \right) \sim {\rm Bern}\left( \frac{e^{\eb_\ell^T\Vb\Sigmab\gt_i}}{\sum_{\ell^\prime\in[k]}e^{\eb_{\ell^\prime}^T\Vb\Sigmab\gt_i}} \right),
\end{align}
where we used \eqref{eq:gt_log} and the SVD decomposition of $\M$. 
In this notation, we can rewrite the PO as follows:
\begin{align*}
\min_{\vct{w}_\ell, b_\ell}\text{ }\max_{\vct{s}}\text{ }\frac{1}{n}\left(\vct{s}^T\X^T\Pb^\perp\vct{w}_\ell+ \vct{s}^T\Gt^T\Ub^T\vct{w}_\ell+b_\ell\vct{s}^T\vct{1}_n-\vct{s}^T\mtx{Y}_\ell-\frac{\twonorm{\vct{s}}^2}{2}\right) 
\end{align*}
From \eqref{eq:gt_ortho} and \eqref{eq:Y_to_gt} notice that $\Y_\ell$ depends only on $\Gt$ and $\Gt$ is independent of $\X^T\Pb^\perp.$ Therefore, the corresponding Auxiliary Optimization (AO) problem becomes
\begin{align}
\min_{\vct{w}_\ell, b_\ell}\text{ }\max_{\vct{s}}\text{ }\frac{1}{n}\left( \twonorm{\Pb^\perp\vct{w}_\ell}\vct{g}^T\vct{s}+ \twonorm{\vct{s}}\vct{h}^T\Pb^\perp\vct{w}_\ell+ \vct{s}^T\Gt^T\Ub^T\vct{w}_\ell+b_\ell\vct{s}^T\vct{1}_n-\vct{s}^T\mtx{Y}_\ell-\frac{\twonorm{\vct{s}}^2}{2}\right), \label{eq:MLM_AO}
\end{align}
where $\g\in\R^{n}$ and $\h\in\R^d$ are iid Gaussian vectors independent of everything else. 

\noindent\textbf{Scalarization of the AO.}~~ 
Maximizing over the direction of $\vct{s}$ and denoting its norm $\beta=\twonorm{\vct{s}}\geq0$ we arrive at
\begin{align}
\min_{\vct{w}_\ell, b_\ell}\text{ }\max_{\beta\ge 0}\text{ } &\frac{1}{n}\left(\beta\twonorm{ \twonorm{\Pb^\perp\vct{w}_\ell}\vct{g}+\Gb^T\Ub^T\vct{w}_\ell+b_\ell\vct{1}_n-\mtx{Y}_\ell}+ \beta\vct{h}^T\Pb^\perp\vct{w}_\ell-\frac{\beta^2}{2}\right)\nn\\
&\quad\quad\quad\quad\quad\quad=\min_{\vct{w}_\ell, b_\ell}\text{ }\frac{1}{2n}\left(\twonorm{ \twonorm{\Pb^\perp\vct{w}_\ell}\vct{g}+\Gt^T\Ub^T\vct{w}_\ell+b_\ell\vct{1}_n-\mtx{Y}_\ell}+ \vct{h}^T\Pb^\perp\vct{w}_\ell\right)_{+}^2 \nn\\
&\quad\quad\quad\quad\quad\quad=\frac{1}{2}\left(\min_{\vct{w}_\ell, b_\ell}\text{ } \frac{1}{\sqrt{n}}\twonorm{ \twonorm{\Pb^\perp\vct{w}_\ell}\vct{g}+\Gt^T\Ub^T\vct{w}_\ell+b_\ell\vct{1}_n-\mtx{Y}_\ell}+\frac{1}{\sqrt{n}} \vct{h}^T\Pb^\perp\vct{w}_\ell\right)_{+}^2\label{eq:koko}
\end{align}
In the remaining, we focus in the inner minimization above. Let us denote
$$
\ab:=\Ub^T\w_\ell\quad\text{and}\quad \alpha_0 = \twonorm{\Pb^\perp\w_\ell}.
$$
Notice that $\ab\perp \Pb^\perp\w_\ell$ and thus the orthogonal decomposition $\w_\ell=\Ub\ab + \Pb^\perp\w_\ell$. With this observation, we can optimize over the direction of $\Pb^T\w_\ell$ in \eqref{eq:koko} by aligning it with $-\Pb^T\h$. With this, the minimization in \eqref{eq:koko} reduces to the following
\begin{align}
\min_{\ab, \alpha_0\geq 0, b_\ell}\text{ } \frac{1}{\sqrt{n}}\twonorm{ \alpha_0\vct{g}+\Gt^T\ab+b_\ell\vct{1}_n-\mtx{Y}_\ell}-\alpha_0\frac{1}{\sqrt{n}} \twonorm{\Pb^\perp\h}\,.\label{eq:koko2}
\end{align}

\noindent\textbf{Convergence of the AO.}~~
First, we argue on point-wise convergence of the objective function in \eqref{eq:koko2}. Fix $\ab,\alpha_0$ and $\bb_\ell$. From the WLLN, $\frac{1}{\sqrt{n}} \twonorm{\Pb^\perp\h}\rP\sqrt{\gamma}$ and
\bea\label{eq:t_conv}
\frac{1}{n}\twonorm{ \alpha_0\vct{g}+\Gt^T\ab+b_\ell\vct{1}_n-\mtx{Y}_\ell}^2 = \frac{1}{n}\sum_{i=1}^n\left({\alpha_0\g_i+\ab^T\gt_i+b_\ell-[\Y_\ell]_i}\right)^2 \rP \E\left[\left({\alpha_0G_0+\ab^T\g+\bb_\ell-Y_\ell}\right)^2\right],
\end{align}
where the expectation is over $\g\sim\Nn(\zero_r,\Iden_r)$ (with some abuse of notation) and 
\begin{align}
Y_\ell \sim{\rm Bern}(V_\ell)\quad\text{and}\quad V_\ell=\frac{e^{\eb_\ell^T\Vb\Sigmab\g}}{\sum_{\ell^\prime=1}^{k}e^{\eb_{\ell^\prime}\Vb\Sigmab\g}}.\label{eq:Y_ell_log}
\end{align}
 Therefore, point-wise on $\ab,\alpha_0$ and $\bb_\ell$, the objective of the AO converges to 
\begin{align}\label{eq:det_log}
\Dc_\ell(\alpha_0,\alphab,b_\ell):=\sqrt{\E\left[\left({\alpha_0G_0+\ab^T\g+\bb_\ell-Y_\ell}\right)^2\right]} - \alpha_0 \sqrt{\gamma}.
\end{align}

Next, with an argument based on convexity and compactness similar to that in ``Convergence analysis of the AO" in Section \ref{sec:proof_LS_GM} it can be argued that the convergence above is uniform. Thus, 
\begin{align}\label{eq:koko2_det}
\eqref{eq:koko2} \rP \min_{\alpha_0\geq 0,\alphab,b_\ell} \Dc_\ell(\alpha_0,\alphab,b_\ell).
\end{align}

\noindent\textbf{Deterministic analysis of the AO.}~~
Here, we solve the deterministic minimization problem in \eqref{eq:koko2_det}. Optimization over $b_\ell$ is straightforward. By setting
$$
b_\ell = \E[Y_\ell] = \E[V_\ell],
$$
we now have to optimize
\begin{align}
\min_{\alpha_0\geq0,\alphab} \sqrt{\alpha_0^2+\E\left[\left(\ab^T\g-Y_\ell\right)^2 \right]- \left(\E[V_\ell]\right)^2}  - \alpha_0 \sqrt{\gamma}.
\end{align}
By direct differentiation and first-order optimality, we compute the optimal values as follows:
\begin{align}
\widetilde\ab_j &= \E[\g_j Y_\ell] = \E[\g_j V_\ell],~j\in[r] \,,\\
\widetilde\alpha_0^2 &= \frac{\gamma}{1-\gamma}\left(\Var[Y_\ell] - \sum_{j=1}^r{\left(\E[\g_j V_\ell]\right)^2}\right) = \frac{\gamma}{1-\gamma}\left(\E[V_\ell]-\left(\E[V_\ell]\right)^2 - \sum_{j=1}^r{\left(\E[\g_j V_\ell]\right)^2}\right)\,.
\end{align}

\vp
\noindent\textbf{Asymptotic Predictions.}~~From the analysis above, we conclude with the following limits about the solution $\bh_\ell,\wh_\ell$ of the PO:
\begin{subequations}
\begin{align}
\bh_\ell &\rP \E[V_\ell] \label{eq:wh_norm_log1} \\
\inp{\mub_c}{\wh_\ell} &\rP \eb_c^T\Vb\Sigmab\E[\g V_\ell],\quad c\in[k]\\
\twonorm{\wh_\ell}^2 &\rP \sum_{j=1}^r\left(\E[{\g_j} V_\ell]\right)^2 + \frac{\gamma}{1-\gamma}\left(\E[V_\ell]-\left(\E[V_\ell]\right)^2 - \sum_{j=1}^r{\left(\E[\g_j V_\ell]\right)^2}\right) \label{eq:wh_norm_log1}\\
&=\frac{\gamma}{1-\gamma}\left(\E[V_\ell]-\left(\E[V_\ell]\right)^2\right) + \frac{1-2\gamma}{1-\gamma} \sum_{j=1}^r{\left(\E[\g_j V_\ell]\right)^2}\label{eq:wh_norm_log}\,.
\end{align}
\end{subequations}

Recall the notation in \eqref{eq:alphas_gen}. Note that $\E[V_\ell]=\pib_\ell$. Moreover, using Gaussian integration by parts Lemma \ref{lem:IBP}, it can be shown that $\E[V_\ell \g ]= \Sigmab\Vb^T \left(\diag{\pib} - \Pib\right)\eb_\ell.$  Using these and writing \label{eq:wh_norm_log1} and \label{eq:wh_norm_log2} in matrix form, we arrive at \eqref{eq:ls_soft_2}.

\subsubsection{Computing $\Sigma_{w,w}$}
Here, we prove \eqref{eq:ls_soft_3}. Specifically, we compute the correlations $\inp{\wh_\ell}{\wh_c},~\ell\neq c\in[k]$ by following the strategy of Section \ref{sec:cross_LS_GM}. Specifically, in view of Lemma \ref{lem:joint_LS} we need to study the following PO:
\begin{align}
\min_{\vct{w}, \vct{b}}\text{ }\max_{\vct{s}}\text{ }\frac{1}{n}\left(\vct{s}^T\X^T\vct{w}+b\vct{s}^T\vct{1}_n-\vct{s}^T(\mtx{Y}_\ell + \Y_c) -\frac{\twonorm{\vct{s}}^2}{2}\right) \label{eq:joint_log}
\end{align}
which is minimized by $\wh_\ell+\wh_c$. Thus the analysis will lead us to an asymptotic formula for $\twonorm{\wh_\ell+\wh_c}$. This when combined with the formulae for $\twonorm{\wh_\ell}$ and $\twonorm{\wh_c}$ in \eqref{eq:wh_norm_log} will give the desired.

The analysis of \eqref{eq:joint_log} is almost identical to the analysis of \eqref{eq:LS_log_PO} in the previous section. Specifically, without repeating all the details for brevity, it can be shown that the AO of \eqref{eq:joint_log} converges to the following (cf. \eqref{eq:det_log}:
\begin{align}
\Dc_\ell(\alpha_0,\alphab,b_\ell):=\sqrt{\E\left[\left({\alpha_0G_0+\ab^T\g+\bb_\ell-Y_{\ell,c}}\right)^2\right]} - \alpha_0 \sqrt{\gamma},
\end{align}
where as before $G_0\sim\Nn(0,1), \g\sim\Nn(\zero_r,\Iden_r)$, only now \eqref{eq:Y_ell_log} is modified to:
\begin{align}
Y_{\ell,c} \sim{\rm Bern}(V_{c} + V_{\ell})\quad\text{and as before:}\quad V_{\ell}=\frac{e^{\eb_\ell^T\Vb\Sigmab\g}}{\sum_{\ell^\prime=1}^{k}e^{\eb_{\ell^\prime}\Vb\Sigmab\g}}.\label{eq:Y_ell_log_2}
\end{align}
With these, it can be shown that 
\begin{align}
\twonorm{\wh_\ell+\wh_c}^2 \rP \sum_{j=1}^r\left(\E[\g_j(V_{c} + V_\ell)]\right)^2 + \frac{\gamma}{1-\gamma}\left(\E[{V_{c} + V_\ell}]-\left(\E[{V_{c} + V_\ell}]\right)^2- \sum_{j=1}^r\left(\E[\g_j(V_{c} + V_\ell)]\right)^2\right)\nn
\end{align}
Combining this with \eqref{eq:wh_norm_log}, we conclude that for $\ell\neq c\in[k]$:
\begin{align}
\inp{\wh_\ell}{\wh_c} \rP  \frac{1-2\gamma}{1-\gamma}\sum_{j=1}^r{\E[\g_jV_c]\E[\g_jV_\ell]}-\frac{\gamma}{1-\gamma}\E[V_c]\E[V_\ell].
\end{align}
This shows \eqref{eq:ls_soft_3} after applying Gaussian integration by parts and expressing it in matrix form; see Lemma \ref{lem:IBP}.

\subsection{Orthogonal means and equal-energy}
Here, we use Theorem \ref{propo:LS_log} to prove that, in contrast to the GMM, in the MLM under orthogonal and equal-energy means: LS outperforms the averaging classifier for large enough sample sizes. Assuming orthogonal means of equal energy $\mu$:
\bea\label{eq:symmetry}
\pib&=\pib_1\one_{k}={(1/k)}\one_{k}, \\
\Pib&=(\Pib_{11}-\Pib_{12})\Iden_{k} + \Pib_{12}\one_k\one_k^T~~\text{with}~~\Pib_{12}=\frac{1-k^2\Pib_{11}^2}{k(k-1)}~~\text{and}~~\Pib_{11} = \E\big[\frac{e^{2\mu G_1}}{\left(\sum_{\ell\in[k]}e^{\mu G_\ell}\,\right)^2} \big].\nn
\end{align}
 Then, 
\bea
\Sigmab_{\w,\w}-\Sigmab_{\w,\mub}\Sigmab_{\mub,\mub}^{-1}\Sigmab_{\w,\mub}^T \rP \frac{\gamma}{1-\gamma}\cdot\left(p\Iden_k - q\one_k\one_k\right),\label{eq:pq}
\end{align}
where we defined 
\bea\label{eq:pq}
p:=\pib_1 - \mu^2(\pib_1-\Pib_{11}+\Pib_{12})^2 \text{ and } q:=(\pib_1^2+\Pib_{12}^2\mu^2k-2\mu^2\Pib_{12}(\pib_1-\Pib_{11}+\Pib_{12})).
\end{align}
 Thus, similar to \eqref{eq:simple_log_Pe_ave} and with the same notation,
\begin{align}\label{eq:LS_ortho_equal_Pe}
\P_{e,\rm LS} \rP \P\Big\{ \arg\max_{\ell\in [k]} \big\{ \sqrt{\frac{\gamma}{1-\gamma}}\cdot\left(p\Iden_k - q\one_k\one_k^T\right)^{1/2}\cdot\gwt + \mu\cdot\left(\diag{\pib}-\Pib\right)\g \big\} \neq Y(\g) \Big\}.
\end{align}
In \eqref{eq:LS_ortho_equal_Pe} (as well as in \eqref{eq:simple_log_Pe_ave}), note that the matrices multiplying $\gwt$ and $\g$ have all the form of a rank one update of a (scaled) identity matrix. It turns out that we can exploit this structure to simplify the formulae for the test error even further. Importantly, this lets us directly compare $\P_{e,\rm LS}$ and $P_{e,\rm Avg}$ of the two classifiers. These are detailed in Section \ref{sec:proof_gamma_star}.

\subsection{Proof of Proposition \ref{propo:gamma_star}}\label{sec:proof_gamma_star}
In \eqref{eq:ave_ortho_equal_Pe} and \eqref{eq:LS_ortho_equal_Pe}, we showed the following limits for orthogonal means of equal-energy $\mu>0$:
\bea
\P_{e,\rm Avg} &\rP \P\big( \arg\max_{\ell\in [k]} \left\{ \gamma\cdot\pib_1\cdot\Iden_k\cdot\gwt + \mu\,\left((\pib_1-\Pib_{11})\cdot\Iden_k +\Pib_{12}\one_k\one_k^T \right)\cdot \g \right\} \neq Y(\g) \big)\nn\\
\P_{e,\rm LS} &\rP \P\Big( \arg\max_{\ell\in [k]} \left\{ \sqrt{\frac{\gamma}{1-\gamma}}\cdot\left(p\Iden_k - q\one_k\one_k^T\right)^{1/2}\cdot\gwt + \mu\,\left((\pib_1-\Pib_{11})\cdot\Iden_k +\Pib_{12}\one_k\one_k^T \right)\cdot \g \right\} \neq Y(\g) \Big),\nn
\end{align}
where $\P(Y(\g)=\ell) = \frac{e^{\mu\g_\ell}}{\sum_{j\in[k]}e^{\mu \g_j}}$ and we have further used \eqref{eq:symmetry} and the notation in \eqref{eq:pq}.

We compare the expression on the RHS in the above display by applying Lemma \ref{lem:Slepian_multi} below  with the following substitutions
\bea
&\g \leftarrow \gwt,\quad \h \leftarrow \g, \quad c(\h)\leftarrow Y(\g)\nn\\
&p_2 \leftarrow \frac{\gamma}{1-\gamma}\left(\pib_1 - \mu^2(\pib_1-\Pib_{11}+\Pib_{12})^2\right),\quad q_2 \leftarrow \frac{\gamma}{1-\gamma}(\pib_1^2+\Pib_{12}^2\mu^2k-2\mu^2\Pib_{12}(\pib_1-\Pib_{11}+\Pib_{12})),\nn\\
&p_1 \leftarrow\gamma \pib_1,\quad\quad q_1 \leftarrow 0.\nn
\end{align}
This shows that with probability 1, $\P_{e,\rm LS} < \P_{e,\rm Avg}$  if and only if $p_2<p_1~\Leftrightarrow ~\gamma<\gamma_\star= \mu^2\left(\pib_1-\Pib_{11}+\Pib_{12}\right)^2\big/\pib_1$. To retrieve \eqref{eq:gamma_star}, recall that $\pib_1=1/k$ and $k\Pib_{11}+(k^2-k)\Pib_{12}=1$. The only thing left to prove is that $\gamma_\star<1$. To see this note that $p_2>0$ from positive semi-definiteness of the Schur matrix in \eqref{eq:pq}. It takes simple algebra to conclude that $p_2>0\implies \gamma_\star<1$. 

\begin{lemma}\label{lem:Slepian_multi}
Let $k\geq 2$, $\g\sim\Nn(\zero,\Iden_k)$ $\h\sim\Nn\left(\zero,\Iden_k\right)$, 
and discrete random variable $c(\h)$ such that $\P(c(\h)=\ell)=e^{\h_\ell}\big/\sum_{j\in[k]}{ e^{\h_j}}$.
Consider the function $F:\R_{>0}\times\R\rightarrow[0,1]$  defined as follows
$$
F(p,q) = \P\left(\arg\max \left\{{\left( p\Iden_k - q \one_k\one_k^T \right)^{1/2}}\,\g  + {\left( \alpha\Iden_k - \beta \one_k\one_k^T \right)^{1/2}}\, \h \right\}\neq c(\h)\right),
$$
such that $p\Iden_k - q \one_k\one_k^T\succ 0$ and fixed $\alpha\Iden_k - \beta \one_k\one_k^T\succ 0$. Then, the following statements are true.
\begin{enumerate}
\item $F(p,q) = \P\left(\arg\max \left\{ \sqrt{p}\cdot\g  + \sqrt{\alpha}\,\h \right\}\neq c(\h)\right)$.
\item For $0<p_2<p_1$ and any $q_1<\frac{p_1}{k},q_2<\frac{p_2}{k}$, it holds that $F(p_2,q_2)<F(p_1,q_1)$.
\end{enumerate}
\end{lemma}
\begin{proof}
Fix any $p>0,q\leq\frac{p}{k}$. Denote $\Tb:=\left( p\Iden_k - q \one_k\one_k^T \right)^{1/2}$ and $\Sb:=\left( \alpha\Iden_k - \beta \one_k\one_k^T \right)^{1/2}$ for convenience. It can be  checked that $\Tb:=\left( \sqrt{p}\Iden_k + \frac{\sqrt{p-qk}-\sqrt{p}}{k} \one_k\one_k^T\right)$ and $\Sb:=\left( \sqrt{\alpha}\Iden_k + \frac{\sqrt{\alpha-\beta k}-\sqrt{\alpha}}{k} \one_k\one_k^T\right)$. From these, it follows directly that
\bea
F(p,q) &= \P\left(\arg\max \left\{ \sqrt{p}\cdot\g  + \sqrt{\alpha}\,\h \right\}\neq c(\h)\right). \nn
\end{align}
This shows the first statement. 

Next, we show the second statement. 
Using the distribution of $c(\h)$ and symmetry we have the following chain of equalities:
\begin{align}
1-F(p,q) &= \P\left\{\arg\max_{j\in[k]} \left\{ \sqrt{p}\cdot\g  + \sqrt{\alpha}\,\h \right\}= c(\h)\right\} \nn\\
& = k\cdot\E\left[\frac{e^{\h_k}}{\sum_{j\in[k]}e^{\h_j}}\cdot\ind{\left\{\arg\max \left\{ \sqrt{p}\cdot\g  + \sqrt{\alpha}\,\h \right\}= k\right\}}\right] \nn\\
& = k\cdot\E\left[\frac{e^{\h_k}}{\sum_{j\in[k]}e^{\h_j}}\cdot\prod_{j\in[k-1]}\ind{\left\{\sqrt{p}\cdot\g_j  + \sqrt{\alpha}\,\h_j < \sqrt{p}\cdot\g_k  + \sqrt{\alpha}\,\h_k\right\}}\right] \nn\\
& = k\cdot\E\left[\frac{e^{\h_k}}{\sum_{j\in[k]}e^{\h_j}}\cdot\prod_{j\in[k-1]}\ind{\left\{ \g_j   <  \g_k  + \frac{\sqrt{\alpha}\,\h_k- \sqrt{\alpha}\,\h_j}{\sqrt{p}}\right\}}\right] \nn\\
& = k\cdot\E\left[\frac{e^{\h_k}}{\sum_{j\in[k]}e^{\h_j}}\cdot\prod_{j\in[k-1]} Q\left(\g_k  + \frac{\sqrt{\alpha}\,\h_j- \sqrt{\alpha}\,\h_k}{\sqrt{p}}\right)\right] =: k\cdot G(\sqrt{p}) \,, \label{eq:Gdef}
\end{align}
where in the last line we used the rotational symmetry of the Gaussian distribution:
$$\P\left\{ \g_j < \g_k  + \frac{\sqrt{\alpha}\,\h_k- \sqrt{\alpha}\,\h_j}{\sqrt{p}} \,|\,\h_1,\ldots,\h_k\right\} = \P\left\{ \g_j > \g_k  + \frac{\sqrt{\alpha}\,\h_j- \sqrt{\alpha}\,\h_k}{\sqrt{p}} \,|\,\h_1,\ldots,\h_k\right\}\,,$$
and the fact that $\g_1,\ldots,\g_{k-1}$ are independent.

Next, we will show that the function $\Gc(\cdot)$ defined above is strictly decreasing in $(0,\infty)$.  Towards this goal, using $Q^\prime(x)=-\frac{1}{\sqrt{2\pi}}e^{-x^2/2}=-\phi(x)$ and using the shorthand 
$$H_{kj}=\h_j-\h_k,~j\in[k],$$
 we may compute the derivative of $\Gc$ at any $s>0$ as follows:
\begin{align}
\frac{\mathrm{d}\Gc(s)}{\mathrm{d}s} &= \sum_{i\in[k-1]}\E\left[\frac{e^{\h_k}}{\sum_{j\in[k]}e^{\h_j}}\cdot\phi\Big( \g_k +  \frac{\sqrt{\alpha}\,H_{ki}}{s}\Big)\cdot \frac{\sqrt{\alpha}\,H_{ki}}{s^2}\cdot\prod_{{j\neq i\in[k-1]}}  Q\left( \g_k + \frac{\sqrt{\alpha}\,H_{kj}}{s}\right)\right] \nn\\
& =  \sum_{i\in[k-1]}\E\left[\frac{1}{\sum_{j\in[k]}e^{H_{kj}}}\cdot\phi\Big( \g_k +  \frac{\sqrt{\alpha}\,H_{ki}}{s}\Big)\cdot \frac{\sqrt{\alpha}\,H_{ki}}{s^2}\cdot\prod_{{j\neq i\in[k-1]}}  Q\left( \g_k + \frac{\sqrt{\alpha}\,H_{kj}}{s}\right)\right] \nn\\
& =  \sum_{i\in[k-1]}\frac{\sqrt{\alpha}}{s^2}\E\left[H_{ki}\cdot \Ac_i\left(\g_k,\{H_{kj}\}_{j\in[k-1]}\right)\right], \label{eq:G<0}
\end{align}
where in the last line we have defined
\bea
 \Ac_i\left(\g_k,\{H_{kj}\}_{j\in[k-1]}\right) := \frac{1}{1+\sum_{j\in[k-1]}e^{H_{kj}}}\cdot\phi\Big( \g_k +  \frac{\sqrt{\alpha}\,H_{ki}}{s}\Big)\cdot\prod_{{j\neq i\in[k-1]}}  Q\left( \g_k + \frac{\sqrt{\alpha}\,H_{kj}}{s}\right),\nn
\end{align}
Next, we use Gaussian integration by parts (GIBP) to further simplify the expression  in \eqref{eq:G<0}. 
 Fix any $i\in[k-1]$. Then, by (GIBP):
\bea
A_i&:=\E\left[H_{ki}\cdot \Ac_i\left(\g_k,\{H_{kj}\}_{j\in[k-1]}\right)\right] \label{eq:G<0_2}\\
&\qquad= \E[H_{ki}^2] \E\left[\frac{\mathrm{d}}{\mathrm{d} H_{ki}}  \Ac_i\left(\g_k,\{H_{kj}\}_{j\in[k-1]}\right)\right]+\sum_{\substack{\ell\in[k-1] \\ \ell\neq i}} \E[H_{ki} \cdot H_{k\ell}] \E\left[\frac{\mathrm{d}}{\mathrm{d} H_{k\ell}}  \Ac_i\left(\g_k,\{H_{kj}\}_{j\in[k-1]}\right)\right] \nn\\
&\qquad= 2 \underbrace{\E\left[\frac{\mathrm{d}}{\mathrm{d} H_{ki}}  \Ac_i\left(\g_k,\{H_{kj}\}_{j\in[k-1]}\right)\right]}_{\rm Term I}+\underbrace{\sum_{\substack{\ell\in[k-1] \\ \ell\neq i}}  \E\left[\frac{\mathrm{d}}{\mathrm{d} H_{k\ell}}  \Ac_i\left(\g_k,\{H_{kj}\}_{j\in[k-1]}\right)\right]}_{\rm Term II},\label{eq:dercomp}
\end{align}
where in the second line, we used the fact that $\h\sim\Nn(\zero,\Iden_k)$ to compute 
$$\E[H_{ki}^2] = 2,\quad\text{and}~~ \E[H_{ki} H_{k\ell}] = 1, ~\ell\neq i,~\ell\in[k-1].$$
We now compute the derivatives in \eqref{eq:dercomp}. First, for any $\ell\in[k-1],~\ell\neq i$,
\bea
&\frac{\mathrm{d}\Ac_i\left(\g_k,\{H_{kj}\}_{j\in[k-1]}\right)}{\mathrm{d} H_{k\ell}} = -\frac{e^{H_{k\ell}}}{\left(1+\sum_{j\in[k-1]}e^{H_{kj}}\right)^2}\cdot\phi\Big( \g_k +  \frac{\sqrt{\alpha}\,H_{ki}}{s}\Big)\cdot\prod_{{j\neq i\in[k-1]}}  Q\left( \g_k + \frac{\sqrt{\alpha}\,H_{kj}}{s}\right)~~=:{\rm Term II(a)_\ell} \nn\\
&-\frac{\sqrt{\alpha}}{s}\cdot\frac{1}{1+\sum_{j\in[k-1]}e^{H_{kj}}}\cdot\phi\Big( \g_k +  \frac{\sqrt{\alpha}\,H_{ki}}{s}\Big)\cdot\phi\Big( \g_k+\frac{\sqrt{\alpha}\,H_{k\ell}}{s}\Big)\cdot\prod_{{j\neq (i,\ell)\in[k-1]}}  Q\left( \g_k + \frac{\sqrt{\alpha}\,H_{kj}}{s}\right)~~=:{\rm Term II(b)_\ell} \nn \\
& \qquad\qquad\qquad= {\rm Term II(a)_\ell} + {\rm Term II(b)_\ell}
\end{align}
Thus, 
\bea\label{eq:TermII}
{\rm Term II} = \sum_{\ell\neq i\in[k-1]} \E\left[{\rm Term II(a)_\ell}\right] + \E\left[{\rm Term II(b)_\ell}\right] =:  
\E\left[{\rm Term II(a)}\right] + N_i ,
\end{align}
where we defined
\bea\label{eq:Ni_def}
N_i &= - \frac{\sqrt{\alpha}}{s}\cdot\E\left[\frac{\phi\Big( \g_k +  \frac{\sqrt{\alpha}\,H_{ki}}{s}\Big)}{1+\sum_{j\in[k-1]}e^{H_{kj}}}\cdot\sum_{\ell\neq i\in[k-1]}\left\{\phi\Big( \g_k+\frac{\sqrt{\alpha}\,H_{k\ell}}{s}\Big)\cdot\prod_{{j\neq (i,\ell)\in[k-1]}}  Q\left( \g_k + \frac{\sqrt{\alpha}\,H_{kj}}{s}\right)\right\}\right] <0.
\end{align}
and we remark for later use that 
\bea\label{eq:TermIIa}
\E\left[{\rm Term II(a)}\right] = \sum_{\ell\neq i\in[k-1]} \E\left[{\rm Term II(a)_\ell}\right]  < 0.
\end{align}
Second, it holds that 
\bea
&\frac{\mathrm{d}\Ac_i\left(\g_k,\{H_{kj}\}_{j\in[k-1]}\right)}{\mathrm{d} H_{ki}}  =  -\frac{e^{H_{ki}}}{\left(1+\sum_{j\in[k-1]}e^{H_{kj}}\right)^2}\cdot\phi\Big( \g_k +  \frac{\sqrt{\alpha}\,H_{ki}}{s}\Big)\cdot\prod_{{j\neq i\in[k-1]}}  Q\left( \g_k + \frac{\sqrt{\alpha}\,H_{kj}}{s}\right) \nn\\
&\qquad\quad+\frac{\mathrm{d}\phi\Big( \g_k +  \frac{\sqrt{\alpha}\,H_{ki}}{s}\Big)}{\mathrm{d} H_{ki}}\cdot\frac{1}{1+\sum_{j\in[k-1]}e^{H_{kj}}}\cdot\prod_{{j\neq i\in[k-1]}}  Q\left( \g_k + \frac{\sqrt{\alpha}\,H_{kj}}{s}\right) \nn\\
&\qquad\qquad~=-\frac{e^{H_{ki}}}{\left(1+\sum_{j\in[k-1]}e^{H_{kj}}\right)^2}\cdot\phi\Big( \g_k +  \frac{\sqrt{\alpha}\,H_{ki}}{s}\Big)\cdot\prod_{{j\neq i\in[k-1]}}  Q\left( \g_k + \frac{\sqrt{\alpha}\,H_{kj}}{s}\right) ~~=:{\rm Term I(a)}\nn\\
&\quad -\frac{\sqrt{\alpha}}{s} \left(\g_k +  \frac{\sqrt{\alpha}\,H_{ki}}{s} \right) \phi\Big( \g_k +  \frac{\sqrt{\alpha}\,H_{ki}}{s}\Big)\cdot\frac{1}{1+\sum_{j\in[k-1]}e^{H_{kj}}}\cdot\prod_{{j\neq i\in[k-1]}}  Q\left( \g_k + \frac{\sqrt{\alpha}\,H_{kj}}{s}\right)~~=:{\rm Term I(b)},\label{eq:TermIb1} \\
&\qquad\qquad~= {\rm Term I(a)} + {\rm Term I(b)},\nn
\end{align}
where in the penultimate line we used the fact that $\phi^\prime(x) = -x\phi(x)$.  Consider the two terms in \eqref{eq:TermIb1}. Clearly,
 \bea\label{eq:TermIa}
\E[{ \rm Term I(a)}]<0.
\end{align}
 For the second term we observe that:
\bea
\E\left[{\rm Term I(b)} \right] &= -\frac{\sqrt{\alpha}}{s}\E\left[ \left(\g_k +  \frac{\sqrt{\alpha}\,H_{ki}}{s} \right) \cdot \Ac_i\left(\g_k,\{H_{kj}\}_{j\in[k-1]}\right)\right]\label{eq:Termba}
\\
&= -\frac{\alpha}{s^2} \cdot \E\left[ H_{ki}\cdot \Ac_i\left(\g_k,\{H_{kj}\}_{j\in[k-1]}\right)\right] - \frac{\sqrt{\alpha}}{s} \E\left[ \g_k\cdot \Ac_i\left(\g_k,\{H_{kj}\}_{j\in[k-1]}\right)\right] \nn\\
& = -\frac{\alpha}{s^2} \cdot A_i - \frac{\sqrt{\alpha}}{s} \E\left[ \g_k\cdot \Ac_i\left(\g_k,\{H_{kj}\}_{j\in[k-1]}\right)\right].\label{eq:Termb}
\end{align}
Moreover, using again GIBP, $\E[\g_k^2]=1$, $\E[\g_k H_{kj}]=0,~j\in[k]$ and the fact that $\phi^\prime(x) = -x\phi(x)$, 
\bea
&\E\left[ \g_k\cdot \Ac_i\left(\g_k,\{H_{kj}\}_{j\in[k-1]}\right)\right]  = \E\left[\frac{\mathrm{d}\Ac_i\left(\g_k,\{H_{kj}\}_{j\in[k]}\right)}{\mathrm{d} \g_k} \right] \nn 
\\
&= -\E\left[ \left(\g_k +  \frac{\sqrt{\alpha}\,H_{ki}}{s} \right) \cdot \Ac_i\left(\g_k,\{H_{kj}\}_{j\in[k-1]}\right)\right]\nn\\
&- \sum_{\ell\neq i\in[k-1]} \E\left[\phi\Big( \g_k +  \frac{\sqrt{\alpha}\,H_{k\ell}}{s}\Big)\cdot\frac{1}{1+\sum_{j\in[k-1]}e^{H_{k\ell}}}\cdot\phi\Big( \g_k +  \frac{\sqrt{\alpha}\,H_{ki}}{s}\Big)\cdot\prod_{{j\neq (i,\ell)\in[k-1]}}  Q\left( \g_k + \frac{\sqrt{\alpha}\,H_{kj}}{s}\right) \right]  \nn\\
&= \frac{s}{\sqrt{\alpha}} \E\left[{\rm Term I(b)} \right] + \frac{s}{\sqrt{\alpha}} N_i,\label{eq:Ni}
\end{align}
where, we have recalled \eqref{eq:Termba} and \eqref{eq:Ni_def}. Using \eqref{eq:Ni} in \eqref{eq:Termb}, we find that
\bea\label{eq:TermIb}
\E\left[{\rm Term I(b)} \right] = -\frac{\alpha}{s^2} \cdot A_i - \E\left[{\rm Term I(b)} \right] - N_i \implies \E\left[{\rm Term I(b)} \right] = -\frac{\alpha}{2s^2} \cdot A_i  - \frac{N_i}{2}.
\end{align}
We are now ready to put things together:
\bea
A_i &= 2\cdot{\rm Term I} + {\rm Term II} \nn \qquad\text{ by \eqref{eq:dercomp}}\\
&= 2\E[{\rm Term I(a)}] + 2\E[{\rm Term I(b)}] + \E[{\rm Term II(a)}] + N_i \qquad\text{by \eqref{eq:TermII}} \nn\\
&= 2\E[{\rm Term I(a)}] -\frac{\alpha}{s^2}\cdot A_i - N_i + \E[{\rm Term II(a)}] + N_i \qquad\text{by \eqref{eq:TermIb}} \nn\\
\implies A_i &= \frac{s^2}{s^2+\alpha}\left( 2\E[{\rm Term I(a)}]  + \E[{\rm Term II(a)}] \right) 
\nn\\
&<0\qquad\qquad\qquad\qquad\qquad\qquad\qquad\qquad\qquad \text{by \eqref{eq:TermIa} and \eqref{eq:TermIIa}}.\nn
\end{align}

From this, \eqref{eq:G<0} and \eqref{eq:G<0_2}, we have shown that $\Gc$ is strictly decreasing in $(0,\infty)$. Recalling the definition of $\Gc$ in \eqref{eq:Gdef}, this implies that $F(p,q)$ is strictly increasing in $p>0$, as desired to complete the proof.
\end{proof}

\section{Weighted LS for GMM (Proof of Theorem \ref{thm:WLS_GM})}

\subsection{Computing $\Sigma_{w,\mu}$}\label{sec:wmu_WLS_GM}
The WLS estimator solves:
\begin{align*}
\underset{\mtx{W}\in\R^{k\times d},\text{ }\vct{b}\in\R^k}{\min}\quad \frac{1}{2n}\fronorm{\left(\mtx{W}\mtx{X}+\vct{b}\vct{1}_n^T-\mtx{Y}\right)\mtx{D}}^2=&\sum_{\ell=1}^k \min_{\vct{w}_\ell, \vct{b}_\ell}\text{ }\frac{1}{2n}\twonorm{\mtx{D}\left(\mtx{X}^T\vct{w}_\ell+b_\ell\vct{1}_n-\mtx{Y}_\ell\right)}^2\\
=&\sum_{\ell=1}^k \min_{\vct{w}_\ell, \vct{b}_\ell}\text{ }\frac{1}{2n}\twonorm{\mtx{D}\left(\mtx{Y}^T\mtx{M}^T\vct{w}_\ell+ \mtx{Z}^T\vct{w}_\ell+b_\ell\vct{1}_n-\mtx{Y}_\ell\right)}^2\,.
\end{align*}

Define 
\begin{align}
\mathcal{L}_{PO}\left(\vct{w}_\ell,b_\ell\right):=\frac{1}{2n}\twonorm{\mtx{D}\left(\mtx{Y}^T\mtx{M}^T\vct{w}_\ell+ \mtx{Z}^T\vct{w}_\ell+b_\ell\vct{1}_n-\mtx{Y}_\ell\right)}^2\label{eq:WLS_PO}\,.
\end{align}

\noindent\textbf{Identifying the AO.}~~By duality we have
\begin{align*}
\min_{\vct{w}_\ell, b_\ell}\text{ }\mathcal{L}_{PO}\left(\vct{w}_\ell,b_\ell\right) =&\min_{\vct{u}, \vct{w}_\ell, b_\ell}\text{ }\max_{\vct{s}}\text{ }\frac{1}{n}\left(\vct{s}^T\mtx{D}\mtx{Y}^T\mtx{M}^T\vct{w}_\ell+ \vct{s}^T\mtx{D}\mtx{Z}^T\vct{w}_\ell+b_\ell\vct{s}^T\mtx{D}\vct{1}_n-\vct{s}^T\mtx{D}\mtx{Y}_\ell-\vct{s}^T\vct{u}+\frac{\twonorm{\vct{u}}^2}{2}\right)
\end{align*}
Note that the above is jointly convex in $(\vct{u}, \vct{w}_\ell, b_\ell)$ and concave in $\vct{s}$. Thus, we consider the Auxiliary Optimization (AO) problem
\begin{align*}
\min_{\vct{u},\vct{w}_\ell, b_\ell}\text{ }\max_{\vct{s}}\text{ }\frac{1}{n}\left(\vct{s}^T\mtx{D}\mtx{Y}^T\mtx{M}^T\vct{w}_\ell+ \sigma\twonorm{\vct{w}_\ell}\vct{g}^T\mtx{D}\vct{s}+ \sigma\twonorm{\mtx{D}\vct{s}}\vct{h}^T\vct{w}_\ell+b_\ell\vct{s}^T\mtx{D}\vct{1}_n-\vct{s}^T\mtx{D}\mtx{Y}_\ell-\vct{s}^T\vct{u}+\frac{\twonorm{\vct{u}}^2}{2}\right),
\end{align*}
where $\vct{g}\in\R^n$ and $\vct{h}\in\R^d$ are independent Gaussian random vectors with i.i.d.~$\mathcal{N}(0,1)$ entries. Moreover, we carry out a change of variable $\vct{s}\rightarrow \mtx{D}\vct{s}$ to arrive at
\begin{align*}
\min_{\vct{u},\vct{w}_\ell, b_\ell}\text{ }\max_{\vct{s}}\text{ }\frac{1}{n}\left(\vct{s}^T\mtx{Y}^T\mtx{M}^T\vct{w}_\ell+ \sigma\twonorm{\vct{w}_\ell}\vct{g}^T\vct{s}+ \sigma\twonorm{\vct{s}}\vct{h}^T\vct{w}_\ell+b_\ell\vct{s}^T\vct{1}_n-\vct{s}^T\mtx{Y}_\ell-\vct{s}^T\mtx{D}^{-1}\vct{u}+\frac{\twonorm{\vct{u}}^2}{2}\right)\,.
\end{align*}

\vp\noindent\textbf{Simplification of the AO.}
Maximizing over the direction of $\vct{s}$ and  setting its norm $\beta=\twonorm{\vct{s}}$ above we arrive at
\begin{align*}
&\min_{\vct{u},\vct{w}_\ell, b_\ell}\text{ }\max_{\beta\ge 0}\text{ }\max_{\vct{s}: \twonorm{\vct{s}}=1}\text{ }\frac{1}{n}\left(\beta\vct{s}^T\mtx{Y}^T\mtx{M}^T\vct{w}_\ell+ \sigma\beta\twonorm{\vct{w}_\ell}\vct{g}^T\vct{s}+ \sigma\beta\vct{h}^T\vct{w}_\ell+b_\ell\beta\vct{s}^T\vct{1}_n-\beta\vct{s}^T\mtx{Y}_\ell-\beta\vct{s}^T\mtx{D}^{-1}\vct{u}+\frac{\twonorm{\vct{u}}^2}{2}\right)\\
&\quad\quad=\min_{\vct{u},\vct{w}_\ell, b_\ell}\text{ }\max_{\beta\ge 0}\text{ }\frac{1}{n}\left(\beta\twonorm{\mtx{Y}^T\mtx{M}^T\vct{w}_\ell+ \sigma\twonorm{\vct{w}_\ell}\vct{g}+b_\ell\vct{1}_n-\mtx{Y}_\ell-\mtx{D}^{-1}\vct{u}}+ \sigma\beta\vct{h}^T\vct{w}_\ell+\frac{\twonorm{\vct{u}}^2}{2}\right)\\
&\quad\quad=\min_{\vct{u}, b_\ell}\text{ }\max_{\beta\ge 0}\text{ }\min_{\vct{w}_\ell}\text{ }\frac{1}{n}\left(\beta\twonorm{\mtx{Y}^T\mtx{M}^T\vct{w}_\ell+ \sigma\twonorm{\vct{w}_\ell}\vct{g}+b_\ell\vct{1}_n-\mtx{Y}_\ell-\mtx{D}^{-1}\vct{u}}+ \sigma\beta\vct{h}^T\vct{w}_\ell+\frac{\twonorm{\vct{u}}^2}{2}\right)\,.
\end{align*}
To continue, consider the singular value decomposition  
\begin{align}\label{eq:M_SVD}
\mtx{M}=\mtx{U}\mtx{\Sigma}\mtx{V}^T = 
\begin{bmatrix}\ub_1 & \ub_2 & \ldots & \ub_r \end{bmatrix} \diag{\sigma_1,\sigma_2,\ldots,\sigma_r} \begin{bmatrix}\vb^T_1 \\ \vb_2^T \\ \ldots \\ \vb_r^T \end{bmatrix}\,,
\end{align}
with $r:=rank(\mtx{M})\le k$ and define the variable $\vct{\alpha}=\mtx{U}^T\vct{w}_\ell$ and $\vct{\alpha}_{\perp}=\mtx{U}_{\perp}^T\vct{w}_\ell$ where $\mtx{U}_{\perp}$ is the orthogonal complement of the columns of $\mtx{U}$. With these definitions the above optimization problem reduces to
\begin{align*}
\min_{\vct{u}, b_\ell}\text{ }\max_{\beta\ge 0}\text{ }\min_{\vct{\alpha}}\text{ }\min_{\vct{\alpha}_\perp}\text{ }\frac{1}{n}\left(\beta\twonorm{\mtx{Y}^T\mtx{V}\mtx{\Sigma}\vct{\alpha}+ \sigma\sqrt{\twonorm{\vct{\alpha}}^2+\twonorm{\vct{\alpha}_\perp}^2}\vct{g}+b_\ell\vct{1}_n-\mtx{Y}_\ell-\mtx{D}^{-1}\vct{u}}+ \sigma\beta\vct{h}^T\mtx{U}\vct{\alpha}+ \sigma\beta\vct{h}^T\mtx{U}_{\perp}\vct{\alpha}_{\perp}+\frac{\twonorm{\vct{u}}^2}{2}\right)\,.
\end{align*}
Decomposing the optimization over $\vct{\alpha}_\perp$ in terms of its direction and norm $\alpha_0=\twonorm{\vct{\alpha}_\perp}$ we arrive at
\begin{align*}
\min_{\vct{u}, b_\ell}\text{ }\max_{\beta\ge 0}\text{ }\min_{\vct{\alpha}}\text{ }\min_{\alpha_0\ge 0}\text{ }\frac{1}{n}\left(\beta\twonorm{\mtx{Y}^T\mtx{V}\mtx{\Sigma}\vct{\alpha}+ \sigma\sqrt{\twonorm{\vct{\alpha}}^2+\alpha_0^2}\vct{g}+b_\ell\vct{1}_n-\mtx{Y}_\ell-\mtx{D}^{-1}\vct{u}}+ \sigma\beta\vct{h}^T\mtx{U}\vct{\alpha}- \sigma\alpha_0\beta\twonorm{\mtx{U}_{\perp}^T\vct{h}}+\frac{\twonorm{\vct{u}}^2}{2}\right)\,.
\end{align*}
Since $\mtx{U}^T\vct{h}$ is $r\le k$ dimensional in our asymptotic regime the term $\frac{\vct{h}^T\mtx{U}\vct{\alpha}}{n}$  can be ignored. Also replacing $\beta$ with $\beta/\sqrt{n}$ we thus arrive at
\begin{align*}
&\min_{\vct{u}, b_\ell}\text{ }\max_{\beta\ge 0}\text{ }\min_{\vct{\alpha}}\text{ }\min_{\alpha_0\ge 0}\text{ }\frac{\beta}{\sqrt{n}}\twonorm{\mtx{Y}^T\mtx{V}\mtx{\Sigma}\vct{\alpha}+ \sigma\sqrt{\twonorm{\vct{\alpha}}^2+\alpha_0^2}\vct{g}+b_\ell\vct{1}_n-\mtx{Y}_\ell-\mtx{D}^{-1}\vct{u}}- \frac{1}{\sqrt{n}}\sigma\alpha_0\beta\twonorm{\mtx{U}_{\perp}^T\vct{h}}+\frac{\twonorm{\vct{u}}^2}{2n}\\
&=\min_{\vct{u}, b_\ell}\text{ }\max_{\beta\ge 0}\text{ }\min_{\vct{\alpha}}\text{ }\min_{\alpha_0\ge 0}\text{ }\min_{\tau\ge 0}\text{ }\frac{\beta}{2n\tau}\twonorm{\mtx{Y}^T\mtx{V}\mtx{\Sigma}\vct{\alpha}+ \sigma\sqrt{\twonorm{\vct{\alpha}}^2+\alpha_0^2}\vct{g}+b_\ell\vct{1}_n-\mtx{Y}_\ell-\mtx{D}^{-1}\vct{u}}^2+\frac{\beta\tau}{2}\\
&\quad\quad\quad\quad\quad\quad\quad\quad\quad\quad\quad\quad-\frac{1}{\sqrt{n}}\sigma\alpha_0\beta\twonorm{\mtx{U}_{\perp}^T\vct{h}}+\frac{\twonorm{\vct{u}}^2}{2n}\\
&=\min_{ b_\ell}\text{ }\max_{\beta\ge 0}\text{ }\min_{\vct{\alpha}}\text{ }\min_{\alpha_0\ge 0}\text{ }\min_{\tau\ge 0}\text{ }\min_{\vct{u}}\text{ }\frac{\beta}{2n\tau}\twonorm{\mtx{Y}^T\mtx{V}\mtx{\Sigma}\vct{\alpha}+ \sigma\sqrt{\twonorm{\vct{\alpha}}^2+\alpha_0^2}\vct{g}+b_\ell\vct{1}_n-\mtx{Y}_\ell-\mtx{D}^{-1}\vct{u}}^2+\frac{\beta\tau}{2}\\
&\quad\quad\quad\quad\quad\quad\quad\quad\quad\quad\quad\quad-\frac{1}{\sqrt{n}}\sigma\alpha_0\beta\twonorm{\mtx{U}_{\perp}^T\vct{h}}+\frac{\twonorm{\vct{u}}^2}{2n}\,.
\end{align*}
Setting the derivative with respect to $\vct{u}$ to zero we arrive at
\begin{align*}
\vct{u}=\frac{\beta}{\tau}\mtx{D}^{-1}\left(\mtx{I}+\frac{\beta}{\tau}\mtx{D}^{-2}\right)^{-1}\left(\mtx{Y}^T\mtx{V}\mtx{\Sigma}\vct{\alpha}+ \sigma\sqrt{\twonorm{\vct{\alpha}}^2+\alpha_0^2}\vct{g}+b_\ell\vct{1}_n-\mtx{Y}_\ell\right)\,.
\end{align*}
Plugging the latter into the above the AO simplifies to
\begin{align*}
&\min_{ b_\ell}\text{ }\max_{\beta\ge 0}\text{ }\min_{\vct{\alpha}}\text{ }\min_{\alpha_0\ge 0}\text{ }\min_{\tau\ge 0}\text{ }\text{ }\frac{\beta}{2\tau n}\text{trace}\left(\vct{t}^T\left(\mtx{I}+\frac{\beta}{\tau}\mtx{D}^{-2}\right)^{-1}\vct{t}\right)-\frac{1}{\sqrt{n}}\sigma\alpha_0\beta\twonorm{\mtx{U}_{\perp}^T\vct{h}}+\frac{\beta\tau}{2}
\end{align*}
where
\begin{align*}
\vct{t}:=\mtx{Y}^T\mtx{V}\mtx{\Sigma}\vct{\alpha}+ \sigma\sqrt{\twonorm{\vct{\alpha}}^2+\alpha_0^2}\vct{g}+b_\ell\vct{1}_n-\mtx{Y}_\ell=\mtx{Y}^T\left(\mtx{V}\mtx{\Sigma}\vct{\alpha}-\vct{e}_\ell\right)+ \sigma\sqrt{\twonorm{\vct{\alpha}}^2+\alpha_0^2}\vct{g}+b_\ell\vct{1}_n
\end{align*}
To continue note that in our asymptotic regime we have
\begin{align*}
\frac{1}{\sqrt{n}}\twonorm{\mtx{U}_{\perp}^T\vct{h}}\text{ }\rP\text{ }\sqrt{\gamma}
\end{align*}
and the cross terms can be ignored so that in an asymptotic sense
\begin{align*}
&\frac{1}{n}\text{trace}\left(\vct{t}^T\left(\mtx{I}+\frac{\beta}{\tau}\mtx{D}^{-2}\right)^{-1}\vct{t}\right)\\
&\quad\quad=\frac{1}{n}\sigma^2\left(\twonorm{\vct{\alpha}}^2+\alpha_0^2\right)\text{trace}\left(\left(\mtx{I}+\frac{\beta}{\tau}\mtx{D}^{-2}\right)^{-1}\right)\\
&\quad\quad\quad+\frac{1}{n}\left(\mtx{Y}^T\left(\mtx{V}\mtx{\Sigma}\vct{\alpha}-\vct{e}_\ell\right)+b_\ell\vct{1}_n\right)^T\left(\mtx{I}+\frac{\beta}{\tau}\mtx{D}^{-2}\right)^{-1}\left(\mtx{Y}^T\left(\mtx{V}\mtx{\Sigma}\vct{\alpha}-\vct{e}_\ell\right)+b_\ell\vct{1}_n\right)
\end{align*}
Therefore we arrive at
\begin{align*}
&\min_{ b_\ell}\text{ }\max_{\beta\ge 0}\text{ }\min_{\vct{\alpha}}\text{ }\min_{\tau\ge 0}\text{ }\min_{\alpha_0\ge 0}\text{ }\text{ }\frac{\beta}{2\tau n}\sigma^2\left(\twonorm{\vct{\alpha}}^2+\alpha_0^2\right)\text{trace}\left(\left(\mtx{I}+\frac{\beta}{\tau}\mtx{D}^{-2}\right)^{-1}\right)-\sigma\alpha_0\beta\sqrt{\gamma}+\frac{\beta\tau}{2}\\
&\quad\quad\quad\quad\quad\quad\quad\quad\quad\quad\quad\quad +\frac{\beta}{2\tau n}\left(\mtx{Y}^T\left(\mtx{V}\mtx{\Sigma}\vct{\alpha}-\vct{e}_\ell\right)+b_\ell\vct{1}_n\right)^T\left(\mtx{I}+\frac{\beta}{\tau}\mtx{D}^{-2}\right)^{-1}\left(\mtx{Y}^T\left(\mtx{V}\mtx{\Sigma}\vct{\alpha}-\vct{e}_\ell\right)+b_\ell\vct{1}_n\right)\,,
\end{align*}
which can be rewritten in the form
\begin{align*}
&\min_{ b_\ell}\text{ }\max_{\beta\ge 0}\text{ }\min_{\vct{\alpha}}\text{ }\min_{\tau\ge 0}\text{ }\text{ }\text{ }\frac{\beta}{2\tau n}\sigma^2\twonorm{\vct{\alpha}}^2\text{trace}\left(\left(\mtx{I}+\frac{\beta}{\tau}\mtx{D}^{-2}\right)^{-1}\right)+\frac{\beta\tau}{2}\\
&\quad\quad\quad\quad\quad\quad\quad\quad\quad +\frac{\beta}{2\tau n}\left(\mtx{Y}^T\left(\mtx{V}\mtx{\Sigma}\vct{\alpha}-\vct{e}_\ell\right)+b_\ell\vct{1}_n\right)^T\left(\mtx{I}+\frac{\beta}{\tau}\mtx{D}^{-2}\right)^{-1}\left(\mtx{Y}^T\left(\mtx{V}\mtx{\Sigma}\vct{\alpha}-\vct{e}_\ell\right)+b_\ell\vct{1}_n\right)\\
&\quad\quad\quad\quad\quad\quad\quad\quad\quad+\frac{\beta}{2\tau n}\sigma^2\alpha_0^2\text{trace}\left(\left(\mtx{I}+\frac{\beta}{\tau}\mtx{D}^{-2}\right)^{-1}\right)-\alpha_0\sigma\beta\sqrt{\gamma}\,.
\end{align*}
To continue further we shall assume $\mtx{D}=\text{diag}\left(\mtx{Y}^T\vct{\omega}\right)$. Note that in this case
\begin{align*}
\frac{1}{n} \text{trace}\left(\left(\mtx{I}+\frac{\beta}{\tau}\mtx{D}^{-2}\right)^{-1}\right)=\frac{1}{n}\sum_{i=1}^n \frac{(\vct{y}_i^T\vct{\omega})^2}{(\vct{y}_i^T\vct{\omega})^2+\frac{\beta}{\tau}}=\sum_{\ell=1}^k \frac{n_\ell}{n}\frac{\omega_\ell^2}{\omega_\ell^2+\frac{\beta}{\tau}}\text{ }\rP\text{ }\sum_{\ell=1}^k\frac{\pi_\ell\omega_\ell^2}{\omega_\ell^2+\frac{\beta}{\tau}}\,.
\end{align*}
Also,
\begin{align*}
&\frac{1}{n}\left(\mtx{Y}^T\left(\mtx{V}\mtx{\Sigma}\vct{\alpha}-\vct{e}_\ell\right)+b_\ell\vct{1}_n\right)^T\left(\mtx{I}+\frac{\beta}{\tau}\mtx{D}^{-2}\right)^{-1}\left(\mtx{Y}^T\left(\mtx{V}\mtx{\Sigma}\vct{\alpha}-\vct{e}_\ell\right)+b_\ell\vct{1}_n\right)\\
&=\frac{1}{n}\left(\mtx{V}\mtx{\Sigma}\vct{\alpha}-\vct{e}_\ell\right)^T\mtx{Y}\left(\mtx{I}+\frac{\beta}{\tau}\mtx{D}^{-2}\right)^{-1}\mtx{Y}^T\left(\mtx{V}\mtx{\Sigma}\vct{\alpha}-\vct{e}_\ell\right)\\
&\quad+\frac{2}{n}b_\ell\vct{1}_n^T\left(\mtx{I}+\frac{\beta}{\tau}\mtx{D}^{-2}\right)^{-1}\mtx{Y}^T\left(\mtx{V}\mtx{\Sigma}\vct{\alpha}-\vct{e}_\ell\right)+\frac{b_\ell^2}{n} \text{trace}\left(\left(\mtx{I}+\frac{\beta}{\tau}\mtx{D}^{-2}\right)^{-1}\right)\\
&\rP\text{ }\left(\mtx{V}\mtx{\Sigma}\vct{\alpha}-\vct{e}_\ell\right)^T\text{diag}\left(\frac{\pi_1\omega_1^2}{\omega_1^2+\frac{\beta}{\tau}},\frac{\pi_2\omega_2^2}{\omega_2^2+\frac{\beta}{\tau}},\ldots,\frac{\pi_k\omega_k^2}{\omega_k^2+\frac{\beta}{\tau}}\right)\left(\mtx{V}\mtx{\Sigma}\vct{\alpha}-\vct{e}_\ell\right)\\
&\quad+2b_\ell\begin{bmatrix}\frac{\pi_1\omega_1^2}{\omega_1^2+\frac{\beta}{\tau}}&\frac{\pi_2\omega_2^2}{\omega_2^2+\frac{\beta}{\tau}}&\ldots &\frac{\pi_k\omega_k^2}{\omega_k^2+\frac{\beta}{\tau}}\end{bmatrix}\left(\mtx{V}\mtx{\Sigma}\vct{\alpha}-\vct{e}_\ell\right)+b_\ell^2\left(\sum_{\ell=1}^k\frac{\pi_\ell\omega_\ell^2}{\omega_\ell^2+\frac{\beta}{\tau}}\right)\,.\\
\end{align*}
Next define 
\begin{align}\label{eq:Amatnew2}
\mtx{A}(\eta):=&\begin{bmatrix} \sigma^2\left(\vct{\pi}^T\vct{\nu}(\eta)\right)\mtx{I}+\Sigmab\Vb^T\text{diag}\left(\vct{\pi}\right)\text{diag}\left(\vct{\nu}(\eta)\right)\Vb\Sigmab & \Sigmab\Vb^T\text{diag}\left(\vct{\nu}(\eta)\right)\vct{\pi} \\ 
\vct{\pi}^T\text{diag}\left(\vct{\nu}(\eta)\right)\Vb\Sigmab & \vct{\pi}^T\vct{\nu}(\eta) \end{bmatrix}\nonumber\\
\cb_\ell :=& \begin{bmatrix} \Sigmab\Vb^T\eb_\ell \\ 1 \end{bmatrix}\,,
\end{align}
where
\begin{align}
\label{nueta}
\vct{\nu}(\eta)=\frac{1}{\gamma}\begin{bmatrix}\frac{\omega_1^2}{\omega_1^2+\eta}\\\frac{\omega_2^2}{\omega_2^2+\eta}\\\ldots \\\frac{\omega_k^2}{\omega_k^2+\eta}\end{bmatrix}\,.
\end{align}
We thus arrive at
\begin{align*}
&\min_{\vct{\alpha}}\text{ }\min_{ b_\ell}\text{ }\max_{\alpha_0\ge 0}\text{ }\max_{\beta\ge 0}\text{ }\min_{\tau\ge 0}\text{ }\text{ }\frac{\gamma\beta}{2\tau}\left(\pi_\ell\vct{\nu}_\ell\left(\frac{\beta}{\tau}\right) +  \begin{bmatrix} \alphab^T & b_\ell \end{bmatrix} \mtx{A}\left(\frac{\beta}{\tau}\right)  \begin{bmatrix} \alphab \\ b_\ell  \end{bmatrix} - 2 \pi_\ell\vct{\nu}_\ell\left(\frac{\beta}{\tau}\right) \cb_\ell^T \begin{bmatrix} \alphab \\ b_\ell  \end{bmatrix}\right)\\
&\quad\quad\quad\quad\quad\quad\quad\quad\quad\quad\quad\quad+\frac{\gamma\beta}{2\tau }\sigma^2\left(\vct{\pi}^T\vct{\nu}\left(\frac{\beta}{\tau}\right)\right)\alpha_0^2-\alpha_0\sigma\beta\sqrt{\gamma}+\frac{\beta\tau}{2}\,.
\end{align*}

\vp\noindent\textbf{Deterministic Analysis of the AO.}~
Setting the derivative of the above with respect to $\alpha_0$ to zero we arrive at
\begin{align*}
\frac{\gamma\beta}{\tau }\sigma^2\left(\vct{\pi}^T\vct{\nu}\left(\frac{\beta}{\tau}\right)\right)\alpha_0-\sigma\beta\sqrt{\gamma}=0\quad\Rightarrow\quad \alpha_0=\frac{\tau}{\sigma\sqrt{\gamma}\left(\vct{\pi}^T\vct{\nu}\left(\frac{\beta}{\tau}\right)\right)}\,.
\end{align*}
Note that the above objective has the form
\begin{align*}
f\left(\frac{\beta}{\tau}\right)-\alpha_0\sigma\beta\sqrt{\gamma}+\frac{\beta\tau}{2}
\end{align*}
with
\begin{align*}
f(\eta):=\frac{\eta\gamma}{2}\left(\pi_\ell\vct{\nu}_\ell(\eta) +  \begin{bmatrix} \alphab^T & b_\ell \end{bmatrix} \mtx{A}(\eta)  \begin{bmatrix} \alphab \\ b_\ell  \end{bmatrix} - 2 \pi_\ell\vct{\nu}_\ell(\eta) \cb_\ell^T \begin{bmatrix} \alphab \\ b_\ell  \end{bmatrix}\right)+\frac{\gamma\eta}{2 }\sigma^2\left(\vct{\pi}^T\vct{\nu}(\eta)\right)\alpha_0^2.
\end{align*}
Thus setting the derivatives with respect to $\beta$ and $\tau$ to zero, we have
\begin{align*}
\frac{1}{\tau}f'\left(\frac{\beta}{\tau}\right)-\alpha_0\sigma\sqrt{\gamma}+\frac{\tau}{2}=0\quad\Rightarrow\quad f'\left(\frac{\beta}{\tau}\right)-\alpha_0\sigma\sqrt{\gamma}\tau+\frac{\tau^2}{2}=0\quad\Rightarrow\quad f'\left(\frac{\beta}{\tau}\right)=\tau^2\left(\frac{1}{\vct{\pi}^T\vct{\nu}\left(\frac{\beta}{\tau}\right)}-\frac{1}{2}\right)
\end{align*}
and
\begin{align*}
-\frac{\beta}{\tau^2}f'\left(\frac{\beta}{\tau}\right)+\frac{\beta}{2}=0\quad\Rightarrow\quad \tau^2=2f'\left(\frac{\beta}{\tau}\right)\,.
\end{align*}
Combining the latter two we conclude that $\vct{\pi}^T\vct{\nu}\left(\frac{\beta}{\tau}\right)=1$.  Thus, $\eta=\frac{\beta}{\tau}$ is the solution to $\vct{\pi}^T\vct{\nu}\left(\eta\right)=1$. To calculate $\tau$ and hence $\alpha_0$ we calculate $f'$ which is equal to
\begin{align*}
f'(\eta)=&\frac{\gamma}{2}\left(\pi_\ell\vct{\nu}_\ell(\eta) +  \begin{bmatrix} \alphab^T & b_\ell \end{bmatrix} \mtx{A}(\eta)  \begin{bmatrix} \alphab \\ b_\ell  \end{bmatrix} - 2 \pi_\ell\vct{\nu}_\ell(\eta) \cb_\ell^T \begin{bmatrix} \alphab \\ b_\ell  \end{bmatrix}\right)+\frac{\gamma}{2 }\sigma^2\left(\vct{\pi}^T\vct{\nu}(\eta)\right)\alpha_0^2+\frac{\gamma\eta}{2 }\sigma^2\alpha_0^2(\vct{\pi}^T\vct{\nu}'(\eta))\\
&+\frac{\gamma\eta}{2}\left(\pi_\ell\vct{\nu}_\ell'(\eta) +  \begin{bmatrix} \alphab^T & b_\ell \end{bmatrix} \mtx{A}'(\eta)  \begin{bmatrix} \alphab \\ b_\ell  \end{bmatrix} - 2 \pi_\ell\vct{\nu}_\ell'(\eta) \cb_\ell^T \begin{bmatrix} \alphab \\ b_\ell  \end{bmatrix}\right)\,,
\end{align*}
where
\begin{align*}
\vct{\nu}'(\eta)=&-\frac{1}{\gamma}\begin{bmatrix}\frac{\omega_1^2}{(\omega_1^2+\eta)^2}\\\frac{\omega_2^2}{(\omega_2^2+\eta)^2}\\\ldots \\\frac{\omega_k^2}{(\omega_k^2+\eta)^2}\end{bmatrix}\\
 \mtx{A}'(\eta):=&\begin{bmatrix} \sigma^2\left(\vct{\pi}^T\vct{\nu}'(\eta)\right)\mtx{I}+\Sigmab\Vb^T\text{diag}\left(\vct{\pi}\right)\text{diag}\left(\vct{\nu}'(\eta)\right)\Vb\Sigmab & \Sigmab\Vb^T\text{diag}\left(\vct{\nu}'(\eta)\right)\vct{\pi} \\ 
\vct{\pi}^T\text{diag}\left(\vct{\nu}'(\eta)\right)\Vb\Sigmab & \vct{\pi}^T\vct{\nu}'(\eta) \end{bmatrix}\nonumber\\
\cb_\ell :=& \begin{bmatrix} \Sigmab\Vb^T\eb_\ell \\ 1 \end{bmatrix}\,.
\end{align*}
Now note that at the optimal point we have
\begin{align*}
f'(\eta)=\frac{\tau^2}{2}=\frac{\gamma}{2 }\sigma^2\left(\vct{\pi}^T\vct{\nu}(\eta)\right)^2\alpha_0^2\,.
\end{align*}
Thus from the above we can conclude that
\begin{align*}
\alpha_0^2=&-\frac{1}{\eta\sigma^2(\vct{\pi}^T\vct{\nu}'(\eta))}\left(\pi_\ell\vct{\nu}_\ell(\eta) +  \begin{bmatrix} \alphab^T & b_\ell \end{bmatrix} \mtx{A}(\eta)  \begin{bmatrix} \alphab \\ b_\ell  \end{bmatrix} - 2 \pi_\ell\vct{\nu}_\ell(\eta) \cb_\ell^T \begin{bmatrix} \alphab \\ b_\ell  \end{bmatrix}\right)\\
&-\frac{1}{\sigma^2(\vct{\pi}^T\vct{\nu}'(\eta))}\left(\pi_\ell\vct{\nu}_\ell'(\eta) +  \begin{bmatrix} \alphab^T & b_\ell \end{bmatrix} \mtx{A}'(\eta)  \begin{bmatrix} \alphab \\ b_\ell  \end{bmatrix} - 2 \pi_\ell\vct{\nu}_\ell'(\eta) \cb_\ell^T \begin{bmatrix} \alphab \\ b_\ell  \end{bmatrix}\right)\,.
\end{align*}
Thus the AO optimization problem reduces to
\begin{align*}
\min_{ b_\ell}\text{ }\min_{\vct{\alpha}}\text{ }\frac{\eta\gamma}{2}\left(\pi_\ell\vct{\nu}_\ell +  \begin{bmatrix} \alphab^T & b_\ell \end{bmatrix} \A  \begin{bmatrix} \alphab \\ b_\ell  \end{bmatrix} - 2 \pi_\ell\vct{\nu}_\ell \cb_\ell^T \begin{bmatrix} \alphab \\ b_\ell  \end{bmatrix}\right)\,,
\end{align*}
where $\eta$ is the solution to
\begin{align*}
\sum_{\ell=1}^k\frac{\pi_\ell\omega_\ell^2}{\omega_\ell^2+\eta}=\gamma\,,
\end{align*}
and
\begin{align}\label{eq:Amatnew22}
\A(\eta):=&\begin{bmatrix} \sigma^2\Iden_r+\Sigmab\Vb^T\text{diag}\left(\vct{\pi}\odot\vct{\nu}\right)\Vb\Sigmab & \Sigmab\Vb^T\left(\vct{\pi}\odot\vct{\nu}\right) \\ 
\left(\vct{\pi}\odot\vct{\nu}\right)^T\Vb\Sigmab & 1 \end{bmatrix}\nonumber\\
\cb_\ell :=& \begin{bmatrix} \Sigmab\Vb^T\eb_\ell \\ 1 \end{bmatrix}\,,
\end{align}
with
\begin{align*}
\vct{\nu}:=\frac{1}{\gamma}\begin{bmatrix}\frac{\omega_1^2}{\omega_1^2+\eta}\\\frac{\omega_2^2}{\omega_2^2+\eta}\\\ldots \\\frac{\omega_k^2}{\omega_k^2+\eta}\end{bmatrix}\,.
\end{align*}

First, note that the matrix $\A$ is positive definite. This can be checked by computing the Schur complement of $\A$:
\begin{align}\label{eq:Deltab_def2}
\Deltab:=\sigma^2\Iden_r+\Sigmab\Vb^T\Pb\Vb\Sigmab:=\sigma^2\Iden_r+\Sigmab\Vb^T\left(\diag{\pib\odot\vct{\nu}}-\left(\pib\odot\vct{\nu}\right)\left(\pib\odot\vct{\nu}\right)^T\right)\Vb\Sigmab \succ \zero_{r\times r}.
\end{align}
Positive definiteness above holds because $\Pb:=\left(\diag{\pib\odot\vct{\nu}}-\left(\pib\odot\vct{\nu}\right)\left(\pib\odot\vct{\nu}\right)^T\right)\succeq \zero_{k\times k}$. Thus the objective is a strictly convex quadratic and is jointly convex in its arguments. We proceed by minimizing the objective over $(\alphab,b_\ell)$ which is equal to
\bea
\begin{bmatrix} \widehat\alphab \\ \widehat b_\ell \end{bmatrix} = \pi_\ell\nu_\ell \A^{-1} \cb_\ell = \pi_\ell \nu_\ell\begin{bmatrix} -\Deltab^{-1}\Sigmab\Vb^T(\pib\odot\vct{\nu}-\eb_\ell) \\ 1+\left(\pib\odot\vct{\nu}\right)^T\Vb\Sigmab\Deltab^{-1}\Sigmab\Vb^T(\pib\odot\vct{\nu}-\eb_\ell) \end{bmatrix}\label{eq:alphab_1t}\,.
\end{align}
Thus, the minimum value attained is
\begin{align}
&-\pi_\ell^2\nu_\ell^2  \begin{bmatrix} -\left(\pib\odot\vct{\nu} - \eb_\ell\right)^T\Vb\Sigmab && 1 \end{bmatrix} \begin{bmatrix} \Deltab^{-1} & \zero \\ \zero^T & 1 \end{bmatrix}\begin{bmatrix} -\Sigmab\Vb^T \left(\pib\odot\vct{\nu} - \eb_\ell\right) \\ 1 \end{bmatrix} \\
&= - \pi_\ell^2\nu_\ell^2\left(1+\left(\pib\odot\vct{\nu} - \eb_\ell\right)^T\Vb\Sigmab \Deltab^{-1} \Sigmab\Vb^T \left(\pib\odot\vct{\nu} - \eb_\ell\right)\right).  \nn
\end{align}
Thus the objective reduces to
\begin{align}
\frac{\eta \gamma}{2}\left(\pi_\ell\nu_\ell\left(1-\pi_\ell\nu_\ell \right)- \pi_\ell^2\nu_\ell^2\left(\pib\odot\vct{\nu} - \eb_\ell\right)^T\Vb\Sigmab \Deltab^{-1} \Sigmab\Vb^T \left(\pib\odot\vct{\nu} - \eb_\ell\right)\right).
\label{eq:final}
\end{align}
Therefore,
\begin{align*}
\alpha_0^2=&-\frac{1}{\eta\sigma^2(\vct{\pi}^T\vct{\nu}'(\eta))}\left(\pi_\ell\nu_\ell \left(1-\pi_\ell\nu_\ell \right)- \pi_\ell^2\nu_\ell^2\left(\pib\odot\vct{\nu} - \eb_\ell\right)^T\Vb\Sigmab \Deltab^{-1} \Sigmab\Vb^T \left(\pib\odot\vct{\nu} - \eb_\ell\right)\right)\\
&-\frac{1}{\sigma^2(\vct{\pi}^T\vct{\nu}'(\eta))}\left(\pi_\ell\vct{\nu}_\ell'(\eta) +  \begin{bmatrix} \alphab^T & b_\ell \end{bmatrix} \mtx{A}'(\eta)  \begin{bmatrix} \alphab \\ b_\ell  \end{bmatrix} - 2 \pi_\ell\vct{\nu}_\ell'(\eta) \cb_\ell^T \begin{bmatrix} \alphab \\ b_\ell  \end{bmatrix}\right)\,,
\end{align*}
where
\begin{align}
\label{alphabetasim}
\begin{bmatrix} \widehat\alphab \\ \widehat b_\ell \end{bmatrix}  = \pi_\ell \nu_\ell\begin{bmatrix} -\Deltab^{-1}\Sigmab\Vb^T(\pib\odot\vct{\nu}-\eb_\ell) \\ 1+\left(\pib\odot\vct{\nu}\right)^T\Vb\Sigmab\Deltab^{-1}\Sigmab\Vb^T(\pib\odot\vct{\nu}-\eb_\ell) \end{bmatrix}
\end{align}
and
\begin{align*}
\vct{\nu}'(\eta)=&-\frac{1}{\gamma}\begin{bmatrix}\frac{\omega_1^2}{(\omega_1^2+\eta)^2}\\\frac{\omega_2^2}{(\omega_2^2+\eta)^2}\\\ldots \\\frac{\omega_k^2}{(\omega_k^2+\eta)^2}\end{bmatrix}\\
\mtx{A}'(\eta):=&\begin{bmatrix} \sigma^2\left(\vct{\pi}^T\vct{\nu}'(\eta)\right)\mtx{I}+\Sigmab\Vb^T\text{diag}\left(\vct{\pi}\right)\text{diag}\left(\vct{\nu}'(\eta)\right)\Vb\Sigmab & \Sigmab\Vb^T\text{diag}\left(\vct{\nu}'(\eta)\right)\vct{\pi} \\ 
\vct{\pi}^T\text{diag}\left(\vct{\nu}'(\eta)\right)\Vb\Sigmab & \vct{\pi}^T\vct{\nu}'(\eta) \end{bmatrix}\,.
\end{align*}

To continue note that
\begin{align*}
&\pi_\ell\vct{\nu}_\ell'(\eta) +  \begin{bmatrix} \alphab^T & b_\ell \end{bmatrix} \mtx{A}'(\eta)  \begin{bmatrix} \alphab \\ b_\ell  \end{bmatrix} - 2 \pi_\ell\vct{\nu}_\ell'(\eta) \cb_\ell^T \begin{bmatrix} \alphab \\ b_\ell  \end{bmatrix}\\
&\quad\quad\quad\quad\quad\quad\quad\quad=\pi_\ell\vct{\nu}_\ell'(\eta) +\pi_\ell^2 \nu_\ell^2\vct{c}_\ell^T\mtx{A}^{-1}\mtx{A}'(\eta)\mtx{A}^{-1}\vct{c}_\ell-2\pi_\ell^2\nu_\ell'\nu_\ell\vct{c}_\ell^T\mtx{A}^{-1}\vct{c}_\ell\,.
\end{align*}
Thus,
\begin{align*}
\alpha_0^2=&-\frac{1}{\eta\sigma^2(\vct{\pi}^T\vct{\nu}'(\eta))}\left(\pi_\ell\nu_\ell \left(1-\pi_\ell\nu_\ell \right)- \pi_\ell^2\nu_\ell^2\left(\pib\odot\vct{\nu} - \eb_\ell\right)^T\Vb\Sigmab \Deltab^{-1} \Sigmab\Vb^T \left(\pib\odot\vct{\nu} - \eb_\ell\right)\right)\\
&-\frac{1}{\sigma^2(\vct{\pi}^T\vct{\nu}'(\eta))}\left(\pi_\ell\vct{\nu}_\ell'(\eta) +\pi_\ell^2 \nu_\ell^2\vct{c}_\ell^T\mtx{A}^{-1}\mtx{A}'(\eta)\mtx{A}^{-1}\vct{c}_\ell-2\pi_\ell^2\nu_\ell'\nu_\ell\vct{c}_\ell^T\mtx{A}^{-1}\vct{c}_\ell\right)\\
=&\frac{\gamma}{\eta\sigma^2\left(\sum_{\ell=1}^k\frac{\pi_\ell\omega_\ell^2}{(\omega_\ell^2+\eta)^2}\right)}\left(\pi_\ell\nu_\ell \left(1-\pi_\ell\nu_\ell \right)- \pi_\ell^2\nu_\ell^2\left(\pib\odot\vct{\nu} - \eb_\ell\right)^T\Vb\Sigmab \Deltab^{-1} \Sigmab\Vb^T \left(\pib\odot\vct{\nu} - \eb_\ell\right)\right)\\
&+\frac{\gamma}{\sigma^2\left(\sum_{\ell=1}^k\frac{\pi_\ell\omega_\ell^2}{(\omega_\ell^2+\eta)^2}\right)}\left(\pi_\ell\vct{\nu}_\ell'(\eta) +\pi_\ell^2 \nu_\ell^2\vct{c}_\ell^T\mtx{A}^{-1}\mtx{A}'(\eta)\mtx{A}^{-1}\vct{c}_\ell-2\pi_\ell^2\nu_\ell'\nu_\ell\vct{c}_\ell^T\mtx{A}^{-1}\vct{c}_\ell\right)\\
:=&\frac{\zeta}{\sigma^2}\left(\pi_\ell\nu_\ell \left(1-\pi_\ell\nu_\ell \right)- \pi_\ell^2\nu_\ell^2\left(\pib\odot\vct{\nu} - \eb_\ell\right)^T\Vb\Sigmab \Deltab^{-1} \Sigmab\Vb^T \left(\pib\odot\vct{\nu} - \eb_\ell\right)\right)\\
&+\frac{\zeta\eta}{\sigma^2}\left(\pi_\ell\vct{\nu}_\ell'(\eta) +\pi_\ell^2 \nu_\ell^2\vct{c}_\ell^T\mtx{A}^{-1}\mtx{A}'(\eta)\mtx{A}^{-1}\vct{c}_\ell-2\pi_\ell^2\nu_\ell'\nu_\ell\vct{c}_\ell^T\mtx{A}^{-1}\vct{c}_\ell\right)%
\end{align*}
where $\zeta:=\frac{\gamma}{\eta\left(\sum_{\ell=1}^k\frac{\pi_\ell\omega_\ell^2}{(\omega_\ell^2+\eta)^2}\right)}$.

\vp
\noindent\textbf{Asymptotic predictions.}~~First, from \eqref{eq:alphab_1t} the bias term converges as follows:
$$
\widehat{b}_\ell \rP \pi_\ell\nu_\ell\left(1+\left(\pib\odot\vct{\nu}\right)^T\Vb\Sigmab\Deltab^{-1}\Sigmab\Vb^T(\pib\odot\vct{\nu}-\eb_\ell)\right).
$$
Thus,
\begin{align*}
\widehat{\vct{b}}\rP\left(\mtx{I}_k-\mtx{P}\Vb\Sigmab\Deltab^{-1}\Sigmab\Vb^T \right)\left(\pib\odot \vct{\nu}\right)\,.
\end{align*}
 Recall that $\alphab = \Ub^T\w_\ell$. Thus, the correlations $\inp{\mub_i}{\w_\ell},~i\in[k]$ converge as follows:
\begin{align}
\M^T\w_\ell = \Vb\Sigmab\Ub^T\w_\ell \rP \Vb\Sigmab\widehat\alphab = -\pi_\ell\nu_\ell\Vb\Sigmab\Deltab^{-1}\Sigmab\Vb^T\left(\pib\odot\vct{\nu}-\eb_\ell\right).\quad
\end{align}
Here, convergence applies element-wise to the entries of the involved random vectors. Moreover, from the analysis above we can predict the limit of the norm $\twonorm{\w_\ell}$. For this, note that $\twonorm{\w_\ell}^2 = \widehat{\alpha}_0^2 + \widehat{\alphab}^T\widehat{\alphab}$. Thus,
\begin{align}\label{eq:norm_limW}
\twonorm{\w_\ell}^2 \rP &\frac{\zeta}{\sigma^2}\pi_\ell\nu_\ell(1-\pi_\ell\nu_\ell)  + \pi_\ell^2\nu_\ell^2 \left(\pib\odot\vct{\nu} - \eb_\ell\right)^T\Vb\Sigmab \Deltab^{-1} \left(\Deltab^{-1} - \frac{\zeta}{\sigma^2}\Iden_{r}\right) \Sigmab\Vb^T \left(\pib\odot\vct{\nu} - \eb_\ell\right)\nonumber\\
&+\frac{\eta\zeta}{\sigma^2}\left(\pi_\ell\vct{\nu}_\ell'(\eta) +\pi_\ell^2 \nu_\ell^2\vct{c}_\ell^T\mtx{A}^{-1}\mtx{A}'\mtx{A}^{-1}\vct{c}_\ell-2\pi_\ell^2\nu_\ell'\nu_\ell\vct{c}_\ell^T\mtx{A}^{-1}\vct{c}_\ell\right) \,.
\end{align}

\subsection{Computing $\Sigma_{w,w}$}\label{sec:cross_WLS_GM}
In the previous section we used the CGMT to predict the bias $\widehat{b}_\ell$, the correlations $\inp{\mub_i}{\wh_\ell},~i[k]$ and the norm $\twonorm{\wh_\ell}$ for all $\ell\in[k]$ members of the multi-output classifier. Here, we show how to compute the limits of the cross-correlations $\inp{\wh_\ell}{\wh_j}, \ell\neq j\in [k]$.

\begin{lemma}\label{lem:joint_WLS}
For $\ell\neq j\in[k]$ let $\wh_\ell$ $\wh_j$ be solutions to the least-squares minimization \eqref{eq:WLS_def}, i.e.,
\begin{align*}
&(\wh_\ell,\wh_j,\widehat{b}_\ell,\widehat{b}_j) \\
&\quad\quad\quad=  \arg\min_{\w_\ell, \w_j, b_\ell, b_j} \left\{ \frac{1}{2n}\twonorm{\mtx{D}\left(\Y_{\ell} - \X^T \w_\ell - b_\ell\one_n\right) }^2 + \frac{1}{2n}\twonorm{\mtx{D}\left(\Y_{j} - \X^T \w_j - b_j\one_n\right) }^2\right\}.
\end{align*}
 Denote $\wh_{\ell,j}:=\wh_\ell+\wh_j$ and $\widehat{b}_{\ell,j}:=\widehat{b}_\ell+\widehat{b}_j$. Then, $(\wh_{\ell,j},\widehat{b}_{\ell,j})$ is a minimizer in the following least-squares problem:
\begin{align}\label{eq:joint_WLS}
(\wh_{\ell,j},\widehat{b}_{\ell,j}) = \arg\min_{\w, b} \frac{1}{2n}\twonorm{\mtx{D}\left(\Y_{\ell} + \Y_{j} - \X^T \w - b\one_n\right) }^2
\end{align}
\end{lemma}
\begin{proof}
Clearly the minimization in \eqref{eq:joint_WLS} is convex. Thus, it suffices to prove that $\wh_\ell+\wh_j$ satisfies the KKT conditions. First, by optimality of $\wh_\ell$, we have that 
$$
\X\mtx{D}^2\left(\Y_\ell-\X^T\wh_\ell-\widehat{b}_\ell\one_n\right) = 0
$$
Similarly, for $\wh_j$:
$$
\X\mtx{D}^2\left(\Y_j-\X^T\wh_j-\widehat{b}_j\one_n\right) = 0.
$$
Adding the equations on the above displays we find that
$$
\X\mtx{D}^2\left(\Y_\ell+\Y_j-\X^T(\wh_j+\wh_\ell)-(\widehat{b}_j+\widehat{b}_\ell)\one_n\right) = 0.
$$
Recognize that this coincides with the optimality condition for \eqref{eq:joint_WLS}. Thus, the proof is complete.
\end{proof}

Thanks to Lemma \ref{lem:joint_WLS}, we can use the CGMT to characterize the limiting behavior of $\twonorm{\wh_\ell+\wh_j}$. Observe that this immediately gives the limit of $\inp{\wh_\ell}{\wh_j}$ since 
\begin{align}\label{eq:inp_from_normsW}
\inp{\wh_\ell}{\wh_j} = \frac{\twonorm{\wh_\ell+\wh_j}^2-\twonorm{\wh_\ell}^2-\twonorm{\wh_j}^2}{2}.
\end{align} 

The analysis of \eqref{eq:joint_WLS} is very similar to that of \eqref{eq:LS_PO}. In particular we use the following decomposition
$$
\w_{\ell,j} = \sum_{i=1}^{r}\beta_i \ub_i + \beta_0\w_{\ell,j}^\perp, 
$$
with $\twonorm{\w_{\ell,j}^\perp}=1$ and $\Ub^T\w_{\ell,j}^\perp=\zero_r$. This allows us to arrive at
\begin{align*}
&\min_{\vct{\beta}}\text{ }\min_{ \vct{b}_{\ell,j}}\text{ }\max_{\alpha_0\ge 0}\text{ }\max_{\beta\ge 0}\text{ }\min_{\tau\ge 0}\text{ }\text{ }\\
&\quad\frac{\gamma\beta}{2\tau}\left(\pi_\ell\vct{\nu}_\ell\left(\frac{\beta}{\tau}\right)+\pi_j\vct{\nu}_j\left(\frac{\beta}{\tau}\right) +  \begin{bmatrix} \vct{\beta}^T & b_{\ell,j} \end{bmatrix} \mtx{A}\left(\frac{\beta}{\tau}\right)  \begin{bmatrix} \vct{\beta} \\ b_{\ell,j}  \end{bmatrix} - 2 \left(\pi_\ell\vct{\nu}_\ell\left(\frac{\beta}{\tau}\right)\vct{c}_\ell+\pi_j\vct{\nu}_j\left(\frac{\beta}{\tau}\right)\vct{c}_j \right)^T \begin{bmatrix} \vct{\beta} \\ b_{\ell,j}  \end{bmatrix}\right)\\
&\quad\quad\quad\quad\quad\quad\quad\quad\quad\quad\quad\quad+\frac{\gamma\beta}{2\tau }\sigma^2\left(\vct{\pi}^T\vct{\nu}\left(\frac{\beta}{\tau}\right)\right)\beta_0^2-\beta_0\sigma\beta\sqrt{\gamma}+\frac{\beta\tau}{2}\,,
\end{align*}
where $\A(\eta)$ and $\vct{c}_\ell$ are as in \eqref{eq:Amatnew2} and $\vct{\nu}(\eta)$ is as in \eqref{nueta}. 

Setting the derivative of the above with respect to $\alpha_0$ to zero we arrive at
\begin{align*}
\frac{\gamma\beta}{\tau }\sigma^2\left(\vct{\pi}^T\vct{\nu}\left(\frac{\beta}{\tau}\right)\right)\beta_0-\sigma\beta\sqrt{\gamma}=0\quad\Rightarrow\quad \beta_0=\frac{\tau}{\sigma\sqrt{\gamma}\left(\vct{\pi}^T\vct{\nu}\left(\frac{\beta}{\tau}\right)\right)}\,.
\end{align*}
Note that the above objective has the form
\begin{align*}
g\left(\frac{\beta}{\tau}\right)-\beta_0\sigma\beta\sqrt{\gamma}+\frac{\beta\tau}{2}\,,
\end{align*}
with
\begin{align*}
g(\eta):=&\frac{\eta\gamma}{2}\left(\pi_\ell\vct{\nu}_\ell\left(\eta\right)+\pi_j\vct{\nu}_j\left(\eta\right) +  \begin{bmatrix} \vct{\beta}^T & b_{\ell,j} \end{bmatrix} \mtx{A}\left(\eta\right)  \begin{bmatrix} \vct{\beta} \\ b_{\ell,j}  \end{bmatrix} - 2 \left(\pi_\ell\vct{\nu}_\ell\left(\eta\right)\vct{c}_\ell+\pi_j\vct{\nu}_j\left(\eta\right)\vct{c}_j \right)^T \begin{bmatrix} \vct{\beta} \\ b_{\ell,j}  \end{bmatrix}\right)\\
&+\frac{\gamma\eta}{2 }\sigma^2\left(\vct{\pi}^T\vct{\nu}(\eta)\right)\beta_0^2.
\end{align*}
Thus, the derivatives with respect to $\beta$ and $\tau$ to zero we have
\begin{align*}
\frac{1}{\tau}g'\left(\frac{\beta}{\tau}\right)-\beta_0\sigma\sqrt{\gamma}+\frac{\tau}{2}=0\quad\Rightarrow\quad g'\left(\frac{\beta}{\tau}\right)-\beta_0\sigma\sqrt{\gamma}\tau+\frac{\tau^2}{2}=0\quad\Rightarrow\quad g'\left(\frac{\beta}{\tau}\right)=\tau^2\left(\frac{1}{\vct{\pi}^T\vct{\nu}\left(\frac{\beta}{\tau}\right)}-\frac{1}{2}\right)
\end{align*}
and
\begin{align*}
-\frac{\beta}{\tau^2}g'\left(\frac{\beta}{\tau}\right)+\frac{\beta}{2}=0\quad\Rightarrow\quad \tau^2=2g'\left(\frac{\beta}{\tau}\right)\,.
\end{align*}
Combining the latter two we conclude that $\vct{\pi}^T\vct{\nu}\left(\frac{\beta}{\tau}\right)=1$.  Thus, $\eta=\frac{\beta}{\tau}$ is the solution to $\vct{\pi}^T\vct{\nu}\left(\eta\right)=1$. To calculate $\tau$ and hence $\beta_0$ we calculate $g'$ which is equal to
\begin{align*}
g'(\eta)=&\frac{\gamma}{2}\left(\pi_\ell\vct{\nu}_\ell\left(\eta\right)+\pi_j\vct{\nu}_j\left(\eta\right) +  \begin{bmatrix} \vct{\beta}^T & b_{\ell,j} \end{bmatrix} \mtx{A}\left(\eta\right)  \begin{bmatrix} \vct{\beta} \\ b_{\ell,j}  \end{bmatrix} - 2 \left(\pi_\ell\vct{\nu}_\ell\left(\eta\right)\vct{c}_\ell+\pi_j\vct{\nu}_j\left(\eta\right)\vct{c}_j \right)^T \begin{bmatrix} \vct{\beta} \\ b_{\ell,j}  \end{bmatrix}\right)\\
&+\frac{\gamma}{2 }\sigma^2\left(\vct{\pi}^T\vct{\nu}(\eta)\right)\beta_0^2+\frac{\gamma\eta}{2 }\sigma^2\beta_0^2(\vct{\pi}^T\vct{\nu}'(\eta))\\
&+\frac{\gamma\eta}{2}\left(\pi_\ell\vct{\nu}_\ell'(\eta) +\pi_j\vct{\nu}_j'\left(\eta\right) +  \begin{bmatrix} \vct{\beta}^T & b_{\ell,j} \end{bmatrix} \mtx{A}'\left(\eta\right)  \begin{bmatrix} \vct{\beta} \\ b_{\ell,j}  \end{bmatrix} - 2 \left(\pi_\ell\vct{\nu}_\ell'\left(\eta\right)\vct{c}_\ell+\pi_j\vct{\nu}_j'\left(\eta\right)\vct{c}_j \right)^T \begin{bmatrix} \vct{\beta} \\ b_{\ell,j}  \end{bmatrix}\right)\,.
\end{align*}
Here, we have
\begin{align*}
\vct{\nu}'(\eta)=&-\frac{1}{\gamma}\begin{bmatrix}\frac{\omega_1^2}{(\omega_1^2+\eta)^2}\\\frac{\omega_2^2}{(\omega_2^2+\eta)^2}\\\ldots \\\frac{\omega_k^2}{(\omega_k^2+\eta)^2}\end{bmatrix}\\
\mtx{A}'(\eta):=&\begin{bmatrix} \sigma^2\left(\vct{\pi}^T\vct{\nu}'(\eta)\right)\mtx{I}+\Sigmab\Vb^T\text{diag}\left(\vct{\pi}\right)\text{diag}\left(\vct{\nu}'(\eta)\right)\Vb\Sigmab & \Sigmab\Vb^T\text{diag}\left(\vct{\nu}'(\eta)\right)\vct{\pi} \\ 
\vct{\pi}^T\text{diag}\left(\vct{\nu}'(\eta)\right)\Vb\Sigmab & \vct{\pi}^T\vct{\nu}'(\eta) \end{bmatrix}\nonumber\\
\cb_\ell :=& \begin{bmatrix} \Sigmab\Vb^T\eb_\ell \\ 1 \end{bmatrix}\,.
\end{align*}
Now note that at the optimal point we have
\begin{align*}
g'(\eta)=\frac{\tau^2}{2}=\frac{\gamma}{2 }\sigma^2\left(\vct{\pi}^T\vct{\nu}(\eta)\right)\beta_0^2\,.
\end{align*}
Thus from the above we can conclude that
\begin{align*}
\beta_0^2=&-\frac{1}{\eta\sigma^2(\vct{\pi}^T\vct{\nu}'(\eta))}\left(\pi_\ell\vct{\nu}_\ell\left(\eta\right)+\pi_j\vct{\nu}_j\left(\eta\right) +  \begin{bmatrix} \vct{\beta}^T & b_{\ell,j} \end{bmatrix} \mtx{A}\left(\eta\right)  \begin{bmatrix} \vct{\beta} \\ b_{\ell,j}  \end{bmatrix} - 2 \left(\pi_\ell\vct{\nu}_\ell\left(\eta\right)\vct{c}_\ell+\pi_j\vct{\nu}_j\left(\eta\right)\vct{c}_j \right)^T \begin{bmatrix} \vct{\beta} \\ b_{\ell,j}  \end{bmatrix}\right)\\
&-\frac{1}{\sigma^2(\vct{\pi}^T\vct{\nu}'(\eta))}\left(\pi_\ell\vct{\nu}_\ell'(\eta) +\pi_j\vct{\nu}_j'\left(\eta\right) +  \begin{bmatrix} \vct{\beta}^T & b_{\ell,j} \end{bmatrix} \mtx{A}'\left(\eta\right)  \begin{bmatrix} \vct{\beta} \\ b_{\ell,j}  \end{bmatrix} - 2 \left(\pi_\ell\vct{\nu}_\ell'\left(\eta\right)\vct{c}_\ell+\pi_j\vct{\nu}_j'\left(\eta\right)\vct{c}_j \right)^T \begin{bmatrix} \vct{\beta} \\ b_{\ell,j}  \end{bmatrix}\right)\,.
\end{align*}
Thus, the AO problem reduces to
\begin{align*}
\min_{ \vct{b}_{\ell,j}}\text{ }\min_{\vct{\beta}}\text{ }\frac{\eta\gamma}{2}\left(\pi_\ell\vct{\nu}_\ell+\pi_j\vct{\nu}_j +  \begin{bmatrix} \vct{\beta}^T & b_{\ell,j} \end{bmatrix} \mtx{A}  \begin{bmatrix} \vct{\beta} \\ b_{\ell,j}  \end{bmatrix} - 2 \left(\pi_\ell\vct{\nu}_\ell\vct{c}_\ell+\pi_j\vct{\nu}_j\vct{c}_j \right)^T \begin{bmatrix} \vct{\beta} \\ b_{\ell,j}  \end{bmatrix}\right)\,,
\end{align*}
where $\eta$ is the solution to
\begin{align*}
\sum_{\ell=1}^k\frac{\pi_\ell\omega_\ell^2}{\omega_\ell^2+\eta}=\gamma\,.
\end{align*}
Thus,  similar to \eqref{alphabetasim} we can compute the minimizer of the deterministic 
\bea
\begin{bmatrix} \widehat\betab \\ \widehat b_{\ell,j} \end{bmatrix} =& \begin{bmatrix} - \Deltab^{-1}\Sigmab\Vb^T\left(\pi_\ell\nu_\ell(\pib\odot\vct{\nu}-\eb_\ell) + \pi_j\nu_j (\pib\odot\vct{\nu}-\eb_j)\right) \\ \pi_\ell \nu_\ell+ \pi_j\nu_j +\left(\pib\odot\vct{\nu}\right)^T\Vb\Sigmab\Deltab^{-1}\Sigmab\Vb^T\left(\pi_\ell\nu_\ell(\pib\odot\vct{\nu}-\eb_\ell) +\pi_j\nu_j(\pib\odot\vct{\nu}-\eb_j) \right) \end{bmatrix}\nn\\
=& \mtx{A}^{-1}\left(\pi_\ell\vct{\nu}_\ell\vct{c}_\ell+\pi_j\vct{\nu}_j\vct{c}_j \right)
\end{align}
and
\begin{align*}
\vct{\nu}'(\eta)=&-\frac{1}{\gamma}\begin{bmatrix}\frac{\omega_1^2}{(\omega_1^2+\eta)^2}\\\frac{\omega_2^2}{(\omega_2^2+\eta)^2}\\\ldots \\\frac{\omega_k^2}{(\omega_k^2+\eta)^2}\end{bmatrix}\\
\mtx{A}'(\eta):=&\begin{bmatrix} \sigma^2\left(\vct{\pi}^T\vct{\nu}'(\eta)\right)\mtx{I}+\Sigmab\Vb^T\text{diag}\left(\vct{\pi}\right)\text{diag}\left(\vct{\nu}'(\eta)\right)\Vb\Sigmab & \Sigmab\Vb^T\text{diag}\left(\vct{\nu}'(\eta)\right)\vct{\pi} \\ 
\vct{\pi}^T\text{diag}\left(\vct{\nu}'(\eta)\right)\Vb\Sigmab & \vct{\pi}^T\vct{\nu}'(\eta) \end{bmatrix}\,.
\end{align*}

To continue note that
\begin{align*}
&\pi_\ell\vct{\nu}_\ell'(\eta) +\pi_j\vct{\nu}_j'\left(\eta\right) +  \begin{bmatrix} \vct{\beta}^T & b_{\ell,j} \end{bmatrix} \mtx{A}'\left(\eta\right)  \begin{bmatrix} \vct{\beta} \\ b_{\ell,j}  \end{bmatrix} - 2 \left(\pi_\ell\vct{\nu}_\ell'\left(\eta\right)\vct{c}_\ell+\pi_j\vct{\nu}_j'\left(\eta\right)\vct{c}_j \right)^T \begin{bmatrix} \vct{\beta} \\ b_{\ell,j}  \end{bmatrix}\\
&\quad\quad=\pi_\ell\vct{\nu}_\ell'(\eta)+\pi_j\vct{\nu}_j'\left(\eta\right) +\left(\pi_\ell\vct{\nu}_\ell'\left(\eta\right)\vct{c}_\ell+\pi_j\vct{\nu}_j'\left(\eta\right)\vct{c}_j \right)^T\mtx{A}^{-1}\mtx{A}'(\eta)\mtx{A}^{-1}\left(\pi_\ell\vct{\nu}_\ell'\left(\eta\right)\vct{c}_\ell+\pi_j\vct{\nu}_j'\left(\eta\right)\vct{c}_j \right)\\
&\quad\quad\quad-2\left(\pi_\ell\vct{\nu}_\ell'\left(\eta\right)\vct{c}_\ell+\pi_j\vct{\nu}_j'\left(\eta\right)\vct{c}_j \right)^T\mtx{A}^{-1}\left(\pi_\ell\vct{\nu}_\ell'\left(\eta\right)\vct{c}_\ell+\pi_j\vct{\nu}_j'\left(\eta\right)\vct{c}_j \right)\,.
\end{align*}
Thus,
\begin{align*}
\beta_0^2=&-\frac{1}{\eta\sigma^2(\vct{\pi}^T\vct{\nu}'(\eta))}\left(\pi_\ell\vct{\nu}_\ell\left(\eta\right)+\pi_j\vct{\nu}_j\left(\eta\right) +  \begin{bmatrix} \vct{\beta}^T & b_{\ell,j} \end{bmatrix} \mtx{A}\left(\eta\right)  \begin{bmatrix} \vct{\beta} \\ b_{\ell,j}  \end{bmatrix} - 2 \left(\pi_\ell\vct{\nu}_\ell\left(\eta\right)\vct{c}_\ell+\pi_j\vct{\nu}_j\left(\eta\right)\vct{c}_j \right)^T \begin{bmatrix} \vct{\beta} \\ b_{\ell,j}  \end{bmatrix}\right)\\
&-\frac{1}{\sigma^2(\vct{\pi}^T\vct{\nu}'(\eta))}\left(\pi_\ell\vct{\nu}_\ell'(\eta) +\pi_j\vct{\nu}_j'\left(\eta\right) +  \begin{bmatrix} \vct{\beta}^T & b_{\ell,j} \end{bmatrix} \mtx{A}'\left(\eta\right)  \begin{bmatrix} \vct{\beta} \\ b_{\ell,j}  \end{bmatrix} - 2 \left(\pi_\ell\vct{\nu}_\ell'\left(\eta\right)\vct{c}_\ell+\pi_j\vct{\nu}_j'\left(\eta\right)\vct{c}_j \right)^T \begin{bmatrix} \vct{\beta} \\ b_{\ell,j}  \end{bmatrix}\right)\\
=&\frac{\zeta}{\sigma^2}\Bigg(\pi_\ell\nu_\ell+\pi_j\nu_j - \left(\pi_\ell\nu_\ell+\pi_j\nu_j\right)^2 \\
&\quad\quad- \left(\pi_\ell\nu_\ell\left(\pib\odot\vct{\nu} - \eb_\ell\right)+\pi_j\nu_j\left(\pib\odot\vct{\nu} - \eb_j\right)\right)^T\Vb\Sigmab \Deltab^{-1} \Sigmab\Vb^T \left(\pi_\ell\nu_\ell\left(\pib\odot\vct{\nu} - \eb_\ell\right)+\pi_j\nu_j\left(\pib\odot\vct{\nu} - \eb_j\right)\right)\Bigg)\\
&+\frac{\zeta\eta}{\sigma^2}
\Bigg(
\pi_\ell\vct{\nu}_\ell'(\eta)+\pi_j\vct{\nu}_j'\left(\eta\right) +\left(\pi_\ell\vct{\nu}_\ell\left(\eta\right)\vct{c}_\ell+\pi_j\vct{\nu}_j\left(\eta\right)\vct{c}_j \right)^T\mtx{A}^{-1}\mtx{A}'(\eta)\mtx{A}^{-1}\left(\pi_\ell\vct{\nu}_\ell\left(\eta\right)\vct{c}_\ell+\pi_j\vct{\nu}_j\left(\eta\right)\vct{c}_j \right)\\
&\quad\quad\quad-2\left(\pi_\ell\vct{\nu}_\ell'\left(\eta\right)\vct{c}_\ell+\pi_j\vct{\nu}_j'\left(\eta\right)\vct{c}_j \right)^T\mtx{A}^{-1}\left(\pi_\ell\vct{\nu}_\ell\left(\eta\right)\vct{c}_\ell+\pi_j\vct{\nu}_j\left(\eta\right)\vct{c}_j \right)
\Bigg)\,,%
\end{align*}
where $\zeta:=\frac{\gamma}{\eta\left(\sum_{\ell=1}^k\frac{\pi_\ell\omega_\ell^2}{(\omega_\ell^2+\eta)^2}\right)}$.
From the CGMT, we have that 
$\twonorm{\wh_{\ell}+\wh_{j}}^2 \rP \widehat\betab_0^2+\twonorm{\betab}^2.$ 
Combining this with the calculations above, we conclude that
\begin{align}\label{eq:joint_norm_limW}
&\twonorm{\wh_{\ell}+\wh_{j}}^2\rP \frac{\zeta}{\sigma^2}\left(\pi_\ell\nu_\ell+\pi_j\nu_j\right)\left(1-\pi_\ell\nu_\ell-\pi_j\nu_j\right)\nn \\
&+  \left(\pi_\ell\nu_\ell\left(\pib\odot\vct{\nu} - \eb_\ell\right)+\pi_j\nu_j\left(\pib\odot\vct{\nu} - \eb_j\right)\right)^T\Vb\Sigmab \Deltab^{-1} \left(\Deltab^{-1} - \frac{\eta}{\sigma^2}\Iden_r\right) \Sigmab\Vb^T \left(\pi_\ell\nu_\ell\left(\pib\odot\vct{\nu} - \eb_\ell\right)+\pi_j\nu_j\left(\pib\odot\vct{\nu} - \eb_j\right)\right)\nn\\
&+\frac{\zeta\eta}{\sigma^2}
\Bigg(
\pi_\ell\vct{\nu}_\ell'(\eta)+\pi_j\vct{\nu}_j'\left(\eta\right) +\left(\pi_\ell\vct{\nu}_\ell\left(\eta\right)\vct{c}_\ell+\pi_j\vct{\nu}_j\left(\eta\right)\vct{c}_j \right)^T\mtx{A}^{-1}\mtx{A}'(\eta)\mtx{A}^{-1}\left(\pi_\ell\vct{\nu}_\ell\left(\eta\right)\vct{c}_\ell+\pi_j\vct{\nu}_j\left(\eta\right)\vct{c}_j \right)\nn\\
&\quad\quad\quad-2\left(\pi_\ell\vct{\nu}_\ell'\left(\eta\right)\vct{c}_\ell+\pi_j\vct{\nu}_j'\left(\eta\right)\vct{c}_j \right)^T\mtx{A}^{-1}\left(\pi_\ell\vct{\nu}_\ell\left(\eta\right)\vct{c}_\ell+\pi_j\vct{\nu}_j\left(\eta\right)\vct{c}_j \right)
\Bigg)\,.
\end{align}

Finally, using \eqref{eq:joint_norm_limW} and \eqref{eq:norm_limW} in \eqref{eq:inp_from_normsW} it follows that
\begin{align}
&\inp{\w_\ell}{\w_j} \rP\nn\\
& \pi_\ell\nu_\ell\pi_j\nu_j\left( - \frac{\zeta}{\sigma^2}  +  \left(\pib\odot\vct{\nu} - \eb_\ell\right)^T\Vb\Sigmab \Deltab^{-1} \left(\Deltab^{-1} - \frac{\zeta}{\sigma^2}\Iden_r\right) \Sigmab\Vb^T \left(\pib\odot\vct{\nu} - \eb_j\right) \right)\nn\\
&+\frac{\zeta\eta}{\sigma^2}\pi_\ell\nu_\ell\pi_j\nu_j
\vct{c}_j ^T\mtx{A}^{-1}\mtx{A}'\mtx{A}^{-1}\vct{c}_\ell-\frac{\zeta\eta}{\sigma^2}\pi_\ell\pi_j(\nu_\ell\nu_j'+\nu_\ell'\nu_j)\vct{c}_j^T\mtx{A}^{-1}\vct{c}_\ell
\nn\\
&= \pi_\ell\nu_\ell\pi_j\nu_j\left( - \frac{\zeta}{\sigma^2}  +  \left(\pib\odot\vct{\nu} - \eb_\ell\right)^T\Vb\Sigmab \Deltab^{-1} \left(\Deltab^{-1} - \frac{\zeta}{\sigma^2}\Iden_r\right) \Sigmab\Vb^T \left(\pib\odot\vct{\nu} - \eb_j\right) \right)\nn\\
&+\frac{\zeta\eta}{\sigma^2}\pi_\ell\nu_\ell\pi_j\nu_j
\left(\vct{e}_j ^T\begin{bmatrix}\mtx{\Sigma}\mtx{V}^T \\\vct{1}^T\end{bmatrix}^T\mtx{A}^{-1}\mtx{A}'\mtx{A}^{-1}\begin{bmatrix}\mtx{\Sigma}\mtx{V}^T \\\vct{1}^T\end{bmatrix}\vct{e}_\ell\right)\nn\\
&-\frac{\eta\zeta}{\sigma^2}\pi_\ell\pi_j(\nu_\ell\nu_j'+\nu_\ell'\nu_j)\left(\vct{e}_j^T\begin{bmatrix}\mtx{\Sigma}\mtx{V}^T \\\vct{1}^T\end{bmatrix}^T\mtx{A}^{-1}\begin{bmatrix}\mtx{\Sigma}\mtx{V}^T \\\vct{1}^T\end{bmatrix}\vct{e}_\ell\,.
\right)
\end{align}
Putting everything together we arrive at
\begin{align*}
\Sigmab_{\w,\w} \rP&\frac{\zeta}{\sigma^2} \Pb+ \Pb\Vb\Sigmab \Deltab^{-1} \Big(\Deltab^{-1} - \frac{\zeta}{\sigma^2}\Iden_r\Big) \Sigmab\Vb^T \Pb+\frac{\zeta\eta}{\sigma^2}\mtx{Q}\,,
\end{align*}
where
\begin{align*}
\mtx{Q}:=&\text{diag}\left(\vct{\pi}\odot \vct{\nu}'\right)+\text{diag}\left(\vct{\pi}\odot \vct{\nu}\right)\begin{bmatrix}\mtx{\Sigma}\mtx{V}^T \\\vct{1}^T\end{bmatrix}^T\left(\mtx{A}^{-1}\mtx{A}'\mtx{A}^{-1}\right)\begin{bmatrix}\mtx{\Sigma}\mtx{V}^T \\\vct{1}^T\end{bmatrix}\text{diag}\left(\vct{\pi}\odot \vct{\nu}\right)\nn\\
&-\text{diag}\left(\vct{\pi}\odot \vct{\nu}'\right)\begin{bmatrix}\mtx{\Sigma}\mtx{V}^T \\\vct{1}^T\end{bmatrix}^T\mtx{A}^{-1}\begin{bmatrix}\mtx{\Sigma}\mtx{V}^T \\\vct{1}^T\end{bmatrix}\text{diag}\left(\vct{\pi}\odot \vct{\nu}\right)-\text{diag}\left(\vct{\pi}\odot \vct{\nu}\right)\begin{bmatrix}\mtx{\Sigma}\mtx{V}^T \\\vct{1}^T\end{bmatrix}^T\mtx{A}^{-1}\begin{bmatrix}\mtx{\Sigma}\mtx{V}^T \\\vct{1}^T\end{bmatrix}\text{diag}\left(\vct{\pi}\odot \vct{\nu}'\right)
\end{align*}
and as mentioned earlier 
\begin{align*}
\mtx{A}':=&\mtx{A}'(\eta):=\begin{bmatrix} \sigma^2\left(\vct{\pi}^T\vct{\nu}'(\eta)\right)\mtx{I}+\Sigmab\Vb^T\text{diag}\left(\vct{\pi}\right)\text{diag}\left(\vct{\nu}'(\eta)\right)\Vb\Sigmab & \Sigmab\Vb^T\text{diag}\left(\vct{\nu}'(\eta)\right)\vct{\pi} \\ 
\vct{\pi}^T\text{diag}\left(\vct{\nu}'(\eta)\right)\Vb\Sigmab & \vct{\pi}^T\vct{\nu}'(\eta) \end{bmatrix}\nonumber\\
\mtx{A}:=&\mtx{A}(\eta):=\begin{bmatrix} \sigma^2\left(\vct{\pi}^T\vct{\nu}(\eta)\right)\mtx{I}+\Sigmab\Vb^T\text{diag}\left(\vct{\pi}\right)\text{diag}\left(\vct{\nu}(\eta)\right)\Vb\Sigmab & \Sigmab\Vb^T\text{diag}\left(\vct{\nu}(\eta)\right)\vct{\pi} \\ 
\vct{\pi}^T\text{diag}\left(\vct{\nu}(\eta)\right)\Vb\Sigmab & \vct{\pi}^T\vct{\nu}(\eta) \end{bmatrix}\,.
\end{align*}
Let us end by simplifying $\mtx{Q}$ to this aim 
\begin{align*}
\A^{-1} \begin{bmatrix}\mtx{\Sigma}\mtx{V}^T \\\vct{1}^T\end{bmatrix}\text{diag}\left(\vct{\pi}\odot \vct{\nu}\right)&= \begin{bmatrix} \Iden & \zero \\ -\widetilde{\vct{\pi}}^T\Vb\Sigmab & 1 \end{bmatrix}\begin{bmatrix} \Deltab^{-1} & \zero \\ \zero^T & 1 \end{bmatrix} \begin{bmatrix} \Iden & -\Sigmab\Vb^T\widetilde{\vct{\pi}} \\ \zero^T & 1 \end{bmatrix}\begin{bmatrix}\mtx{\Sigma}\mtx{V}^T \\\vct{1}^T\end{bmatrix}\text{diag}\left(\vct{\pi}\odot \vct{\nu}\right)\\
&= \begin{bmatrix} \Iden & \zero \\ -\widetilde{\vct{\pi}}^T\Vb\Sigmab & 1 \end{bmatrix}\begin{bmatrix} \Deltab^{-1} & \zero \\ \zero^T & 1 \end{bmatrix} \begin{bmatrix}\mtx{\Sigma}\mtx{V}^T\left(\mtx{I}-\widetilde{\vct{\pi}}\vct{1}^T\right) \\\vct{1}^T\end{bmatrix}\text{diag}\left(\vct{\pi}\odot \vct{\nu}\right)\\
&= \begin{bmatrix} \Iden & \zero \\ -\widetilde{\vct{\pi}}^T\Vb\Sigmab & 1 \end{bmatrix} \begin{bmatrix}\Deltab^{-1}\mtx{\Sigma}\mtx{V}^T\left(\mtx{I}-\widetilde{\vct{\pi}}\vct{1}^T\right) \\\vct{1}^T\end{bmatrix}\text{diag}\left(\vct{\pi}\odot \vct{\nu}\right)\\
&= \begin{bmatrix}\Deltab^{-1}\mtx{\Sigma}\mtx{V}^T\left(\mtx{I}-\widetilde{\vct{\pi}}\vct{1}^T\right) \\
-\widetilde{\vct{\pi}}^T\Vb\Sigmab\Deltab^{-1}\mtx{\Sigma}\mtx{V}^T\left(\mtx{I}-\widetilde{\vct{\pi}}\vct{1}^T\right)+\vct{1}^T\end{bmatrix}\text{diag}\left(\vct{\pi}\odot \vct{\nu}\right)\\
&= \begin{bmatrix}\Deltab^{-1}\mtx{\Sigma}\mtx{V}^T\left(\text{diag}(\widetilde{\vct{\pi}})-\widetilde{\vct{\pi}}\widetilde{\vct{\pi}}^T\right) \\
-\widetilde{\vct{\pi}}^T\Vb\Sigmab\Deltab^{-1}\mtx{\Sigma}\mtx{V}^T\left(\text{diag}(\widetilde{\vct{\pi}})-\widetilde{\vct{\pi}}\widetilde{\vct{\pi}}'^T\right)+\widetilde{\vct{\pi}}^T\end{bmatrix}
\end{align*}
Thus, defining $\widetilde{\vct{\pi}}'=\vct{\pi}\odot \vct{\nu}'$ we have
\begin{align*}
\text{diag}\left(\vct{\pi}\odot \vct{\nu}'\right)\begin{bmatrix}\mtx{\Sigma}\mtx{V}^T \\\vct{1}^T\end{bmatrix}^T\A^{-1} \begin{bmatrix}\mtx{\Sigma}\mtx{V}^T \\\vct{1}^T\end{bmatrix}\text{diag}\left(\vct{\pi}\odot \vct{\nu}\right)&=\text{diag}\left(\vct{\pi}\odot \vct{\nu}'\right)\begin{bmatrix}\mtx{\Sigma}\mtx{V}^T \\\vct{1}^T\end{bmatrix}^T\begin{bmatrix}\Deltab^{-1}\mtx{\Sigma}\mtx{V}^T\left(\text{diag}(\widetilde{\vct{\pi}})-\widetilde{\vct{\pi}}\widetilde{\vct{\pi}}^T\right) \\
-\widetilde{\vct{\pi}}^T\Vb\Sigmab\Deltab^{-1}\mtx{\Sigma}\mtx{V}^T\left(\text{diag}(\widetilde{\vct{\pi}})-\widetilde{\vct{\pi}}\widetilde{\vct{\pi}}'^T\right)+\widetilde{\vct{\pi}}^T\end{bmatrix}\\
&=\left(\text{diag}\left(\widetilde{\vct{\pi}}'\right)-\widetilde{\vct{\pi}}'\widetilde{\vct{\pi}}^T\right)\mtx{V}\mtx{\Sigma}\Deltab^{-1}\mtx{\Sigma}\mtx{V}^T\left(\text{diag}(\widetilde{\vct{\pi}})-\widetilde{\vct{\pi}}\widetilde{\vct{\pi}}^T\right)+\widetilde{\vct{\pi}}'\widetilde{\vct{\pi}}^T
\end{align*}
Using the above and recalling $\widetilde{\vct{\pi}}'=\vct{\pi}\odot \vct{\nu}'$ we arrive at
\begin{align}\label{eq:Qmat_app}
\mtx{Q}=&\text{diag}(\widetilde{\vct{\pi}}')\nn\\
&+\begin{bmatrix}\Deltab^{-1}\mtx{\Sigma}\mtx{V}^T\left(\text{diag}(\widetilde{\vct{\pi}})-\widetilde{\vct{\pi}}\widetilde{\vct{\pi}}^T\right) \nn\\
-\widetilde{\vct{\pi}}^T\Vb\Sigmab\Deltab^{-1}\mtx{\Sigma}\mtx{V}^T\left(\text{diag}(\widetilde{\vct{\pi}})-\widetilde{\vct{\pi}}\widetilde{\vct{\pi}}^T\right)+\widetilde{\vct{\pi}}^T\end{bmatrix}^T\mtx{A}'\begin{bmatrix}\Deltab^{-1}\mtx{\Sigma}\mtx{V}^T\left(\text{diag}(\widetilde{\vct{\pi}})-\widetilde{\vct{\pi}}\widetilde{\vct{\pi}}^T\right) \\
-\widetilde{\vct{\pi}}^T\Vb\Sigmab\Deltab^{-1}\mtx{\Sigma}\mtx{V}^T\left(\text{diag}(\widetilde{\vct{\pi}}')-\widetilde{\vct{\pi}}\widetilde{\vct{\pi}}^T\right)+\widetilde{\vct{\pi}}^T\end{bmatrix}\nn\\
&-\left(\text{diag}\left(\widetilde{\vct{\pi}}'\right)-\widetilde{\vct{\pi}}'\widetilde{\vct{\pi}}^T\right)\mtx{V}\mtx{\Sigma}\Deltab^{-1}\mtx{\Sigma}\mtx{V}^T\left(\text{diag}(\widetilde{\vct{\pi}})-\widetilde{\vct{\pi}}\widetilde{\vct{\pi}}^T\right)-\widetilde{\vct{\pi}}'\widetilde{\vct{\pi}}^T\nn\\
&-\left(\text{diag}\left(\widetilde{\vct{\pi}}\right)-\widetilde{\vct{\pi}}\widetilde{\vct{\pi}}^T\right)\mtx{V}\mtx{\Sigma}\Deltab^{-1}\mtx{\Sigma}\mtx{V}^T\left(\text{diag}(\widetilde{\vct{\pi}}')-\widetilde{\vct{\pi}}\widetilde{\vct{\pi}}'^T\right)-\widetilde{\vct{\pi}}\widetilde{\vct{\pi}}'^T
\end{align}
where
\begin{align*}
\mtx{A}'=&\begin{bmatrix} \sigma^2\left(\widetilde{\vct{\pi}}'^T\vct{1}\right)\mtx{I}+\Sigmab\Vb^T\text{diag}\left(\widetilde{\vct{\pi}}'\right)\Vb\Sigmab & \Sigmab\Vb^T\widetilde{\vct{\pi}}' \\ 
\widetilde{\vct{\pi}}'^T\Vb\Sigmab & \widetilde{\vct{\pi}}'^T\vct{1} \end{bmatrix}\nonumber\\
\mtx{A}=&\begin{bmatrix} \sigma^2\mtx{I}+\Sigmab\Vb^T\text{diag}\left(\widetilde{\vct{\pi}}\right)\Vb\Sigmab & \Sigmab\Vb^T\widetilde{\vct{\pi}} \\ 
\widetilde{\vct{\pi}}^T\Vb\Sigmab & 1 \end{bmatrix}\,.
\end{align*}
Using the above the cross-correlation matrix $\Sigmab_{\w,\w}$ is given by
\begin{align*}
\Sigmab_{\w,\w} \rP&\frac{\zeta}{\sigma^2} \Pb+ \Pb\Vb\Sigmab \Deltab^{-1} \Big(\Deltab^{-1} - \frac{\zeta}{\sigma^2}\Iden_r\Big) \Sigmab\Vb^T \Pb+\frac{\zeta\eta}{\sigma^2}\mtx{Q}\,.
\end{align*}


\section{Weighted LS for MLM (Proof of Theorem \ref{thm:WLS_MLM})}\label{sec:WLS_MLM}

Let $\Db:=\Db^{(n)}:=\diag{D_1,\ldots,D_n}$ be a diagonal matrix with non-zero diagonal entries. In particular, assume that the diagonal entries of $\Db$ are distributed $D_i\simiid D$ where the random variable $D$ may depend on the entries of the matrix of response variables $\Y$. Here, we focus on the following  setting:
%
%
\begin{align}\label{eq:D_common}
\Db =  \sum_{j\in[k]}\diag{ \omega_j \Y_j},\quad \omega_j\geq 0,~j\in[k].
\end{align}
Specifically, for \eqref{eq:D_common}, we have $D_i\simiid D$ with $D=\omega_\ell Y_\ell + \sum_{i\neq\ell\in[k]}\omega_i Y_i $, where for all $c\in[k]$:
\begin{align}
\P\left([Y_1, Y_2, \ldots, Y_k]^T=\eb_c\right)=V_c=\frac{e^{\eb_c^T\Vb\Sigmab\g}}{\sum_{\ell^\prime=1}^{k}e^{\eb_{\ell^\prime}\Vb\Sigmab\g}}\label{eq:Y_ell_log_WLS},
\end{align}
$\M\M^T = \V\Sigmab^2\V^T$, and $\g\sim\Nn(\zero,\Iden_r).$

With these, we consider the weighted least-squares (WLS) solution for $\ell\in[k]$:
\begin{align}
(\widehat\w_\ell,\widehat b) = \arg\min_{\w,b}\mathcal{L}_{PO}\left(\w,b\right):=\frac{1}{2n}\twonorm{\Db\left(\X^T\w+b\vct{1}_n-\mtx{Y}_\ell\right)}^2,\nn
\end{align}
where $\Db$ is as in \eqref{eq:D_common}.
In fact, it is convenient to rewrite the above as follows:
\begin{align}\label{eq:PO_WLS}
(\widehat\w_\ell,\widehat b) = \arg\min_{\w,b,\ub}~\max_{\s}~\frac{1}{n}\left(\s^T\Db\X^T\w+b\s^T\Db\vct{1}_n\s^T\Db\mtx{Y}_\ell -\s^T\ub + \frac{\twonorm{\ub}^2}{2}\right).
\end{align}

\noindent\textbf{Identifying the AO.}
~~The PO in \eqref{eq:PO_WLS} is very similar to \eqref{eq:LS_log_PO}. In particular, following step by step the same decomposition trick as in Section \ref{sec:MLM_proof_1}, it can be shown that the AO corresponding to \eqref{eq:PO_WLS} becomes (cf. \eqref{eq:MLM_AO}) 
\begin{align*}
\min_{\vct{w}_\ell, \vct{b}_\ell,\ub}\text{ }\max_{\vct{s}}\text{ }\frac{1}{n}\left( \twonorm{\Pb^\perp\vct{w}_\ell}\vct{g}^T\Db\vct{s}+ \twonorm{\Db\vct{s}}\vct{h}^T\Pb^\perp\vct{w}_\ell+ \vct{s}^T\Db\Gt^T\Ub^T\vct{w}_\ell+b_\ell\vct{s}^T\Db\vct{1}_n-\vct{s}^T\Db\mtx{Y}_\ell - \ub^T\s + \frac{\twonorm{\vct{\ub}}^2}{2}\right), 
\end{align*}
where we use the same notation as in Section \ref{sec:MLM_proof_1} for $\Pb^\perp, \Ub, \Gt, \g$ and $\h$. Recall also the relation of $\Y_\ell$ to $\Gt$ in \eqref{eq:Y_to_gt}.

\noindent\textbf{Scalarization of the AO.}
~~We start the process of simplifying the AO by setting $\beta:=\twonorm{\Db\s}\big/\sqrt{n}$ and optimizing over the direction of $\Db\s$ to equivalently write the AO as
\begin{align}\label{eq:koko_WLS}
\min_{\vct{w}_\ell, \vct{b}_\ell,\ub}\text{ }\max_{\beta\geq0}\text{ }\frac{1}{\sqrt{n}}\left( \beta\twonorm{ \twonorm{\Pb^\perp\vct{w}_\ell}\vct{g}+ \Gt^T\Ub^T\vct{w}_\ell+b_\ell\vct{1}_n-\mtx{Y}_\ell - \Db^{-1}\ub} + \beta \vct{h}^T\Pb^\perp\vct{w}_\ell \right) + \frac{\twonorm{\vct{\ub}}^2}{2n}, 
\end{align}
Next, focus on the minimization over $\w_\ell$. Let us denote
$$
\ab:=\Ub^T\w_\ell\quad\text{and}\quad \alpha_0 = \twonorm{\Pb^\perp\w_\ell}.
$$
Notice that $\ab\perp \Pb^\perp\w_\ell$ and thus the orthogonal decomposition $\w_\ell=\Ub\ab + \Pb^\perp\w_\ell$. With this observation, note that the optimal direction of $\Pb^T\w_\ell$ in \eqref{eq:koko_WLS} aligns with $\Pb^T\h$ for all values of $\beta$. Therefore, \eqref{eq:koko_WLS} reduces to 
\begin{align}\label{eq:koko_WLS_1}
\min_{\ab, \alpha_0\geq0, \vct{b}_\ell,\ub}\text{ }\max_{\beta\geq0}\text{ }\frac{1}{\sqrt{n}}\left( \beta\twonorm{ \alpha_0\vct{g}+ \Gt^T\ab+b_\ell\vct{1}_n-\mtx{Y}_\ell - \Db^{-1}\ub}  - \beta\alpha_0 \twonorm{\Pb^\perp\h} \right)+ \frac{\twonorm{\vct{\ub}}^2}{2n}, 
\end{align}
Continuing let us denote $\tb:=\alpha_0\vct{g}+ \Gt^T\ab+b_\ell\vct{1}_n-\mtx{Y}_\ell$ for convenience and rewrite $\twonorm{\tb-\Db^{-1}\ub}$ as follows
$$
\frac{\twonorm{\tb-\Db^{-1}\ub}}{\sqrt{n}} = \min_{\tau>0}~\frac{\tau}{2} + \frac{\twonorm{\tb-\Db^{-1}\ub}^2}{2\tau n}.
$$
Note that the resulting minimization is convex in $\ub$ and concave in $\beta$. Also, by considering the bounded AO (such that $\beta$ is bounded; see \cite[Sec.~A]{deng2019model}), we can flip the order of min-max and optimize over $\ub$ first. In particular, $\ub$ minimizes the following strictly convex quadratic
\bea
\min_{\ub}\left\{ \frac{1}{n}\left(\frac{\beta}{2\tau}\twonorm{\Db^{-1}\ub} + \frac{1}{2}\twonorm{\ub}^2 - \frac{\beta}{\tau}\tb^T\Db^{-1}\ub \right) = \frac{1}{2n}\ub^T\left(\frac{\beta}{\tau}\Db^{-2} + \Iden_n \right)\ub-\frac{\beta}{\tau n}\tb^T\Db^{-1}\ub  \right\}.\nn
\end{align}
In particular,
$$
\ub = \frac{\beta}{\tau}\left(\frac{\beta}{\tau}\Db^{-2}+\Iden\right)^{-1} \Db^{-1} \tb = \left(\Db^{-1}+\frac{\tau}{\beta}\Db\right)^{-1}\left(\alpha_0\vct{g}+ \Gt^T\ab+b_\ell\vct{1}_n-\mtx{Y}_\ell\right)
$$
Putting things together, the new objective function of \eqref{eq:koko_WLS_1} becomes 
\bea
&\min_{\ab,\alpha_0\geq0 ,b_\ell,\tau>0}~~\max_{\beta\geq 0}~~\Rc(\ab,\alpha_0,b_\ell,\tau,\beta)\label{eq:koko2_WLS}\\
&\qquad\text{where}~~\Rc(\ab,\alpha_0,b_\ell,\tau,\beta):=\frac{\beta\tau}{2n}+ \frac{\beta}{2\tau n}\twonorm{\tb}^2 -\frac{\beta}{2\tau n} \tb^T\left(\Iden + \frac{\tau}{\beta}\Db^2\right)^{-1}\tb - \frac{{\beta\alpha_0}}{\sqrt{n}}\twonorm{\Pb^\perp\h}\nn.
\end{align}

\noindent\textbf{Convergence of the AO}
After having simplified the AO into an optimization problem over $r+4$ variables, we are ready to study its asymptotic behavior. First, we argue on point-wise convergence of $\Rc$ in \eqref{eq:koko2_WLS}. Fix $\ab,\alpha_0,\bb_\ell,\tau$ and $\beta$. From the WLLN, $\frac{1}{\sqrt{n}} \twonorm{\Pb^\perp\h}\rP\sqrt{\gamma}$ and as in \eqref{eq:t_conv}
\bea\nn
\frac{1}{n}\twonorm{ \tb}^2 = \frac{1}{n}\sum_{i=1}^n\left({\alpha_0\g_i+\ab^T\gt_i+b_\ell-[\Y_\ell]_i}\right)^2 \rP \E\left[\left({\alpha_0G_0+\ab^T\g+\bb_\ell-Y_\ell}\right)^2\right],
\end{align}
where the expectation is over $\g\sim\Nn(\zero_r,\Iden_r)$ (with some abuse of notation) and 
\begin{align}
Y_\ell \sim{\rm Bern}(V_\ell)\quad\text{and}\quad V_\ell=\frac{e^{\eb_\ell^T\Vb\Sigmab\g}}{\sum_{\ell^\prime=1}^{r}e^{\eb_{\ell^\prime}\Vb\Sigmab\g}}.\label{eq:Y_ell_log_WLS}
\end{align}
Furthermore,
\bea
\frac{1}{n}\tb^T\left(\Iden + \frac{\tau}{\beta}\Db^2\right)^{-1}\tb  = \frac{1}{n}\sum_{i=1}^{n}\frac{\left({\alpha_0\g_i+\ab^T\gt_i+b_\ell-[\Y_\ell]_i}\right)^2 }{1+\frac{\tau}{\beta}d_{i}^2} \rP \E\left[\frac{\left({\alpha_0G_0+\ab^T\g+\bb_\ell-Y_\ell}\right)^2}{1+\frac{\tau}{\beta}D^2}\right] \nn
\end{align}

 Therefore, point-wise on $\ab,\alpha_0,\bb_\ell,\tau$ and $\beta$, the objective $\Rc$ of the AO converges to 
 \begin{align}
\Dc_\ell(\alpha_0,\alphab,b_\ell,\tau,\beta)&:=\frac{\beta\tau}{2} + \frac{\beta}{2\tau}\E\left[\left({\alpha_0G_0+\ab^T\g+\bb_\ell-Y_\ell}\right)^2\right] - \frac{\beta}{2\tau} \E\left[\frac{\left({\alpha_0G_0+\ab^T\g+\bb_\ell-Y_\ell}\right)^2}{1+\frac{\tau}{\beta}D^2}\right] - \beta\alpha_0 \sqrt{\gamma}\nn\\
&=\frac{\beta\tau}{2} + \frac{1}{2} \E\left[\frac{D^2\left({\alpha_0G_0+\ab^T\g+\bb_\ell-Y_\ell}\right)^2}{1+\frac{\tau}{\beta}D^2}\right] - \beta\alpha_0 \sqrt{\gamma}\nn\\
&=\frac{\beta\tau}{2} + \frac{1}{2} \E\left[\frac{\left({\alpha_0G_0+\ab^T\g+\bb_\ell-Y_\ell}\right)^2}{D^{-2}+\left({\tau}/{\beta}\right)}\right] - \beta\alpha_0 \sqrt{\gamma}.\label{eq:det_log_WLS}
\end{align}
We note that the function above is jointly convex in $(\alpha_0,\alphab,b_\ell,\tau)$ and concave in $\beta$.

\subsection{Computing $\Sigma_{w,\mu}$}

It can be checked that the first order optimality conditions of $\Dc_\ell(\alpha_0,\alphab,b_\ell,\tau,\beta)$ with respect to $\beta$ and $\tau>0$ are given as follows:
\bea
\beta^2 &= \E\left[\frac{\left({\alpha_0G_0+\ab^T\g+\bb_\ell-Y_\ell}\right)^2}{\left(D^{-2}+\left({\tau}/{\beta}\right)\right)^2}\right]\quad\text{or}\quad \beta = 0, \label{eq:beta2_WLS}\\
\alpha_0\sqrt{\gamma} &= \frac{\tau}{2} + \frac{\tau}{2\beta^2}\cdot\E\left[\frac{\left({\alpha_0G_0+\ab^T\g+\bb_\ell-Y_\ell}\right)^2}{\left(D^{-2}+\left({\tau}/{\beta}\right)\right)^2}\right].
\end{align}
Thus, at optimality either $\beta=0$ or $\tau=\alpha_0\sqrt{\gamma}$. 
In what follows, consider the solution $\tau=\alpha_0\sqrt{\gamma}$. We will show that this leads to the true saddle point of $\Dc$.

Moreover, by denoting $\eta:=\frac{\beta}{\tau}$ and recalling from \eqref{eq:Y_ell_log_WLS} that $Y_\ell={\rm Bern}(V_\ell)$, we can express $\Dc_\ell(\alpha_0,\alphab,b_\ell,\tau,\beta)$ as follows
\begin{align}\nn
\frac{\beta\tau}{2} +  \frac{\alpha_0^2}{2}\E\left[\frac{1}{D^{-2}+1/\eta}\right] - \beta\alpha_0\sqrt\gamma + \frac{1}{2}\begin{bmatrix} \ab^T & \bb_\ell \end{bmatrix}\cdot\A\left(\eta\right)\cdot \begin{bmatrix} \ab \\ \bb_\ell \end{bmatrix} - \cb_\ell^T\left(\eta\right) \begin{bmatrix} \ab \\ \bb_\ell \end{bmatrix} + \frac{1}{2}\E\left[\frac{Y_\ell^2}{D^{-2}+1/\eta}\right],
\end{align}
where
\begin{subequations}\label{eq:Aeta}
\begin{align}
\A\left(\eta\right)&:= \begin{bmatrix} \E\left[\frac{\g\g^T}{D^{-2}+1/\eta}\right] 
&
\E\left[\frac{\g}{D^{-2}+1/\eta}\right] 
\\
\E\left[\frac{\g^T}{D^{-2}+1/\eta}\right] 
&
\E\left[\frac{1}{D^{-2}+1/\eta}\right] 
\end{bmatrix} \\
\cb_\ell\left(\eta\right)&:= \begin{bmatrix} \E\left[\frac{\g Y_\ell}{D^{-2}+1/\eta}\right] \end{bmatrix}
\end{align}
\end{subequations}
we have the following first-order optimality conditions for $\alpha_0,\ab$ and $\bb_\ell$:
\bea
\begin{bmatrix}
\ab \\ \bb_\ell
\end{bmatrix} &= \A^{-1}\left(\eta\right)\cdot \cb_\ell\left(\eta\right)\label{eq:WLS_ab}\\
\alpha_0 &= \beta\sqrt{\gamma}\Big/ \E\left[\frac{1}{D^{-2}+1/\eta}\right].\label{eq:E12}
\end{align}

Rearranging \eqref{eq:E12} and using $\tau=\alpha_0\sqrt{\gamma}$ gives the following equation for $\eta$:
\bea
\frac{\alpha_0\sqrt{\gamma}}{\beta}  \E\left[\frac{1}{D^{-2}+1/\eta}\right] = \gamma \stackrel{\tau=\alpha_0\sqrt{\gamma}}{\implies}~ \E\left[\frac{1/\eta}{D^{-2}+1/\eta}\right] = \gamma .\label{eq:key_eta}
\end{align}
Thus, the optimal values of $\ab$ and $\bb_\ell$ are found by \eqref{eq:WLS_ab} for $\eta$ the positive solution of the equation in \eqref{eq:key_eta}. To solve for $\alpha_0$, we combine \eqref{eq:E12} and \eqref{eq:beta2_WLS} which leads to 
\bea\label{eq:al0}
\alpha_0^2 \left( {\gamma}{\eta^2} - \E\left[\left(\frac{1}{D^{-2}+1/\eta}\right)^2\right] \right) = \E\left[\frac{\left({\ab^T\g+\bb_\ell-Y_\ell}\right)^2}{\left(D^{-2}+1/\eta\right)^2}\right],
\end{align}
where we have also used the RHS of \eqref{eq:key_eta}. Next, we specialize these findings to the special structure of the weighting matrix $\Db$ in \eqref{eq:D_common}.

\vp
\noindent\textbf{Applying weighting \eqref{eq:D_common}.}~ Assume \eqref{eq:D_common} holds. 
In this case, Equation \eqref{eq:key_eta} that determines the value of $\eta>0$ becomes
\bea\label{eq:eta_common}
F(\eta):=\sum_{i\in[k]}{\frac{\pib_i \omega_i^2}{\omega_i^2+\eta}} = {\gamma},
\end{align}
where we have recalled the notation in \eqref{eq:alphas_gen} $\pib_i:=\E[V_i]>0,~i\in[k]$. It can be easily checked by direct differentiation that $\eta\mapsto F$ is strictly decreasing in $(0,\infty)$. Also, using $\sum_{i\in[k]}\pib_i=1$ the range of $F$ in $(0,\infty)$ is $(0,1)$. Thus, it follows that \eqref{eq:eta_common} has a unique solution for all $\gamma\in(0,1)$.

Also, in this case we can write \eqref{eq:Aeta} in the following more convenient form:
\begin{align}\label{eq:Aks}
\A\left(\eta\right)&:= \sum_{i\in[k]} \left(\frac{\omega_i^2\eta}{{\omega_i^2} + \eta}\right) \underbrace{ \E\left[\begin{bmatrix} \g \\ 1 \end{bmatrix}\begin{bmatrix} \g^T & 1 \end{bmatrix} V_i\right] }_{=:\widetilde\A_i}
\\
\cb_\ell\left(\eta\right)&:= \left(\frac{\omega_\ell^2\eta}{{\omega_\ell^2} + \eta}\right) \underbrace{\E\left[\begin{bmatrix} \g \\ 1 \end{bmatrix} V_\ell \right]}_{=:\widetilde\cb_\ell}.\label{eq:E23}
\end{align}

For convenience let us define vectors $\nub:=\nub(\eta), \pibt=\pibt(\eta)\in\R^{k}$ with entries:
\bea\label{eq:pibt_def}
\pibt_i := \pib_i\left( \frac{1}{\gamma}\cdot{\frac{\omega_i^2}{\omega_i^2+\eta}}\right) =: \pib_i\cdot \nub_i
\end{align}
Because of \eqref{eq:eta_common}, notice that $\pibt$ is a probability vector, i.e. $$\pibt^T\one_k=\pib^T\nub =1.$$ 

 With the notation above, it holds
\bea
\A(\eta) &= \gamma\cdot\eta\cdot\begin{bmatrix} \sum_{i\in[k]} \nub_i\cdot\E[V_i\g\g^T] & \sum_{i\in[k]} \nub_i\cdot\E[V_i\g] \\ 
\sum_{i\in[k]} \nub_i\cdot\E[V_i\g^T] & 1
 \end{bmatrix} \nn \\
 &= \gamma\cdot\eta\cdot\begin{bmatrix} \sum_{i\in[k]} \nub_i\cdot\E[V_i\g\g^T] & \Sigmab\Vb^T\left(\diag{\pib}-\Pib\right)\,\nub \\ 
\nub^T\left(\diag{\pib}-\Pib\right)\Vb\Sigmab & 1
 \end{bmatrix} \\
  &= \gamma\cdot\eta\cdot\begin{bmatrix} \E\left[\left(\nub^T\vb\right)\g\g^T\right] & \Sigmab\Vb^T\left(\diag{\pib}-\Pib\right)\,\nub \\ 
\nub^T\left(\diag{\pib}-\Pib\right)\Vb\Sigmab & 1
 \end{bmatrix} \label{eq:AWLS}
 \\
\cb_\ell(\eta) &= \gamma\cdot\eta\cdot \begin{bmatrix}
 \nub_\ell \E[V_\ell\g] \\ \pibt_\ell
 \end{bmatrix} \nn\\
 &= \gamma \cdot \eta\cdot \begin{bmatrix}  
\Sigmab\Vb^T\left(\diag{\pib}-\Pib\right)\,\nub_\ell\,\eb_\ell \\ \pibt_\ell
 \end{bmatrix}\label{eq:cWLS}
\end{align}
where we have also used the fact that $\E[V_i\g]=\Sigmab\Vb^T\left(\diag{\pib}-\Pib\right)\eb_i,~i\in[k]$ and recalled the notation $$\vb = [V_1,\ldots,V_k]^T.$$

Using  \eqref{eq:AWLS} and \eqref{eq:cWLS}, we conclude from \eqref{eq:WLS_ab} the following expressions for $\ab$ and $\bb$:
\bea
\ab &= \Deltab^{-1}\Sigmab\Vb^T\left(\diag{\pib}-\Pib\right) \cdot\nub_\ell\cdot \left(\eb_\ell-\pib_\ell\nub\right), \label{eq:a_WLS1}\\
\bb_\ell &= \pibt_\ell - \nub^T\left(\diag{\pib}-\Pib\right)\V\Sigmab\Deltab^{-1}\Sigmab\Vb^T\left(\diag{\pib}-\Pib\right)\cdot\nub_\ell\cdot\left(\eb_\ell-\pib_\ell\nub\right), \label{eq:b_WLS1}
\end{align}
where we defined
\bea\label{eq:Delta_WLS}
\Deltab = \E\left[ \left(\nub^T\vb\right) \g\g^T \right] - \Sigmab\Vb^T\left(\diag{\pib}-\Pib\right)\nub\nub^T\left(\diag{\pib}-\Pib\right)\Vb\Sigmab \succ \zero_{r\times r}.
\end{align}


Finally, we show how to compute $\alpha_0$ using \eqref{eq:al0}. The  RHS in \eqref{eq:al0} can be computed as 
\bea
&\sum_{i\neq\ell\in[k]}\frac{\begin{bmatrix} \ab^T & \bb_\ell \end{bmatrix} \widetilde\A_i \begin{bmatrix} \ab \\ \bb_\ell \end{bmatrix} }{\left(\omega_i^{-2}+1/\eta\right)^2} + \frac{\begin{bmatrix} \ab^T & \bb_\ell \end{bmatrix} \widetilde\A_\ell \begin{bmatrix} \ab \\ \bb_\ell \end{bmatrix} - 2\begin{bmatrix} \ab^T & \bb_\ell \end{bmatrix} \widetilde\cb_\ell + \pib_\ell}{\left(\omega_\ell^{-2}+1/\eta\right)^2} \nn \\
&\qquad\qquad=\eta^2\cdot\gamma^2\cdot\left\{ {\begin{bmatrix} \ab^T & \bb_\ell \end{bmatrix} \left( \sum_{i\in[k]} \nub_i^2 \widetilde\A_i \right) \begin{bmatrix} \ab \\ \bb_\ell \end{bmatrix} } - 2 \begin{bmatrix} \ab^T & \bb_\ell \end{bmatrix} \nub_\ell^2 \widetilde\cb_\ell + \pib_\ell\nub_\ell^2 \right\},\nn
\end{align}
where $\ab,\bb_\ell$ are as in \eqref{eq:a_WLS1} and  \eqref{eq:b_WLS1}. Also, note that 
$$
\E\left[\left(\frac{1}{D^{-2}+1/\eta}\right)^2\right]  = \eta^2\sum_{i\in[k]}\frac{\pib_i \omega_i^4}{\left(\omega_i^2+\eta\right)^2} = \eta^2\cdot\gamma^2\cdot\pib^T\diag{\nub}\nub = \eta^2\cdot\gamma^2\cdot\pibt^T\nub.
$$
Put together, we have the following expression for $\alpha_0$:
\bea\label{eq:al0_WLS}
\alpha_0^2 &= \frac{1}{\left(1/\gamma - \pibt^T\nub\right)}\cdot\left\{ \begin{bmatrix} \ab^T & \bb_\ell \end{bmatrix} \left( \sum_{i\in[k]} \nub_i^2 \widetilde\A_i \right) \begin{bmatrix} \ab \\ \bb_\ell \end{bmatrix}  - 2 \begin{bmatrix} \ab^T & \bb_\ell \end{bmatrix} \nub_\ell^2 \widetilde\cb_\ell + \pib_\ell\nub_\ell^2 \right\}\nn\\
&= \frac{1}{\left(1/\gamma - \pibt^T\nub\right)}\cdot\left\{ \begin{bmatrix} \ab^T & \bb_\ell \end{bmatrix} \A^\prime \begin{bmatrix} \ab \\ \bb_\ell \end{bmatrix}  - 2 \begin{bmatrix} \ab^T & \bb_\ell \end{bmatrix} \begin{bmatrix}  
\Sigmab\Vb^T\left(\diag{\pib}-\Pib\right)\,\nub_\ell^2\,\eb_\ell \\ \pibt_\ell\cdot\nub_\ell
 \end{bmatrix}  + \pibt_\ell\cdot\nub_\ell \right\},
\end{align}
where $\ab,\bb_\ell$ are as in \eqref{eq:a_WLS1}, \eqref{eq:b_WLS1} and we have also defined
\bea\label{eq:A_al0}
\A^\prime = \begin{bmatrix} \E\left[\left(\nub^T\diag{\nub}\vb\right)\g\g^T\right] & \Sigmab\Vb^T\left(\diag{\pib}-\Pib\right)\,\diag{\nub}\nub \\ 
\nub^T\diag{\nub}\left(\diag{\pib}-\Pib\right)\Vb\Sigmab & \nub^T\diag{\nub}\pib
 \end{bmatrix} \,.
\end{align}

\vp
\noindent\textbf{Asymptotic Predictions.}~ Writing \eqref{eq:b_WLS1} in vector form we find that
\bea
\bh \rP\pibt-\diag{\nub}\left(\Iden_k-\pib\nub^T\right)\left(\diag{\pib}-\Pib\right)\V\Sigmab\Deltab^{-1}\Sigmab\Vb^T\left(\diag{\pib}-\Pib\right)\nub. \label{eq:b_WLS_fin}
\end{align}
Also, recalling that $\eb_\ell^T\Sigmab_{\w,\mub} = \wh_\ell^T\Ub\Sigmab\Vb^T \rP \ab^T\Sigmab\Vb^T$ and using \eqref{eq:a_WLS1}:
\bea\label{eq:Sigmawmu_WLS_fin}
\Sigmab_{\w,\mub} \rP \diag{\nub}\left(\Iden_k-\pib\nub^T\right)\left(\diag{\pib}-\Pib\right)\V\Sigmab\Deltab^{-1}\Sigmab\Vb^T.
\end{align}

Finally, for the magnitudes of the weight vectors, recall that $\twonorm{\wh_\ell}^2\rP\twonorm{\ab}^2+\alpha_0^2$. Thus, to find the limiting values of the norms, we can combine \eqref{eq:al0_WLS} and \eqref{eq:a_WLS1}-\eqref{eq:b_WLS1}. For convenience, we summarize the final expression here. Define the following\footnote{Note the slight abuse of notation compared to the definitions in \eqref{eq:AWLS} and \eqref{eq:AWLS}. This ``renaming" should not be confusing as the constant $\gamma\cdot\eta$ (that is different between the two definitions) cancels when computing $\begin{bmatrix} \ab \\ \bb_\ell \end{bmatrix}=\A^{-1}\cb_\ell$ (see \eqref{eq:WLS_ab}).}
\bea
\A &:= \begin{bmatrix} \E\left[\left(\nub^T\vb\right)\g\g^T\right] & \Sigmab\Vb^T\left(\diag{\pib}-\Pib\right)\,\nub \\ 
\nub^T\left(\diag{\pib}-\Pib\right)\Vb\Sigmab & 1
 \end{bmatrix} \label{eq:AWLS}  \\
\cb_\ell &:=
 \begin{bmatrix}  
\Sigmab\Vb^T\left(\diag{\pib}-\Pib\right)\,\nub_\ell\,\eb_\ell \\ \pibt_\ell
 \end{bmatrix}\label{eq:cWLS}.
\end{align}
Further recall the matrix $\A^\prime$ in \eqref{eq:A_al0}. 
\bea
\twonorm{\wh_\ell}^2&\rP\twonorm{\Deltab^{-1}\Sigmab\Vb^T\left(\diag{\pib}-\Pib\right) \cdot\nub_\ell\cdot \left(\eb_\ell-\pib_\ell\nub\right)}^2 \nn\\
&\qquad+ 
\frac{1}{\left(1/\gamma - \pibt^T\nub\right)}\cdot\left\{ \cb^T_\ell\A^{-1}\A^\prime\A^{-1}\cb_\ell- 2 \nub_\ell\cb_\ell^T\A^{-1}\cb_\ell  + \pibt_\ell\cdot\nub_\ell \right\}.\label{eq:well}
\end{align}

\begin{remark}Consider the special case $\omega_i=1,~i\in[k]$. We show how the above recovers the solution for (un-weighted) LS. First, note that in this case \eqref{eq:eta_common} simply gives 
$
\eta = \frac{1}{\gamma} - 1.
$
Thus, $\nub=\one_k$ and $\pibt=\pib$. Also, recall that $\left(\diag{\pib}-\Pib\right)\one_k = \zero$ and $\one^T\vb=1$. Thus, \eqref{eq:Delta_WLS} simply gives
$\Deltab = \E[\g\g^T] = \Iden_r$. With these, it can be readily checked that \eqref{eq:b_WLS_fin} and \eqref{eq:a_WLS1} simplify to the expressions in \eqref{eq:ls_soft_2}. Similarly, $\A=\A^\prime = \Iden_{r+1}$ and \eqref{eq:al0_WLS} reduces in this case to \eqref{eq:wh_norm_log1}. For general weight coefficients, such simplifications do not seem possible and one needs to compute the matrix $\E\left[\left(\nub^T\vb\right)\g\g^T\right]$ that appears in the definitions of $\Deltab, \A$ and $\A^\prime$. We note that this calculation can be somewhat simplified by applying Gaussian integration by parts similar to lemma \ref{lem:IBP}. 
\end{remark}

\subsection{Computing $\Sigma_{w,w}$}
In this section, we use Lemma \ref{lem:joint_LS} to compute the cross-correlations $\inp{\wh_\ell}{\wh_j},~j\neq\ell\in[k]$. Specifically, the analysis of \eqref{eq:joint_WLS} is almost identical to the analysis of \eqref{eq:PO_WLS} in the previous section. Specifically, without repeating all the details for brevity, it can be shown that the AO of \eqref{eq:joint_WLS} converges to $\min_{\ab,\alpha_0\geq0 ,b_\ell,\tau>0}~~\max_{\beta\geq 0}~~\Dc(\ab,\alpha_0,b_\ell,\tau,\beta)$ where $\Dc(\ab,\alpha_0,b_\ell,\tau,\beta)$ is as in  \eqref{eq:det_log_WLS} only with $Y_\ell$ substituted by $Y_{\ell,c}$:
\begin{align}
Y_{\ell,c} \sim{\rm Bern}(V_{c} + V_{\ell})\quad\text{and as before:}\quad V_{i}=\frac{e^{\eb_i^T\Vb\Sigmab\g}}{\sum_{\ell^\prime=1}^{r}e^{\eb_{\ell^\prime}\Vb\Sigmab\g}},~i=\ell,c.\label{eq:Y_ell_log_2}
\end{align}

Thus, what changes in the calculations above is in \eqref{eq:E23} and \eqref{eq:RHS24}, where we now have instead
\bea
\cb\left(\eta\right)&:= \left(\frac{1}{\frac{1}{\omega_\ell^2} + 1/\eta}\right) \underbrace{\E\left[\begin{bmatrix} \g \\ 1 \end{bmatrix} V_\ell \right]}_{=:\widetilde\cb_\ell}+\left(\frac{1}{\frac{1}{\omega_c^2} + 1/\eta}\right)\underbrace{\E\left[\begin{bmatrix} \g \\ 1 \end{bmatrix} V_c \right]}_{=:\widetilde\cb_c} 
\end{align}
and
\bea\label{eq:RHS24}
&\sum_{i\neq\{\ell,c\}\in[k]}\frac{\begin{bmatrix} \ab^T & \bb_\ell \end{bmatrix} \widetilde\A_i \begin{bmatrix} \ab \\ \bb_\ell \end{bmatrix} }{\left(\omega_i^{-2}+1/\eta\right)^2} + \frac{\begin{bmatrix} \ab^T & \bb_\ell \end{bmatrix} \widetilde\A_\ell \begin{bmatrix} \ab \\ \bb_\ell \end{bmatrix} - 2\begin{bmatrix} \ab^T & \bb_\ell \end{bmatrix} \widetilde\cb_\ell + \pib_\ell}{\left(\omega_\ell^{-2}+\eta\right)^2} 
\\
&\qquad\qquad\qquad\qquad\qquad\qquad\qquad+ \frac{\begin{bmatrix} \ab^T & \bb_\ell \end{bmatrix} \widetilde\A_c \begin{bmatrix} \ab \\ \bb_\ell \end{bmatrix} - 2\begin{bmatrix} \ab^T & \bb_\ell \end{bmatrix} \widetilde\cb_c + \pib_c}{\left(\omega_c^{-2}+1/\eta\right)^2},\nn
\end{align}
respectively. With these and following mutatis-mutandis the steps and the notation in the previous section, we find the following asymptotic expression for the magnitude of $\wh_\ell+\wh_c$:
\bea
&\twonorm{\wh_\ell+\wh_c}^2 \rP \twonorm{\Deltab^{-1}\Sigmab\Vb^T\left(\diag{\pib}-\Pib\right) \cdot\left(\nub_\ell\cdot \left(\eb_\ell-\pib_\ell\nub\right) + \nub_c\cdot \left(\eb_c-\pib_c\nub\right)\right)}^2 \nn\\
&\quad+ 
\frac{1}{\left(1/\gamma - \pibt^T\nub\right)}\cdot\left\{ (\cb_\ell+\cb_c)^T\A^{-1}\A^\prime\A^{-1}(\cb_\ell+\cb_c)- 2\left(\nub_\ell\cb_\ell+\nub_c\cb_c\right)^T\A^{-1}\left(\nub_\ell\cb_\ell+\nub_c\cb_c\right)  + \pibt_\ell\cdot\nub_\ell + \pibt_c\cdot\nub_c \right\}. \nn
\end{align}
We may now combine this with \eqref{eq:well} to conclude with the following asymptotic limits for the cross-correlations for all $\ell\neq c\in[k]$:
\bea
&\inp{\wh_\ell}{\wh_c} \rP
\nub_c \left(\eb_c-\pib_c\nub\right)^T\left(\diag{\pib}-\Pib\right)\Vb\Sigmab
\Deltab^{-2}\Sigmab\Vb^T\left(\diag{\pib}-\Pib\right)  \nub_\ell \left(\eb_\ell-\pib_\ell\nub\right) \nn
\\
&\qquad\qquad\qquad\qquad+\frac{1}{\left(1/\gamma - \pibt^T\nub\right)}\cdot\left\{ \cb^T_c\A^{-1}\A^\prime\A^{-1}\cb_\ell- 2 \nub_c\nub_\ell \cb_c^T\A^{-1}\cb_\ell\right\}\,. \label{eq:cross-corr}\\
&=\nub_c \left(\eb_c-\pib_c\nub\right)^T\left(\diag{\pib}-\Pib\right)\Vb\Sigmab
\Deltab^{-2}\Sigmab\Vb^T\left(\diag{\pib}-\Pib\right)  \nub_\ell \left(\eb_\ell-\pib_\ell\nub\right) \nn
\\
&\qquad\qquad\qquad\qquad+\frac{1}{\left(1/\gamma - \pibt^T\nub\right)}\cdot\left\{ \cb^T_c\left(\A^{-1}\A^\prime\A^{-1} - 2 \nub_c\nub_\ell \A^{-1}\right)\cb_\ell\right\}\,. \label{eq:cross-corr}
\end{align}
In matrix form, we have
\bea 
&\Sigmab_{\w,\w}\rP \diag{\nub}\left(\Iden_k-\pib\nub^T\right)\left(\diag{\pib}-\Pib\right)\Vb\Sigmab
\Deltab^{-2}\Sigmab\Vb^T\left(\diag{\pib}-\Pib\right)\left(\Iden_k-\nub\pib^T\right)\diag{\nub} \nn\\
&\quad+\frac{1}{\left(1/\gamma - \pibt^T\nub\right)}\Big\{\begin{bmatrix}\diag{\nub}\left(\diag{\pib}-\Pib\right)\Vb\Sigmab & \pibt\end{bmatrix}  \A^{-1}\A^\prime\A^{-1} \begin{bmatrix}\Sigmab\Vb^T\left(\diag{\pib}-\Pib\right)\diag{\nub}\\ \pibt^T\end{bmatrix}  \Big\}  \nn
\\&\quad-2\,\frac{1}{\left(1/\gamma - \pibt^T\nub\right)}\Big\{\diag{\nub}\begin{bmatrix}\diag{\nub}\left(\diag{\pib}-\Pib\right)\Vb\Sigmab & \pibt\end{bmatrix}  \A^{-1} \begin{bmatrix}\Sigmab\Vb^T\left(\diag{\pib}-\Pib\right)\diag{\nub}\\ \pibt^T\end{bmatrix} \diag{\nub}
 \Big\}  \nn\\
& \quad+ \frac{1}{\left(1/\gamma - \pibt^T\nub\right)}\Big\{\diag{\nub}\,\diag{\pibt}\Big\}\,.
\label{eq:ww_WLS}
\end{align}

\end{document}